\pdfoutput=1
\documentclass{article}

\usepackage[utf8]{inputenc}
\usepackage[T1]{fontenc}
\usepackage{microtype}
\usepackage{graphicx}
\usepackage{booktabs}
\usepackage{amsmath,amssymb,amsfonts,mathtools,amsthm}
\usepackage{nicefrac}
\usepackage{xcolor}
\usepackage{subcaption}
\usepackage{mathrsfs}
\usepackage[normalem]{ulem}
\usepackage{pifont}
\usepackage{algorithmic,algorithm}
\usepackage{wrapfig}
\usepackage{diagbox}
\usepackage{enumitem}
\usepackage{multirow}
\usepackage{soul}
\usepackage{float}
\usepackage{placeins}
\usepackage{thmtools,thm-restate}
\usepackage{bm}
\usepackage{titletoc}
\usepackage{etoc}

\usepackage{tikz}
\usetikzlibrary{fit, positioning, arrows.meta}

\PassOptionsToPackage{numbers, compress}{natbib}
\usepackage[final,main]{neurips_2026}

\usepackage{xurl}
\usepackage{hyperref}
\hypersetup{
  pdftitle={Structure-agnostic Causal Representation Learning},
  pdfauthor={Arman Behnam, Binghui Wang},
  pdfsubject={Causal representation learning (NeurIPS 2026)},
  pdfkeywords={causal representation learning, structure identification, invariant representation learning, HSIC, anti-causal learning, out-of-distribution generalization, domain generalization}
}
\usepackage[capitalize,noabbrev]{cleveref}

\theoremstyle{plain}
\newtheorem{theorem}{Theorem}[section]
\newtheorem{proposition}{Proposition}
\newtheorem{lemma}{Lemma}
\newtheorem{corollary}{Corollary}
\theoremstyle{definition}
\newtheorem{definition}{Definition}[section]
\newtheorem{assumption}{Assumption}
\theoremstyle{remark}

\definecolor{myred}{RGB}{215,48,39}
\definecolor{mygreen}{RGB}{26,152,80}

\renewcommand{\algorithmicrequire}{\textbf{Input:}}
\renewcommand{\algorithmicensure}{\textbf{Output:}}

\usepackage[disable]{todonotes}

\begin{document}

\title{Structure-agnostic Causal Representation Learning}

\author{%
  Arman Behnam \\
  Department of Computer Science \\
  Illinois Institute of Technology \\
  Chicago, IL, USA \\
  \texttt{abehnam@hawk.illinoistech.edu} \\
  \And
  Binghui Wang \\
  Department of Computer Science \\
  Illinois Institute of Technology \\
  Chicago, IL, USA \\
  \texttt{bwang70@illinoistech.edu} \\
}

\maketitle

\begin{abstract}
Causal representation learning aims to discover robust features by exploiting the causal structure underlying data generation. Existing methods require specifying the causal structure a priori, yet different structures demand fundamentally incompatible invariance constraints, and misspecification leads to representations that discard predictive information. We introduce SaCRL, a framework that jointly identifies the causal structure and learns the corresponding invariant representation without prior structural knowledge. Our approach formulates structure selection as a soft optimization over candidate invariances using HSIC-based violation metrics, with adaptive weights that automatically concentrate on the achievable structure. We provide theoretical guarantees for structure identification, including under random-feature approximation, invariance satisfaction, and out-of-distribution generalization. Empirically, SaCRL recovers the true structure on synthetic and semi-synthetic Bayesian-network benchmarks, outperforms fixed-invariance baselines on Colored MNIST, achieves state-of-the-art accuracy on three DomainBed benchmarks (PACS, VLCS, OfficeHome), and degrades gracefully under structural misspecification and limited environment diversity. {Code is available at: \url{https://github.com/ArmanBehnam/sacrl}}.
\end{abstract}
\section{Introduction}
\vspace{-2mm}
Causal representation learning~\cite{scholkopf2021toward} has emerged as a principled approach to building robust models by uncovering the underlying mechanisms that generate data, enabling generalization beyond training distributions~\cite{scholkopf2021toward,peters2017elements}. Current approaches are broadly categorized into two paradigms, each with distinct structural assumptions.
\textbf{Structure-based} methods~\cite{wang2024disentangled,squires2023causal,behnam2024graph} require explicit knowledge of Directed Acyclic Graphs (DAGs) and complete Structural Causal Models (SCMs), imposing significant barriers when causal structures are unknown. \textbf{Distribution-invariant} methods~\cite{arjovsky2019invariant,krueger2021out} avoid an explicit graph but enforce a fixed invariance across environments, committing to one causal structure. When that commitment is wrong, the enforced invariance discards predictive signal.

\begin{wrapfigure}{r}{0.4\textwidth}
\vspace{-4mm}
    \centering
     \scalebox{0.7}{\begin{tikzpicture}[
        scale=0.6,
        node distance=1cm,
        box/.style={rectangle, draw, thick, rounded corners, align=center, minimum height=1cm},
        databox/.style={box, fill=blue!10, minimum width=2cm},
        algobox/.style={box, fill=orange!10, minimum width=2cm},
        resultbox/.style={box, fill=green!10, minimum width=2cm},
        arrow/.style={->, >=stealth, very thick}
    ]
        
        \node[font=\scriptsize, databox] (data) at (4,4.5) {Multi-Hospital\\Clinical Data\\$\mathcal{D} = \{D_{e_i}\}_{i=1}^3$};
        \node[font=\scriptsize, right=0.1cm of data, align=left] {$e_1$: General\\$ e_2$: ICU\\$e_3$: Rural};
        
        \node[font=\scriptsize, box, fill=red!10, minimum width=2.5cm] (unknown) at (4.5,-0.7) {Unknown Structure\\for Each Variable};
        
        \node[font=\scriptsize,algobox] (algo) at (7.5,2) {\textbf{Structure-Agnostic}\\ Causal Representation Learning};
        
        \node[font=\scriptsize,resultbox] (structure) at (10.2,4.5) {Discovered\\Structures};
        \node[font=\scriptsize, right=0.01cm of structure, align=center] 
            {Symptoms: $\mathcal{G}_1$\\Biomarker: $\mathcal{G}_2$\\Vitals: $\mathcal{G}_3$};
        
        \node[font=\scriptsize, resultbox] (rep) at (9,-.7) {Optimal\\Representation $\phi^*$};
        
        \node[font=\scriptsize, resultbox] (pred) at (13,1.5) {Prediction on New \\Hospital $e_{\text{test}}$};
        \node[font=\scriptsize, below=0.1cm of pred] {Low Test Error};
        
        \draw[arrow] (data) -- (algo);
        \draw[arrow] (unknown) -- (algo);
        \draw[arrow] (algo) -- (structure);
        \draw[arrow] (algo) -- (rep);
        \draw[arrow] (structure) -- (pred);
        \draw[arrow] (rep) -- (pred); 
    \end{tikzpicture}
    }
    \caption{Clinical toy example.}
    \label{fig:toy_example}
    \vspace{-4mm}
\end{wrapfigure}
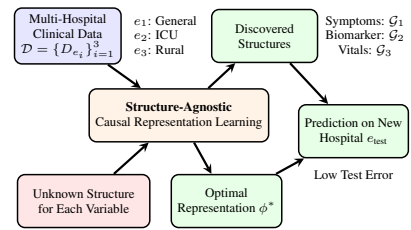
Consider predicting patient outcomes from clinical data collected at several hospitals (Figure~\ref{fig:toy_example}). Different variables follow different causal patterns: symptoms arise anti-causally from the disease ($\mathcal{G}_1$); biomarkers are confounded descendants that share a genetic cause with the diagnosis ($\mathcal{G}_2$); and vital signs are confounded outcomes, responding to the disease and to unobserved severity ($\mathcal{G}_3$). The structure for each variable is unknown to the learner.

Existing methods require manually partition variables by  (a priori) causal type, a process that is error-prone (misspecification yields suboptimal representations~\cite{yang2022learning,li2022learning}), computationally prohibitive ($O(m^k)$ evaluations for $m$ structures over $k$ variable subsets~\cite{zhu2024causal, tchetgen2020introduction}), and theoretically unsound (no guarantees exist when structure identification introduces uncertainty~\cite{zhang2023identifiability,chen2022learning}).

\begin{figure}[t]
\centering

\begin{minipage}{0.5\textwidth}
\centering
\scalebox{0.75}{
\begin{tikzpicture}[
    scale=0.9,
    node distance=1.0cm,
    every node/.style={font=\small},
    structure/.style={circle, draw, minimum size=0.5cm, font=\scriptsize},
    latent/.style={circle, draw, dashed, minimum size=0.5cm, font=\scriptsize},
    arrow/.style={->, >=stealth, thick},
    box/.style={rectangle, draw, rounded corners, minimum width=2.6cm, minimum height=0.7cm, align=center},
    goodbox/.style={box, fill=green!10, draw=green!60},
    badbox/.style={box, fill=red!10, draw=red!60}
]

\node[font=\bfseries] at (-3.5,3.2) {\footnotesize Struc.-Specific (Correct)};

\node[latent] (U1) at (-4.3,2.4) {$U$};
\node[structure] (X1) at (-4.3,1.5) {$X$};
\node[structure] (Y1) at (-2.9,1.5) {$Y$};

\draw[arrow] (U1) -- (X1);
\draw[arrow] (X1) -- (Y1);
\draw[arrow] (U1) to[bend right=15] (Y1);

\node[font=\footnotesize] at (-3.5,0.8) {True: Conf-Desc};

\node[goodbox] at (-3.5,-0.4) {Apply $R_{\mathcal{G}_2}$:\\$P(Y|\phi(X))$ invariant};

\node[goodbox] at (-3.5,-1.8) {\textbf{Good OOD}\\Test Error: Low};

\draw[->, thick, green!60!black] (-3.5,0.5) -- (-3.5,-0.05);
\draw[->, thick, green!60!black] (-3.5,-0.9) -- (-3.5,-1.4);

\node[font=\bfseries] at (0,3.2) {\footnotesize Struc.-Specific (Wrong)};

\node[latent] (U2) at (-0.8,2.4) {$U$};
\node[structure] (X2) at (-0.8,1.5) {$X$};
\node[structure] (Y2) at (0.6,1.5) {$Y$};

\draw[arrow] (U2) -- (X2);
\draw[arrow] (X2) -- (Y2);
\draw[arrow] (U2) to[bend right=15] (Y2);

\node[font=\footnotesize] at (0,0.8) {True: Conf-Desc};

\node[badbox] at (0,-0.4) {Apply $R_{\mathcal{G}_1}$:\\$\phi(X) \perp E | Y$};

\node[badbox] at (0,-1.8) {\textbf{Poor OOD}\\Test Error: High};

\draw[->, thick, red!60!black] (0,0.5) -- (0,-0.05);
\draw[->, thick, red!60!black] (0,-0.9) -- (0,-1.4);

\node[font=\bfseries] at (3.5,3.2) {\footnotesize Struc.-Agnostic (Ours)};

\node[structure, draw=blue!60, thick] (X3) at (2.9,1.5) {$X$};
\node[structure, draw=blue!60, thick] (Y3) at (4.3,1.5) {$Y$};

\node[font=\footnotesize] at (3.5,0.8) {Unknown Structure};

\node[box, fill=blue!10, draw=blue!60] at (3.5,-0.4) {Learn $\boldsymbol{\alpha}$:\\Auto-select $R_{\mathcal{G}_k}$};

\node[goodbox] at (3.5,-1.8) {\textbf{Good OOD}\\Robust to Structure};

\draw[->, thick, blue!60!black] (3.5,0.5) -- (3.5,-0.05);
\draw[->, thick, blue!60!black] (3.5,-0.9) -- (3.5,-1.4);

\end{tikzpicture}
}
\caption{\textbf{Left}: When true structure (confounded-descendant) is known, applying the correct regularizer yields good Out-of-Distribution performance. \textbf{Middle}: Misspecifying structure leads to wrong invariance constraints and poor generalization. \textbf{Right}: Our structure-agnostic approach automatically discovers the structure and applies appropriate regularization.}
\label{fig:motivation}
\end{minipage}
\hfill
\begin{minipage}{0.48\textwidth}
    \centering

    \begin{minipage}{0.28\linewidth}
        \centering
       \scalebox{0.8}{\begin{tikzpicture}[
            scale=0.65,
            var/.style={draw,circle,inner sep=0pt,minimum size=0.4cm,font=\scriptsize},
            latentconf/.style={draw,circle,dashed,inner sep=0pt,minimum size=0.6cm,fill=orange!20,font=\scriptsize}
        ]
            \node (X) [var] {$X_z$};
            \node (Y) [var, right=0.55cm of X] {$Y$};
            \node (Z) [latentconf, left=0.55cm of X] {$Z$};
            \node (Xzperp) [var, above=0.55cm of X] {$X_z^{\perp}$};
            \node (U) [latentconf, below=0.55cm of X] {$U$};
            \node (E) [var, below=0.55cm of Y] {$E$};
            
            \path[->, line width=0.25mm] (Z) edge (X); 
            \path[->, line width=0.25mm] (Xzperp) edge (X);
            \path[->, line width=0.25mm] (E) edge (U); 
            \path[->, green!30!black!70, line width=0.25mm] (U) edge (X); 
            \path[->, green!30!black!70, line width=0.25mm] (Y) edge (X); 
            \path[->, green!30!black!70, bend left=45, line width=0.25mm] (Y) edge (Z); 
            \path[->, magenta, line width=0.25mm, draw opacity=0.5] (Y) edge (Xzperp);
            \path[->, magenta, line width=0.25mm, draw opacity=0.5] (U) edge (Y); 
            \path[->, magenta, line width=0.25mm, draw opacity=0.5] (U) edge (Z); 
            \path[->, green!30!black!70, line width=0.25mm] (Xzperp) edge (Z); 

            \tikzstyle{bigbox}=[draw, dotted, thick, inner sep=6pt]
            \node[bigbox, fit=(X)(Xzperp)] (Xfull) {};
            \node[below right, font=\scriptsize] at (Xfull.north west) {$X$};
        \end{tikzpicture}
        }
    \\[1mm]{\scriptsize (a) $\mathcal{G}_1$}
    \end{minipage}
    \hfill
        \begin{minipage}{0.28\linewidth}
        \centering
        \scalebox{0.8}{\begin{tikzpicture}[
            scale=0.65,
            var/.style={draw,circle,inner sep=0pt,minimum size=0.4cm,font=\scriptsize},
            latentconf/.style={draw,circle,dashed,inner sep=0pt,minimum size=0.6cm,fill=orange!20,font=\scriptsize}
        ]
    \node (X) [var] {$X_z$};
    \node (Y) [var, right=0.55cm of X] {$Y$};
    \node (Z) [latentconf, left=0.55cm of X] {$Z$};
    \node (Xzperp) [var, above=0.55cm of X] {$X_z^{\perp}$};
    \node (U) [latentconf, below=0.55cm of X] {$U$};
    \node (E) [var, below=0.55cm of Y] {$E$};
    \path[->, line width=0.3mm] (Z) edge (X); 
    \path[->, line width=0.3mm] (Xzperp) edge (X);
    \path[->, line width=0.3mm] (E) edge (U); 
    \path[->, green!30!black!70, line width=0.3mm] (U) edge (X); 
    \path[->, green!30!black!70, line width=0.3mm] (Y) edge (X); 
    \path[->, green!30!black!70, bend left=45, line width=0.3mm] (Y) edge (Z); 
    \path[->, magenta, line width=0.3mm, draw opacity=0.5] (Xzperp) edge (Y);
    \path[->, magenta, bend left=90, line width=0.3mm, draw opacity=0.5] (U) edge (Xzperp); 
    \path[->, magenta, dashed, line width=0.3mm, draw opacity=0.5] (U) edge (Z); 
    \path[->, magenta, dashed, line width=0.3mm, draw opacity=0.5] (Xzperp) edge (Z); 

    \tikzstyle{bigbox}=[draw, dotted, very thick, inner sep=8pt]
    \node[bigbox, fit=(X)(Xzperp)] (Xfull) {};
    \node[below right] at (Xfull.north west) {$X$};
\end{tikzpicture} 
}
\\[1mm]{\scriptsize (c) $\mathcal{G}_2$}
    \end{minipage}
    \hfill
    \begin{minipage}{0.28\linewidth}
        \centering
    \scalebox{0.8}{\begin{tikzpicture}[
            scale=0.65,
            var/.style={draw,circle,inner sep=0pt,minimum size=0.4cm,font=\scriptsize},
            latentconf/.style={draw,circle,dashed,inner sep=0pt,minimum size=0.6cm,fill=orange!20,font=\scriptsize}
        ]
    \node (X) [var] {$X_z$};
    \node (Y) [var, right=0.55cm of X] {$Y$};
    \node (Z) [latentconf, left=0.55cm of X] {$Z$};
    \node (Xzperp) [var, above=0.55cm of X] {$X_z^{\perp}$};
    \node (U) [latentconf, below=0.55cm of X] {$U$};
    \node (E) [var, below=0.55cm of Y] {$E$};
    \path[->, line width=0.3mm] (Z) edge (X); 
    \path[->, line width=0.3mm] (Xzperp) edge (X);
    \path[->, line width=0.3mm] (E) edge (U); 
    \path[->, green!30!black!70, line width=0.3mm] (U) edge (X); 
    \path[->, green!30!black!70, line width=0.3mm] (Y) edge (X); 
    \path[->, green!30!black!70, bend left=45, line width=0.3mm] (Y) edge (Z); 
    \path[->, magenta, line width=0.3mm, draw opacity=0.5] (Xzperp) edge (Y);
    \path[->, magenta, line width=0.3mm, draw opacity=0.5] (U) edge (Y); 
    \path[->, magenta, line width=0.3mm, draw opacity=0.5] (U) edge (Z); 
    \path[->, green!30!black!70, line width=0.3mm] (Xzperp) edge (Z); 

    \tikzstyle{bigbox}=[draw, dotted, very thick, inner sep=8pt]
    \node[bigbox, fit=(X)(Xzperp)] (Xfull) {};
    \node[below right] at (Xfull.north west) {$X$};
\end{tikzpicture} }
    \\[1mm]{\scriptsize (b) $\mathcal{G}_3$}
    \end{minipage}
    \caption{Three canonical causal structures: (a) anti-causal $\mathcal{G}_1$, (b) confounded-descendant $\mathcal{G}_2$ (c) confounded-outcome $\mathcal{G}_3$. Latent variables $U, Z$ (shaded dashed nodes) influence the relationship between observed features $X$ (dotted box) and target $Y$, with environment $E$ capturing distributional shifts. Magenta arrows denote  causal mechanisms that distinguish each structure: $Y \rightarrow X$ for anti-causal, $X \rightarrow Y$ for both confounded types. Dashed arrows represent confounding paths through $U$ and $Z$.}
    \label{fig:structures}
\end{minipage}
\vspace{-6mm}
\end{figure}

Different causal structures demand incompatible invariance properties~\citep{makar2022fairness,wang2023causal}. A representation optimized for one structure performs poorly under another, as shown in Figure~\ref{fig:motivation}. The central question we address is \textbf{can we learn optimal invariant representations without knowing the causal structure?} This requires a paradigm shift from \emph{structure-specific} to \emph{structure-agnostic} learning\footnote{It refers to the learner’s prior knowledge, not the space of all possible structures. Prior methods require specifying the true structure that governs the data in advance, whereas we eliminate this requirement by learning the structure directly from data.}:  jointly discover causal mechanisms generating the data and learn representations tailored to those mechanisms. The challenge is both computational (avoiding exponential structure search) and theoretical (providing guarantees when structure itself is uncertain). 

We introduce a kernel-based framework, \emph{Structure-agnostic Causal Representation Learning (SaCRL)}, for joint structure discovery and representation learning. SaCRL generalizes Anti-Causal Invariant Abstractions (ACIA)~\cite{behnam2026measure}, which learns invariant representations for the anti-causal structure $\mathcal{G}_1$, to settings where the structure itself is unknown. The key theoretical insight is that under the true causal structure, the corresponding invariance is achievable by predictive representations, whereas wrong invariances incur irreducible violation. This asymmetry enables structure identification via a SoftMin objective that automatically concentrates on the correct structure. We prove: (i) consistent structure identification, which also holds under random-feature approximation of HSIC, (ii) bounds on the remaining invariance violation, (iii) out-of-distribution generalization bounds, and (iv) sample complexity for joint identification and generalization. Empirically, SaCRL also recovers the true structure on semi-synthetic Bayesian-network data with known ground truth, and its adaptive weights provide a calibrated confidence for the selected structure.
\vspace{-2mm}
\section{Background}
\label{sec:background}

\vspace{-2mm}
Random variables are uppercase ($X, Y, E$), realizations lowercase ($x, y, e$), and spaces calligraphic ($\mathcal{X}, \mathcal{Y}, \mathcal{E}$). We write $|\mathcal{E}|$ for set cardinality, $\mathbb{E}_P[\cdot]$ for expectation under $P$, and $\mathbb{I}_A$ for the indicator of set $A$. Statistical independence is denoted $X \perp\!\!\!\perp Y$ (unconditional) and $X \perp\!\!\!\perp Y \mid Z$ (conditional). Complete notations appear in Appendix~\ref{app:notation}.

\vspace{-2mm}
\subsection{Measure Theory}
\label{sec:preliminaries}
\vspace{-2mm}

Following ACIA~\cite{behnam2026measure}, this work uses measure-theoretic probability. See Appendix~\ref{app:measure_foundations} for formal definitions of measurable space, probability measure, product space, $\sigma$-algebra, and Markov kernel.

\vspace{-2mm}
\subsection{Canonical Causal Structures} 
\vspace{-2mm}
We formalize causal relationships using DAGs in Appendix~\ref{app:causal_structures}, which encode the qualitative structure of causal mechanisms. There exist three canonical causal structures \cite{wang2023causal}, with definitions below:

\begin{definition}
\label{def:canonical_structures}
Three Canonical Causal Structures. 

\begin{enumerate}[label=(\alph*),leftmargin=*,nosep]
\vspace{-1mm}
    \item \textbf{Anti-causal ($\mathcal{G}_1$)}: DAG contains edges $Y \rightarrow X \leftarrow E$, capturing settings where labels generate observations (e.g., disease causes symptoms). 

    \item \textbf{Confounded-descendant ($\mathcal{G}_2$)}:  Observed features decompose as $X = (X_z, X_z^{\perp})$ where $X_z^{\perp} \rightarrow Y$, and a latent confounder $U$ influences both $X_z^{\perp}$ and other components. DAG includes $X_z^{\perp} \rightarrow Y$ with $U \rightarrow X_z^{\perp}$.

    \item \textbf{Confounded-outcome $\mathcal{G}_3$}: Observed features decompose as $X = (X_z, X_z^{\perp})$ where $X_z^{\perp} \rightarrow Y \rightarrow X_z$, and a latent confounder $U$ influences both $Y$ and the latent variable $Z$ (which in turn affects $X_z$). DAG includes $X_z^{\perp} \rightarrow Y \rightarrow X_z$ with $U \rightarrow Y$ and $U \rightarrow Z \rightarrow X_z$.
\end{enumerate}
\vspace{-2mm}
\end{definition}

Here, $X_Z^{\perp}$ denotes the component of $X$ that is causally orthogonal to the $Z$-mediated pathway. Specifically, $X_Z$ captures variation in $X$ that flows through the latent $Z$ (e.g., features influenced by $Y$ via $Z$ in the anti-causal structure), and $X_Z^{\perp}$ captures variation from other sources, including direct effects from $Y$, confounders $U$, or environment $E$. Figure~\ref{fig:structures} illustrates the three canonical structures.

The three canonical structures form an \emph{exhaustive} partition of valid DAGs under two mild conditions: (i) {\bf partial observability:} the learner observes only $(X, Y, E)$ with latent confounders $U$ unobserved; (ii) {\bf environment mechanism:} $E$ influences the system through latent confounders rather than directly rewiring the $X$--$Y$ relationship. Under (i)--(ii), 
the causal relationship between $X$ and $Y$ must take one of three forms: 
(1) the $X$--$Y$ edge is either $Y \to X$ (anti-causal $\mathcal{G}_1$) or $X \to Y$ (causal). In the causal direction, confounding either operates through features ($\mathcal{G}_2$) or bypasses the label via a latent mediator ($\mathcal{G}_3$). Bidirectional $X \leftrightarrow Y$ violates acyclicity; unconfounded $X \to Y$ is recovered by $\mathcal{G}_2$ with zero confounder influence. These conditions are standard in causal representation learning~\cite{wang2023causal}.

Our theoretical analysis relies on the distinction between observational conditioning $P(Y \mid X=x)$ and interventional conditioning $P(Y \mid \text{do}(X=x))$, where the latter severs incoming edges to $X$ in the causal graph. While this distinction motivates why different causal structures require different invariance properties, our method circumvents explicit do-calculus. Formal definitions of interventions and causal independence appear in Appendix~\ref{app:interventions}.

\begin{definition}[Environment Space]
\label{def:environment_space}
An environment space consists of a finite collection $\mathcal{E} = \{e_1, \ldots, e_m\}$ of environments, where each environment $e_i \in \mathcal{E}$ has a measurable input space $(\mathcal{X}_{e_i}, \mathscr{F}_{\mathcal{X}_{e_i}})$, a measurable output space $(\mathcal{Y}_{e_i}, \mathscr{F}_{\mathcal{Y}_{e_i}})$, a probability measure $P_{e_i}$ on the product space $(\mathcal{X}_{e_i} \times \mathcal{Y}_{e_i}, \mathscr{F}_{\mathcal{X}_{e_i}} \otimes \mathscr{F}_{\mathcal{Y}_{e_i}})$, and a dataset $D_{e_i} = \{(x_j^{e_i}, y_j^{e_i})\}_{j=1}^{n_{e_i}}$ sampled i.i.d. from $P_{e_i}$.
\vspace{-2mm}
\end{definition}

Each environment $e$ induces a joint distribution $P_e(X, Y)$ that factorizes according to an underlying causal DAG $\mathcal{G}_e$. We formalize this via causal spaces that pair probability measures with their generating mechanisms. The key quantity is the interventional kernel distance $d_{\mathcal{K}}(e, e')$ measuring how causal mechanisms differ across environments. Full measure-theoretic causality definitions are in Appendix~\ref{app:causal_kernels}. One should understand the meaning of representation for this purpose.

\begin{definition}[Representation]
\label{def:representation}
It is a measurable function $\phi: \mathcal{X} \rightarrow \mathcal{Z}$ mapping input space $\mathcal{X}$ to a latent representation space $\mathcal{Z}$, where $(\mathcal{Z}, \mathscr{F}_{\mathcal{Z}})$ is a measurable space. 
\vspace{-2mm}
\end{definition}
\section{Problem Formulation}
\label{sec:problem}

\vspace{-2mm}
\subsection{Motivation}
\label{sec:motivation}
\vspace{-2mm}
Different causal structures demand fundamentally incompatible invariance properties. An anti-causal setting ($Y \rightarrow X$) requires representations conditionally independent of environments given the target. A confounded-descendant setting ($X \rightarrow Y$ with confounding) requires the predictive distribution $P(Y|\phi(X))$ to be invariant. A confounded-outcome setting requires the marginal $P(\phi(X))$ to be invariant. A representation optimized for one structure typically performs poorly under another.

Suppose the true structure is confounded-descendant ($\mathcal{G}_2$), but we enforce anti-causal invariance ($\mathcal{I}_1: \phi(X) \perp E | Y$). This constraint forces the representation to discard information that varies with $E$ conditional on $Y$, but in $\mathcal{G}_2$, this includes precisely the causal signal $X \rightarrow Y$. The learned representation becomes predictively useless. We observe that the correct invariance is achievable (the representation can satisfy it while maintaining prediction accuracy), whereas incorrect invariances are unachievable (satisfying them requires discarding predictive information).

\vspace{-2mm}
\subsection{Problem Setup}
\label{sec:setup}
\vspace{-2mm}
Let $\mathcal{E} = \{e_1, \ldots, e_m\}$ with $m \geq 2$ denote the set of environments. For each environment $e \in \mathcal{E}$, we observe a dataset $D_e = \{(x_j^e, y_j^e)\}_{j=1}^{n_e}$ drawn i.i.d.\ from an unknown joint distribution $P_e(X, Y)$ on the product space $(\mathcal{X} \times \mathcal{Y}, \mathscr{F}_{\mathcal{X}} \otimes \mathscr{F}_{\mathcal{Y}})$. The data-generating process follows one of three canonical causal structures $\mathcal{G}^* \in \{\mathcal{G}_1, \mathcal{G}_2, \mathcal{G}_3\}$, \emph{but the true structure $\mathcal{G}^*$ is unknown to the learner.}
\begin{definition}[Structure-Specific Invariance]
\label{def:structure_invariances}
For each causal structure $\mathcal{G}_k$, the corresponding invariance condition $\mathcal{I}_k$ on representation $\phi: \mathcal{X} \rightarrow \mathcal{Z}$ is: 
$\mathcal{I}_1(\phi): \phi(X) \perp\!\!\!\perp E \mid Y$ for anti-causal, 
$\mathcal{I}_2(\phi): Y \perp\!\!\!\perp E \mid \phi(X)$ for confounded-descendant, and 
$\mathcal{I}_3(\phi): \phi(X) \perp\!\!\!\perp E$ for confounded-outcome case,
where $\perp\!\!\!\perp$ denotes statistical independence and $E$ is the environment indicator.
\vspace{-2mm}
\end{definition}
We formalize the structure-agnostic setting with the product causal space framework. The learner's task is to identify which kernel family governs the observed data. Given data $\mathcal{D} = \{D_{e_i}\}_{i=1}^m$ from $m$ environments generated by an unknown causal structure $\mathcal{G}^* \in \{\mathcal{G}_1, \mathcal{G}_2, \mathcal{G}_3\}$, we want to identify the true causal structure $\mathcal{G}^*$ (structure discovery), find $\phi^*: \mathcal{X} \rightarrow \mathcal{Z}$ satisfying the invariance $\mathcal{I}_{k^*}$ corresponding to $\mathcal{G}^*$, and learn predictor $w^*: \mathcal{Z} \rightarrow \mathcal{Y}$ achieving optimal out-of-distribution generalization simultaneously.
Formally, we solve the below objective function: 
\begin{equation}
\label{eq:main_objective}
\min_{\phi, w, k} \left\{ \max_{e \in \mathcal{E}} \mathcal{R}_e(w \circ \phi) + \lambda \cdot \mathbb{I}[\neg\mathcal{I}_k(\phi)] \right\}
\end{equation}
where $\mathcal{R}_e(f) = \mathbb{E}_{(X,Y) \sim P_e}[\ell(f(X), Y)]$ is the risk in environment $e$. The objective jointly selects the causal structure $\mathcal{G}_k$, learns a representation $\phi$ satisfying the invariance $\mathcal{I}_k$, and optimizes worst-case risk across environments. Two challenges arise: (i) the structure $k$ is discrete, requiring search over candidates; (ii) the indicator $\mathbb{I}[\cdot]$ is non-differentiable. We address both by observing that each $\mathcal{I}_k$ is an independence statement, which can be quantified continuously via kernel independence testing.
\vspace{-2mm}
\section{Kernel Independence Framework}
\label{sec:kernel_framework}
\vspace{-2mm}
We develop a kernel-based framework offering rigorous independence characterization and differentiable optimization. Reproducing kernel Hilbert spaces (RKHS) (See Appendix~\ref{app:kernel_framework}) embed probability measures as unique elements, enabling independence testing via norm computations.

\begin{definition}[Characteristic Kernel]
\label{def:characteristic_kernel}
A kernel $k: \mathcal{X} \times \mathcal{X} \rightarrow \mathbb{R}$ is characteristic if the kernel mean embedding $\mu: \mathscr{P}(\mathcal{X}) \rightarrow \mathcal{H}$ is injective on the space of probability measures $\mathscr{P}(\mathcal{X})$ defined by
\begin{equation}
\mu_P := \mathbb{E}_{X \sim P}[k(\cdot, X)] = \int_{\mathcal{X}} k(\cdot, x) \, dP(x)
\end{equation}
\vspace{-4mm}
\end{definition}

Common characteristic kernels include the Gaussian RBF kernel $k(x, x') = \exp(-\|x - x'\|^2 / 2\sigma^2)$. For characteristic kernels, $\mu_P = \mu_Q$ if and only if $P = Q$.

We quantify dependence between random variables via the cross-covariance operator $\mathcal{C}_{XY}$, whose Hilbert-Schmidt norm measures dependence strength.

\begin{definition}[Hilbert-Schmidt Independence Criterion (HSIC)]
\label{def:hsic} 
Let $\|\cdot\|_{\mathrm{HS}}$ be the Hilbert-Schmidt norm. The HSIC between random variables $X$ and $Y$ w.r.t. characteristic kernels $k_{\mathcal{X}}$ and $k_{\mathcal{Y}}$ is:

\begin{equation}
\mathrm{HSIC}(X, Y) := \|\mathcal{C}_{XY}\|_{\mathrm{HS}}^2
\end{equation}

\vspace{-4mm}
\end{definition}
Key property of HSIC is it completely characterizes statistical independence for characteristic kernels.

\begin{lemma}[HSIC Characterizes Independence]
\label{lem:hsic_independence}
Let $k_{\mathcal{X}} \otimes k_{\mathcal{Y}}$ be characteristic (e.g., Gaussian or delta kernels~\cite{szabo2018characteristic}). Then:
\begin{equation}
\mathrm{HSIC}(X, Y) = 0 \quad \Longleftrightarrow \quad X \perp\!\!\!\perp Y
\end{equation}
\vspace{-2mm}
\end{lemma}
In practice, we estimate HSIC from finite samples using kernel matrices.

\begin{definition}[Empirical HSIC]
\label{def:empirical_hsic}
Given samples $\{(x_i, y_i)\}_{i=1}^n$, the empirical HSIC estimator is:
\begin{equation}
\widehat{\mathrm{HSIC}}(X, Y) = \frac{1}{n^2} \mathrm{tr}(K_X H K_Y H)
\end{equation}
where $[K_X]_{ij} = k_{\mathcal{X}}(x_i, x_j)$, $[K_Y]_{ij} = k_{\mathcal{Y}}(y_i, y_j)$, and $H = I_n - \frac{1}{n}\mathbf{1}_n\mathbf{1}_n^\top$ is the centering matrix.
\vspace{-2mm}
\end{definition}

To test conditional independence, we employ a residualization approach based on kernel regression.

\begin{definition}[Kernel Ridge Regression]
\label{def:kernel_residual}
Let $\{(x_i, z_i)\}_{i=1}^n$ be observations. The kernel ridge regression estimate is:
\begin{equation}
\hat{m}_X(z) := \mathbf{k}_z^\top (K_Z + \gamma I_n)^{-1} \mathbf{x}
\end{equation}
where $\gamma > 0$ is the regularization parameter, chosen via cross-validation or set to $O(n^{-1/(2+d)})$ for $d$-dimensional inputs. The residual at observation $i$ is $\tilde{x}_i := x_i - \hat{m}_X(z_i)$.
\vspace{-2mm}
\end{definition}
These residuals remove the dependence on $Z$, allowing us to test conditional independence by applying standard HSIC to the residualized variables.

\begin{definition}[Conditional HSIC]
\label{def:conditional_hsic}
The conditional HSIC for testing $X \perp\!\!\!\perp Y \mid Z$ is:
\begin{equation}
\mathrm{HSIC}(X, Y \mid Z) := \mathrm{HSIC}(\tilde{X}, \tilde{Y})
\end{equation}
where $\tilde{X} = X - \hat{m}_X(Z)$ and $\tilde{Y} = Y - \hat{m}_Y(Z)$ are kernel regression residuals.
\end{definition}
Below lemma states kernel regression residuals yield a consistent test for conditional independence.

\begin{lemma}[Conditional HSIC Consistency]
\label{lem:conditional_hsic}
Assume additive noise, $X = m_X(Z) + \varepsilon_X$ and $Y = m_Y(Z) + \varepsilon_Y$ with $(\varepsilon_X, \varepsilon_Y) \perp\!\!\!\perp Z$, bounded Lipschitz kernels with $k_{\mathcal{X}} \otimes k_{\mathcal{Y}}$ characteristic, and $L^2$-consistent regression estimates $\hat{m}_X, \hat{m}_Y$. Then $X \perp\!\!\!\perp Y \mid Z \Leftrightarrow \varepsilon_X \perp\!\!\!\perp \varepsilon_Y$, and:
\begin{enumerate}[nosep]
    \item If $X \perp\!\!\!\perp Y \mid Z$, then $\widehat{\mathrm{HSIC}}(\tilde{X}, \tilde{Y}) \xrightarrow{p} 0$.
    \item If $X \not\perp\!\!\!\perp Y \mid Z$, then $\widehat{\mathrm{HSIC}}(\tilde{X}, \tilde{Y}) \xrightarrow{p} \mathrm{HSIC}(\varepsilon_X, \varepsilon_Y) > 0$.
\end{enumerate}
\end{lemma}
Under this model, the violation metrics $V_1$ and $V_2$, which rely on conditional HSIC, correctly detect whether the corresponding invariances hold.
\vspace{-2mm}
\section{Structure-agnostic Causal Representation Learning (SaCRL)}
\label{sec:method}
\vspace{-2mm}

Our theoretical findings rely on standard regularity conditions for kernel methods and the representation class (Appendix~\ref{app:assumptions}). The key structural requirement is:

\begin{assumption}[Sufficient Environment Diversity]
\label{assum:sufficient_diversity}
The training environments exhibit variation along three axes: marginal distribution $P(X)$, conditional $P(X|Y)$, and outcome $P(Y|X)$. We assume the three causal structures are distinguishable: testing $\mathcal{I}_1$ requires variation in $P(X|Y)$, testing $\mathcal{I}_2$ requires variation in $P(Y|X)$, and testing $\mathcal{I}_3$ requires variation in $P(X)$.
\vspace{-2mm}
\end{assumption}

\subsection{Structure Violation Metrics}
\label{sec:violations}
\vspace{-2mm}
We quantify the degree to which a representation $\phi$ violates structure-specific invariance using HSIC.

\begin{definition}[Population Violation Functions]
\label{def:violations_population}
Let $\phi: \mathcal{X} \to \mathcal{Z}$ be a representation. The population violations are $V_1^*(\phi) := \mathrm{HSIC}(\phi(X), E \mid Y)$ (anti-causal), $V_2^*(\phi) := \mathrm{HSIC}(Y, E \mid \phi(X))$ (confounded-descendant), and $V_3^*(\phi) := \mathrm{HSIC}(\phi(X), E)$ (confounded-outcome).
\vspace{-2mm}
\end{definition}

For empirical violation functions, given pooled observations $\{(x_i, y_i, e_i)\}_{i=1}^N$ from all environments, the empirical violations are $V_1(\phi) := \widehat{\mathrm{HSIC}}(\tilde{\phi}, \tilde{E})$, $V_2(\phi) := \widehat{\mathrm{HSIC}}(\tilde{Y}, \tilde{E}')$, and $V_3(\phi) := \widehat{\mathrm{HSIC}}(\phi(X), E)$. Here, $\tilde{\phi}_i := \phi(x_i) - \hat{m}_{\phi}(y_i)$ are residuals from regressing $\phi(X)$ on $Y$, $\tilde{E}_i := e_i - \hat{m}_E(y_i)$ are residuals from regressing $E$ on $Y$, $\tilde{Y}_i := y_i - \hat{m}_Y(\phi(x_i))$ are residuals from regressing $Y$ on $\phi(X)$ $\tilde{E}'_i := e_i - \hat{m}_{E'}(\phi(x_i))$ are residuals from regressing $E$ on $\phi(X)$, and $\hat{m}(\cdot)$ kernel ridge regression estimates.

The distinction between population ($V_k^*$) and empirical ($V_k$) violations is essential. Theoretical results concern $V_k^*$, while the algorithm optimizes $V_k$. Consistency ensures $V_k(\phi) \xrightarrow{p} V_k^*(\phi)$ as $n \to \infty$.

\vspace{-2mm}
\subsection{Soft Structure Selection}
\label{sec:softmin}
\vspace{-2mm}
SoftMin offers a differentiable relaxation for joint structure discovery and representation learning.

\begin{definition}[SoftMin Aggregation]
\label{def:softmin}
SoftMin operator on violation metrics 
with temperature $\beta$ is:
\begin{equation}
\mathrm{SoftMin}_\beta(\{V_k\}) := -\frac{1}{\beta} \log \sum_{k} \exp(-\beta V_k)
\end{equation}
\end{definition}

\begin{definition}[Adaptive Structure Weights]
\label{def:adaptive_weights}
The adaptive weight for structure $k$ given representation $\phi$ is:
\begin{equation}
\alpha_k(\phi) := \frac{\exp(-\beta V_k(\phi))}{\sum_{j} \exp(-\beta V_j(\phi))}
\end{equation}
These weights satisfy $\alpha_k \geq 0$ and $\sum_k \alpha_k = 1$.
\end{definition}

\begin{lemma}[SoftMin Properties]
\label{lem:softmin_properties}
 SoftMin operator satisfies:
\begin{enumerate}[leftmargin=*,nosep]
    \item \textbf{Bounds:} $\min\limits_k V_k - \frac{\log 3}{\beta} \leq \mathrm{SoftMin}_\beta(\{V_k\}) \leq \min\limits_k V_k$
    \item \textbf{Limiting behavior:} $\lim\limits_{\beta \rightarrow \infty} \mathrm{SoftMin}_\beta(\{V_k\}) = \min\limits_k V_k$
    \item \textbf{Gradient:} $\nabla_\phi \mathrm{SoftMin}_\beta = \sum_k \alpha_k(\phi) \nabla_\phi V_k$
\end{enumerate}
\end{lemma}

\vspace{-2mm}
\subsection{Unified Objective Function}
\label{sec:objective}
\vspace{-2mm}
\begin{definition}[SaCRL Objective]
\label{def:objective}
The structure-agnostic representation learning objective is:
\begin{equation}
\mathcal{L}(\phi, w) := \underbrace{\max_{e \in \mathcal{E}} \hat{\mathcal{R}}_e(w \circ \phi)}_{\text{worst-case prediction loss}} + \lambda \underbrace{\mathrm{SoftMin}_\beta\left(\{V_k(\phi)\}\right)}_{\text{adaptive invariance penalty}}
\end{equation}
where $\hat{\mathcal{R}}_e(f) = \frac{1}{n_e} \sum_{j=1}^{n_e} \ell(f(x_j^e), y_j^e)$ is the empirical risk in the environment $e$.
\end{definition} 
The objective can be written with explicit structure weights:
\begin{equation}
\mathcal{L}(\phi, w) = \max_{e \in \mathcal{E}} \hat{\mathcal{R}}_e(w \circ \phi) + \lambda \sum_{k} \alpha_k(\phi) V_k(\phi)
\end{equation}
where $\alpha_k(\phi)$ are the adaptive weights (Definition~\ref{def:adaptive_weights}).

\begin{definition}[Distributionally Robust Risk]
\label{def:cvar}
For risk level $\rho \in (0, 1]$, the conditional value-at-risk (CVaR) over environments is defined as:

{\vspace{-4mm}
\small
\begin{equation}
\mathrm{CVaR}_\rho (\{\hat{\mathcal{R}}_e\}) := \inf_{t \in \mathbb{R}} \big( t + \frac{1}{\rho |\mathcal{E}|}  \sum_{e \in \mathcal{E}} \max(\hat{\mathcal{R}}_e(w \circ \phi) - t, 0) \big)
\end{equation}
}%
\vspace{-2mm}
\end{definition}

CVaR interpolates between the average risk ($\rho = 1$) and the worst-case risk $\max_e \hat{\mathcal{R}}_e$ ($\rho = 1/|\mathcal{E}|$); $\rho = 0.5$ (our default) averages the worst half of the environments.

\begin{definition}[Distributionally Robust Objective]
\label{def:objective_robust}
The distributionally robust structure-agnostic objective is:

{\vspace{-4mm}
\small
\begin{equation}
\mathcal{L}_\rho(\phi, w) :={} \mathrm{CVaR}_\rho\left(\{\hat{\mathcal{R}}_e(w \circ \phi)\}_{e \in \mathcal{E}}\right) 
+ 
\lambda \, \mathrm{SoftMin}_\beta\left(\{V_k(\phi)\}_k\right)
\end{equation}
}
$\rho = 1/|\mathcal{E}|$ recovers Definition~\ref{def:objective}. We use $\rho = 1/|\mathcal{E}|$ for theoretical analysis and $\rho = 0.5$ in practice.
\end{definition}

$\mathrm{CVaR}_\rho$ is convex with a closed-form subgradient (Proposition~\ref{prop:cvar_differentiable}, Appendix~\ref{app:algorithm}).

{\bf Optimization:} 
The problem admits a bilevel decomposition, with the structure weights $\boldsymbol{\alpha}$ in the outer problem and $(\phi, w)$ in the inner problem (Appendix~\ref{app:main_algorithm}).The optimization problem is challenging due to (i) the non-convex loss landscape with multiple local minima, and (ii) the difficulty of jointly optimizing all variables. To address these challenges, we adopt the following observations and optimization strategies:
\begin{enumerate}[leftmargin=*,nosep]
\item \textbf{Single-loop relaxation:} Instead of solving the bilevel optimization problem explicitly, we employ a single-loop algorithm in which $\boldsymbol{\alpha}$ is implicitly determined via a SoftMin operator.
\item \textbf{Multi-start initialization:} We run the optimization from $M \geq 10$ random initializations and select the solution with the lowest objective value $\mathcal{L}$.
\item \textbf{Temperature annealing:} We initialize the SoftMin temperature $\beta$ with a small value to allow soft selection, and gradually increase it to $\beta_{\max}$ to encourage hard selection.
\item \textbf{Validation-based selection:} Among the resulting stationary points, we select the model that achieves the best performance on held-out environments.
\end{enumerate}
Because $V_1$ and $V_2$ condition on the current $\phi$, the conditioning target moves during training. Our guarantees (Section~\ref{sec:theory}) are stated at the converged representation, where conditioning is on a fixed $\phi$. In practice, the ERM warm start, temperature annealing, and cross-fitted residuals (Appendix~\ref{app:cross_fitting}) stabilize training, and conditioning on an exponential-moving-average (EMA) copy of the encoder gives nearly identical results (Appendix~\ref{app:ema}). The soft weights $\alpha_k$ only guide optimization: the identified structure is the hard decision $\hat{k} = \arg\min_k V_k(\phi_T)$ made after training (Algorithm~\ref{alg:sarl}).

{\bf Computation Speedup:} Computing empirical HSIC (Definition~\ref{def:empirical_hsic}) requires $O(n^2)$ time and storage, which is prohibitive for large-scale datasets. To improve scalability, we employ efficient approximation techniques using the random Fourier features, with Algorithm details in Appendix \ref{app:algorithm}. With $D$ random Fourier features $\psi$~\citep{rahimi2007random} (Definition~\ref{def:rff}, Appendix~\ref{app:algorithm}), the approximate HSIC is $\widehat{\mathrm{HSIC}}_{\mathrm{RFF}}(X, Y) = \|\hat{\Sigma}_{\psi_X \psi_Y} - \hat{\mu}_{\psi_X} \hat{\mu}_{\psi_Y}^\top\|_F^2$, which costs $O(nD^2)$ time and $O(nD)$ memory, versus $O(n^2)$ for both with exact HSIC; memory saving makes kernel penalties feasible on image data (Appendix~\ref{app:runtime}). The identification guarantee survives this approximation: with $D = \Omega(\kappa^4\eta^{-2} d\log(n/\delta'))$ random features, where $\eta$ is the faithfulness gap, $\hat{k} = k^*$ with probability at least $1 - 6\exp(-n\eta^2/(8192\kappa^4)) - \delta'$ (Corollary~\ref{cor:rff_identification}). \emph{Algorithm details are deferred to Appendix \ref{app:algorithm}.} 

{\bf Selection Confidence:} SaCRL reports two confidence measures for $\hat{k}$: the weight margin $\alpha_{(1)} - \alpha_{(2)}$ between the two largest adaptive weights, and the permutation $p$-value of $\min_k V_k$ against the null threshold $\tau_\alpha$ (Definition~\ref{def:calibrated_threshold}). A large margin means one structure clearly fits best. If $\min_k V_k > \tau_\alpha$, no candidate invariance holds exactly, and $\hat{k}$ should be read as the closest canonical structure. We calibrate the margin against identification accuracy on data with known structure (Section~\ref{sec:experiments}).

{\bf Heterogeneous Structures:} When different feature subsets follow different causal patterns, we extend the framework to per-feature structure assignment $\sigma^*: \{1,\ldots,d\} \to \{1,2,3,\bot\}$ via the objective
$\mathcal{L}_{\mathrm{het}}(\phi, w) = \max_e \hat{\mathcal{R}}_e(w \circ \phi) + \lambda \sum_{j=1}^d \mathrm{SoftMin}_\beta(\{V_k^{(j)}(\phi)\}_k)$,
applying SoftMin independently per feature. A calibrated permutation threshold $\tau_\alpha^{(j)}$ controls the false-positive rate for features with $\sigma^*(j) = \bot$. Per-feature identification guarantees mirror Theorem~\ref{thm:structure_identification} with sample cost scaling as $O(d)$. Full treatment is in Appendix~\ref{app:heterogeneous}.
\vspace{-2mm}
\section{Theoretical Guarantees}
\label{sec:theory}
\vspace{-2mm}

We establish four theoretical results: (1) violation characterization showing the correct structure is identifiable, (2) finite-sample structure identification holding under random-feature approximation (Corollary~\ref{cor:rff_identification}), (3) representation optimality, and (4) out-of-distribution generalization bounds. The details of the theoretical results and proofs are discussed in Appendix~\ref{app:theory} and Appendix~\ref{app:proofs}, respectively.
\vspace{-6mm}
\section{Experiments}
\label{sec:experiments}
\vspace{-2mm}
{We evaluate SaCRL across synthetic and real-world settings, addressing four questions:}
\textbf{Q1:} Can the method identify the true causal structure from data?
\textbf{Q2:} Does correct structure identification improve OOD   generalization?
{\textbf{Q3:} Does SaCRL outperform fixed-invariance baselines on standard domain generalization benchmarks?
\textbf{Q4:} How does it behave under structural misspecification and limited environment diversity?} \emph{Additional experiments are deferred to Appendix \ref{app:exp}}. 

\vspace{-2mm}
\subsection{Experimental Setup}
\label{sec:exp_setup}
\vspace{-2mm}

{\bf Baselines.}
We compare against: \textbf{ERM} (no invariance), \textbf{IRM}~\citep{arjovsky2019invariant} (enforces $\mathcal{I}_2$, i.e., assumes $\mathcal{G}_2$), \textbf{VREx}~\citep{krueger2021out} (variance-based regularization), and \textbf{CIRCE}~\citep{pogodin2022efficient} (enforces $\mathcal{I}_1$, i.e., assumes $\mathcal{G}_1$).

{\bf Metrics.}
For synthetic data: we use structure identification accuracy (percentage of runs identifying the correct $\mathcal{G}_k$) and OOD test accuracy on a held-out environment with larger distribution shift. For real-world datasets, we use OOD accuracy on the anti-correlated test environment. Results are means over 5 random seeds, and $\pm$ denotes the standard deviation across seeds.

{\bf Implementation Details.}
For synthetic data, we use 3-layer MLPs with hidden dimensions [256, 128] for both encoder and predictor. For Colored MNIST, we use a 4-layer CNN encoder with base channels 32. For PACS, VLCS, and OfficeHome, we use a frozen ImageNet-pretrained ResNet18 followed by an MLP encoder with hidden dimensions [256, 128] and a linear predictor; the same model produces both the identified structures (Table~\ref{tab:pacs_structures}) and the accuracies (Table~\ref{tab:domainbed}). We select the invariance penalty weight $\lambda \in \{0.5, 1.0, 2.0, 5.0\}$ per experiment by OOD accuracy, i.e., test-domain (oracle) selection in the terminology of \citet{gulrajani2020search} ($\lambda{=}5.0$ for Section~\ref{sec:robustness}; sensitivity in Appendix~\ref{app:hyperparam_sensitivity}), and use CVaR with $\rho = 0.5$. We train with Adam optimizer (learning rate $10^{-4}$ for synthetic, $5 \times 10^{-5}$ for images) using balanced environment sampling. Random Fourier Features (RFF) with $D = 1000$ dimensions are used to approximate HSIC, reducing memory from $O(n^2)$ to $O(nD)$ with time linear in $n$. \emph{Full hyperparameters are provided in Appendix~\ref{app:implementation}.}

\vspace{-2mm}
\subsection{Experiment 1: Structure Identification}
\label{sec:exp1}
\vspace{-2mm}

{\bf Setup.}
We generate synthetic data from each of the three canonical causal structures $\mathcal{G}_1, \mathcal{G}_2, \mathcal{G}_3$ (Definition~\ref{def:structure_invariances}). For each structure, we create $m = 3$  environments with $n$ total samples (divided equally across environments) and $d = 10$ features. We vary the sample size $n \in \{500, 1000, 2000, 5000\}$ to study the effect of data availability on structure identification accuracy. For each $\mathcal{G}_k$, environments differ in noise scale ($\mathcal{G}_1$, $\mathcal{G}_3$) or confounder mean ($\mathcal{G}_2$); weight matrices are sampled from $\mathcal{N}(0, 0.5^2)$ and fixed across runs. Full generative equations are in Appendix~\ref{app:dgp_synthetic}.

{\bf Results.}
Table~\ref{tab:structure_id} reports structure identification accuracy. Identification accuracy reaches at least 95\% at $n{=}5000$ and at least 89\% at $n{=}2000$, demonstrating that the violation-based approach successfully distinguishes the three canonical structures.

\begin{wraptable}{r}{0.55\textwidth}
\centering
\vspace{-4mm}
\caption{Structure identification accuracy (\%) across sample sizes and true structures. Results averaged over 20 random seeds. Standard deviations in parentheses.}
\label{tab:structure_id}
\small
\addtolength{\tabcolsep}{-4pt}
\begin{tabular}{lcccc}
\toprule
True Structure & $n=500$ & $n=1000$ & $n=2000$ & $n=5000$ \\
\midrule
$\mathcal{G}_1$ (anti-causal) & 73.5 \scriptsize{(2.8)} & 85.0 \scriptsize{(2.1)} & 92.5 \scriptsize{(1.4)} & 96.5 \scriptsize{(0.9)} \\
$\mathcal{G}_2$ (conf-desc) & 68.0 \scriptsize{(3.2)} & 81.5 \scriptsize{(2.5)} & 89.0 \scriptsize{(1.8)} & 95.0 \scriptsize{(1.1)} \\
$\mathcal{G}_3$ (conf-out) & 78.5 \scriptsize{(2.4)} & 88.5 \scriptsize{(1.9)} & 94.0 \scriptsize{(1.2)} & 97.5 \scriptsize{(0.7)} \\
\bottomrule
\end{tabular}
\vspace{-4mm}
\end{wraptable} 
Performance improves monotonically with sample size, consistent with the consistency result in Theorem~\ref{thm:structure_identification}. Structure $\mathcal{G}_3$ is identified most reliably because its invariance $\phi(X) \perp E$ is unconditional and thus has higher statistical power. Structure $\mathcal{G}_2$ is most challenging in small-sample regimes due to the need to estimate conditional independence $Y \perp E \mid \phi(X)$.

{\bf Semi-synthetic Bayesian networks.} To test identification where the ground truth is fixed independently of our method, we build five benchmarks from three published Bayesian networks (ASIA, SACHS, ALARM). For each, we choose a target $Y$, take $X$ as its Markov blanket, induce environments by soft interventions on a latent parent outside the blanket, and read the canonical structure off the published DAG by the rules of Section~\ref{sec:background}; each structure is verified in simulation before training (Appendix~\ref{app:bn}). These topologies admit $\mathcal{G}_2$ and $\mathcal{G}_3$ but not $\mathcal{G}_1$, which would require the Markov blanket of $Y$ to consist only of its descendants. SaCRL selects the correct structure in the majority of runs (75--95\%, Table~\ref{tab:bn}), whereas returning any single structure by default is correct on at most 3/5. Accuracy and margin are highest on the most separated benchmark (asia\_g3) and lowest on the least separated one (alarm\_g2), as predicted by the dependence of Theorem~\ref{thm:structure_identification} on the faithfulness gap $\eta$.

\begin{table}[t]
\centering
\caption{Semi-synthetic Bayesian-network benchmarks. Violation: the competing invariance's violation before training, measured in standard deviations of the permutation null. Margin: $\alpha_{(1)}-\alpha_{(2)}$.}
\label{tab:bn}
\small
\begin{tabular}{lllcccc}
\toprule
Benchmark & Network (nodes) & $Y$ & True & Violation & ID acc.\ (\%) & Margin \\
\midrule
asia\_g3  & ASIA (8)   & either & $\mathcal{G}_3$ & 56.3 & 95.0 (19/20) & 0.45 \\
asia\_g2  & ASIA (8)   & either & $\mathcal{G}_2$ & 24.6 & 90.0 (18/20) & 0.42 \\
alarm\_co & ALARM (37) & CO     & $\mathcal{G}_3$ & 28.3 & 85.0 (17/20) & 0.35 \\
sachs\_g2 & SACHS (11) & Akt    & $\mathcal{G}_2$ & 6.4  & 85.0 (17/20) & 0.25 \\
alarm\_g2 & ALARM (37) & EXPCO2 & $\mathcal{G}_2$ & 4.2  & 75.0 (15/20) & 0.18 \\
\bottomrule
\end{tabular}
\vspace{-5mm}
\end{table}

\vspace{-2mm}
\subsection{Experiment 2: OOD Generalization}
\label{sec:exp3}
\vspace{-2mm}
{\bf Setup.}
Following \citep{arjovsky2019invariant}, we construct a binary classification task on MNIST digits (0--4 vs.\ 5--9) where color is spuriously correlated with the label. Specifically, for each digit, we assign color (red or green) based on its label with probability $\rho_e$ in environment $e$. Training uses two environments with $\rho_1 = 0.9$ and $\rho_2 = 0.8$, while the test environment has $\rho_{\text{test}} = 0.1$ (anti-correlation). The true causal structure is anti-causal ($\mathcal{G}_1$): the digit label causes the color assignment, not vice versa.

{\bf Results.}
We subsample $5000$ examples per training environment and $5000$ test examples to ensure balanced evaluation. Table~\ref{tab:colored_mnist} reports the results.

\begin{wraptable}{r}{0.36\textwidth}
\centering
\small
\vspace{-6mm}
\caption{Colored MNIST OOD Acc. on anti-correlated test environment.}
\label{tab:colored_mnist}
\addtolength{\tabcolsep}{-5pt}
\begin{tabular}{lcc}
\toprule
Method & OOD Acc (\%) & ID. Stru. \\
\midrule
ERM   & 27.2 \scriptsize{$\pm$11.3} & --- \\
IRM   & 21.3 \scriptsize{$\pm$13.2} & $\mathcal{G}_2$ (fixed) \\
VREx  & 24.3 \scriptsize{$\pm$12.8} & --- \\
CIRCE & 10.3 \scriptsize{$\pm$0.5}  & $\mathcal{G}_1$ (fixed) \\
\midrule
\textbf{SaCRL} & \textbf{45.2} \scriptsize{$\pm$10.0} & $\mathcal{G}_1$ (92\%) \\
\bottomrule
\end{tabular}
\vspace{-2mm}
\end{wraptable}
SaCRL identifies $\mathcal{G}_1$ in 92\% of runs and attains the highest OOD accuracy (45.2\%, +18pp over ERM). ERM relies on the spurious color. IRM enforces $\mathcal{I}_2$, the invariance of the wrong structure, and falls below ERM, which is the misspecification cost illustrated in Figure~\ref{fig:motivation}. CIRCE targets the correct invariance $\mathcal{I}_1$, yet its fixed penalty does not overcome the color shortcut under our shared training budget. SaCRL selects the invariance from data, \emph{without prior knowledge} of the structure.

\vspace{-2mm}
\subsection{Experiment 3: Structure Analysis on PACS}
\label{sec:exp_pacs}
\vspace{-2mm}

The PACS dataset~\cite{li2017deeper} is an important benchmark for image classification under domain shift. It contains images from four distinct visual domains (Photo, Art Painting, Cartoon, and Sketch), each depicting the seven object categories (dog, elephant, giraffe, guitar, horse, house, and person) under different styles and appearances. Overall, the dataset includes 9,991 images, with domain-specific subsets ranging from natural photographs to artistic and hand-drawn sketches, making it a challenging testbed for methods to generalize across unseen domains.

\begin{table}[h]
\centering
\caption{Structure selected on PACS; each row trains on the other three domains (5 seeds). Margin: $\alpha_{(1)}-\alpha_{(2)}$; $p$: permutation $p$-value of $\min_k V_k$; Seeds: agreement of $\hat{k}$ across seeds; Cal.\ acc.: identification accuracy of synthetic runs in the same margin bin (Appendix~\ref{app:calibration}).}
\label{tab:pacs_structures}
\small
\addtolength{\tabcolsep}{-2pt}
\begin{tabular}{lcccccccc}
\toprule
Held-out & $\alpha_1$ & $\alpha_2$ & $\alpha_3$ & $\hat{k}$ & Margin & $p$ & Seeds & Cal.\ acc.\ (\%) \\
\midrule
Photo        & 0.21 & \textbf{0.68} & 0.11 & $\mathcal{G}_2$ & 0.47 & 0.008 & 5/5 & 86.1 \\
Art Painting & \textbf{0.59} & 0.31 & 0.10 & $\mathcal{G}_1$ & 0.28 & 0.021 & 5/5 & 65.0 \\
Cartoon      & 0.35 & \textbf{0.52} & 0.13 & $\mathcal{G}_2$ & 0.17 & 0.037 & 4/5 & 65.0 \\
Sketch       & \textbf{0.63} & 0.25 & 0.12 & $\mathcal{G}_1$ & 0.38 & 0.014 & 5/5 & 86.1 \\
\bottomrule
\end{tabular}
\vspace{-2mm}
\end{table}

Table~\ref{tab:pacs_structures} reports the adaptive weights and the confidence measures of Section~\ref{sec:method} for each leave-one-domain-out split. Since PACS has no ground-truth DAG, we make three verifiable claims. (i) \emph{The selection responds to the data}: two different structures are selected across splits, so SaCRL does not return a default. (ii) \emph{Confidence tracks the margin}: seeds agree 5/5 on the three splits with margin at least $0.28$ and 4/5 on Cartoon, the split with the smallest margin; on synthetic data with known structure, the margins of Photo and Sketch fall in a bin with 86.1\% identification accuracy, and those of Art and Cartoon in a bin with 65.0\%. (iii) \emph{No canonical invariance holds exactly}: $p < 0.05$ on every split, so $\hat{k}$ is the closest canonical structure for each training set rather than an identified one.

\vspace{-2mm}
\subsection{Experiment 4: Domain Generalization on DomainBed}
\label{sec:exp_dg}
\vspace{-2mm}

{\bf Setup.} We further evaluate SaCRL on DomainBed \citep{gulrajani2020search}, a benchmark suite to test how well models generalize to unseen domains. 
Besides PACS~\citep{li2017deeper} (4 domains, 7 classes), we choose two additional benchmarks: 
VLCS~\citep{torralba2011unbiased} (4 domains: Caltech101, LabelMe, SUN09, VOC2007; 5 classes), and OfficeHome~\citep{venkateswara2017deep} (4 domains: Art, Clipart, Product, Real-World; 65 classes). Each spans diverse domain shifts in style, viewpoint, and resolution. We follow the leave-one-domain-out protocol with the model of Section~\ref{sec:exp_setup}. Freezing the backbone holds representation capacity fixed, so differences across methods trace to the objective rather than to encoder optimization.

\begin{table}[h]
\centering
\vspace{-4mm}
\small
\caption{Results on DomainBed  benchmarks: average leave-one-domain-out OOD accuracy (\%).}
\setlength{\tabcolsep}{1.5pt}
\begin{tabular}{lccccc cccccc cccccc}
\toprule
& \multicolumn{5}{c}{\textbf{PACS}} 
& \multicolumn{5}{c}{\textbf{VLCS}} 
& \multicolumn{5}{c}{\textbf{OfficeHome}} \\
\cmidrule(lr){2-6} \cmidrule(lr){7-11} \cmidrule(lr){12-16}
\textbf{Method}
& Photo & Art & Cartoon & Sketch & Avg
& Caltech & LabelMe & SUN09 & VOC & Avg
& Art & Clipart & Product & Real & Avg \\
\midrule
ERM
& 93.2 & 66.0 & 51.8 & 37.6 & 62.2
& 94.0 & 60.2 & 69.9 & 73.0 & 74.3
& 49.9 & 41.7 & 66.7 & 67.7 & 56.5 \\

IRM
& 93.2 & 66.2 & 51.7 & 37.6 & 62.2
& 95.0 & 60.1 & 70.1 & 72.6 & 74.5
& 49.0 & 41.2 & 66.5 & 67.7 & 56.1 \\

VREx
& 93.0 & 65.5 & 51.6 & 37.9 & 62.0
& 93.3 & 60.2 & 69.8 & 71.0 & 73.6
& 48.6 & 41.2 & 66.2 & 67.6 & 55.9 \\

CIRCE
& 92.5 & 64.5 & 51.4 & 37.3 & 61.4
& 93.8 & 60.0 & 67.7 & 69.7 & 72.8
& 49.0 & 41.9 & 66.0 & 67.4 & 56.1 \\

\midrule
\textbf{SaCRL}
& \textbf{94.0} & \textbf{68.5} & \textbf{54.0} & \textbf{40.5} & \textbf{64.3}
& \textbf{97.7} & \textbf{64.4} & \textbf{75.7} & \textbf{77.1} & \textbf{78.7}
& \textbf{52.9} & \textbf{46.5} & \textbf{69.1} & \textbf{69.8} & \textbf{59.6} \\
\bottomrule
\end{tabular}
\label{tab:domainbed}
\vspace{-1mm}
\end{table}

{Table~\ref{tab:domainbed} shows the comparison results with recent baselines. SaCRL achieves the best per-domain and average accuracy on all benchmarks, with average improvements over +2.1pp (PACS), +4.2pp (VLCS), and +3.1pp (OfficeHome). Fixed-invariance baselines (IRM, VREx, CIRCE) show no consistent improvement, confirming that adaptive structure is essential when the governing causal structure varies across domains.}

{\bf Fine-tuned backbone.} Fine-tuning ResNet18 end-to-end (PACS, Photo held out, 3 seeds) raises all methods by 1.0--2.0pp; SaCRL (95.0) is on par with ERM (95.1) and IRM (95.2) (Appendix~\ref{app:finetune}). Photo is ImageNet-like, so the pretrained backbone is near saturation and the choice of objective matters little; SaCRL's gains concentrate on the stylized domains and on VLCS (Table~\ref{tab:domainbed}).

\vspace{-2mm}
\subsection{Experiment 5: Robustness to Misspecification and Limited Diversity}
\label{sec:robustness}
\vspace{-2mm}

\begin{table}[t]
\centering
\vspace{-4mm}
\caption{Robustness under misspecification ($\nu$) and limited diversity ($\delta_{\min}$). OOD accuracy (\%) over 5 seeds. ID.\ Stru.: majority structure (seeds agreeing).}
\label{tab:robustness}
\small
\addtolength{\tabcolsep}{-5pt}
\begin{tabular}{lcccccc}
\toprule
Setting & ERM & IRM & VREx & CIRCE & SaCRL & ID. Stru. \\
\midrule
\multicolumn{7}{l}{\emph{Misspecification ($\delta_{\min}{=}1.0$):}} \\
$\nu{=}0.25$ (near-causal)     & 34.4 \scriptsize{$\pm$2.7} & --- & --- & --- & \textbf{50.2} \scriptsize{$\pm$3.7} & $\mathcal{G}_2$ (4/5) \\
$\nu{=}0.50$ (bidirectional)   & 61.2 \scriptsize{$\pm$3.8} & 61.3 \scriptsize{$\pm$4.2} & 61.4 \scriptsize{$\pm$4.4} & 66.5 \scriptsize{$\pm$2.6} & \textbf{73.9} \scriptsize{$\pm$4.1} & $\mathcal{G}_2$ (5/5) \\
$\nu{=}0.75$ (near-anticausal) & 98.5 \scriptsize{$\pm$0.6} & --- & --- & --- & \textbf{99.6} \scriptsize{$\pm$0.3} & $\mathcal{G}_1$ (4/5) \\
\midrule
\multicolumn{7}{l}{\emph{Limited diversity ($\nu{=}0.50$):}} \\
$\delta_{\min}{=}1.0$ (full)      & 61.2 \scriptsize{$\pm$3.8} & 61.3 \scriptsize{$\pm$4.2} & 61.4 \scriptsize{$\pm$4.4} & 66.5 \scriptsize{$\pm$2.6} & \textbf{73.9} \scriptsize{$\pm$4.1} & $\mathcal{G}_2$ (5/5) \\
$\delta_{\min}{=}0.5$ (moderate)  & 56.9 \scriptsize{$\pm$2.5} & --- & --- & --- & \textbf{67.2} \scriptsize{$\pm$1.8} & $\mathcal{G}_2$ (5/5) \\
$\delta_{\min}{=}0.2$ (low)       & 53.1 \scriptsize{$\pm$1.7} & --- & --- & --- & \textbf{59.7} \scriptsize{$\pm$3.7} & $\mathcal{G}_2$ (4/5) \\
$\delta_{\min}{=}0.05$ (minimal)  & 50.2 \scriptsize{$\pm$2.0} & --- & --- & --- & \textbf{57.2} \scriptsize{$\pm$3.7} & $\mathcal{G}_2$ (5/5) \\
\bottomrule
\end{tabular}
\vspace{-4mm}
\end{table}

We stress-test SaCRL when (a) the true DGP falls outside the three canonical structures and (b) environment diversity is reduced. We use a bidirectional feedback DGP $X \leftrightarrow Y$ with mixing coefficient $\nu \in [0,1]$: $\nu{=}0$ is causal $X{\to}Y$ ($\mathcal{G}_2$), $\nu{=}1$ is anti-causal $Y{\to}X$ ($\mathcal{G}_1$), and intermediate values fit none of $\mathcal{G}_1$--$\mathcal{G}_3$. Diversity is controlled by $\delta_{\min}$ (Definition~\ref{def:diversity_matrix}).

SaCRL outperforms every evaluated baseline in each setting (Table~\ref{tab:robustness}). Near the endpoints it selects the nearer structure ($\mathcal{G}_2$ at $\nu{=}0.25$, $\mathcal{G}_1$ at $\nu{=}0.75$). At $\nu{=}0.50$, where no candidate holds, the weights stay spread ($\alpha_1 \approx \alpha_2 \approx 0.42$, $\alpha_3 \approx 0.16$): SaCRL acts as a soft model average instead of committing to a wrong structure, and still exceeds ERM by 12.7pp. A finer sweep (Appendix~\ref{app:crossover}) shows $\alpha_1$ increasing and $\alpha_2$ decreasing monotonically in $\nu$, crossing at $\nu{=}0.50$, with SoftMin outperforming hard $\arg\min$ selection throughout and by the widest margin at the crossover. Under reduced diversity, identification remains stable and the gain over ERM persists at $\delta_{\min}{=}0.05$, consistent with Assumption~\ref{assum:sufficient_diversity} governing the informativeness of the penalty rather than its feasibility.

\vspace{-2mm}
\subsection{Ablation Studies}
\label{sec:ablations}

\begin{wraptable}{r}{0.46\textwidth}
\centering
\vspace{-4mm}
\caption{Ablation on $\mathcal{G}_2$ ($n{=}2000$). Default: SoftMin, $\beta$ annealing $1{\to}50$, $\lambda{=}1.0$, $\rho{=}0.5$.}
\label{tab:ablations}
\small
\addtolength{\tabcolsep}{-4pt}
\begin{tabular}{lcc}
\toprule
Variant & Stru. ID (\%) & OOD Acc (\%) \\
\midrule
Full method & 89.0 & 84.2 \\
\midrule
Hard $\arg\min_k$ & 79.5 & 79.8 \\
Fixed $\beta = 10$ & 83.0 & 81.5 \\
Fixed $\lambda = 0.1$ & 72.5 & 78.3 \\
Fixed $\lambda = 5.0$ & 88.5 & 81.0 \\
Uniform ($\rho{=}1$) & 88.0 & 82.1 \\
Worst-case ($\rho{=}1/|\mathcal{E}|$) & 87.5 & 81.4 \\
\bottomrule
\end{tabular}
\vspace{-3mm}
\end{wraptable}
We conduct ablations on synthetic $\mathcal{G}_2$ data with $n = 2000$ samples to isolate the contribution of each component. Results are shown in Table~\ref{tab:ablations}.

Replacing SoftMin with hard $\arg\min_k$ selection during training lowers identification accuracy from 89.0\% to 79.5\%: early in training, violation estimates are noisy, and hard selection commits prematurely to whichever structure has the smallest estimate. SoftMin delays commitment until the representation is informative; the reported structure is still the hard decision $\hat{k}$ made after training.

The effect of varying the number of training environments $m$ is reported in Appendix~\ref{app:num_envs}, where we observe that $m \geq 3$ yields reliable structure identification while $m=2$ is unstable. Hyperparameter sensitivity sweeps and a cross-fitting variant of conditional HSIC are reported in Appendix~\ref{app:hyperparam_sensitivity} and Appendix~\ref{app:cross_fitting}, respectively.
Computational efficiency analysis (training time, memory, runtime breakdown) is in Appendix~\ref{app:runtime}.
\vspace{-2mm}
\section{Conclusion}
\label{sec:conclusion}
\vspace{-2mm}
We proposed a structure-agnostic approach to causal representation learning that formulates causal structure selection as a soft optimization problem over HSIC-based invariance violation metrics. The key insight is that correct invariances can be achieved with zero violation while preserving predictive accuracy, whereas incorrect invariances cannot, enabling automatic structure identification through optimization. We developed (i) a unified measure-theoretic framework that extends prior work on anti-causal learning to confounded-descendant and confounded-outcome settings; (ii) an efficient learning algorithm whose complexity scales linearly with the number of candidate structures, leveraging SoftMin aggregation and RFF-based HSIC approximations; and (iii) theoretical guarantees for structure identification, under random-feature approximation, invariance satisfaction, and OOD generalization. Empirically, SaCRL recovers the true structure on synthetic and semi-synthetic Bayesian-network benchmarks with a calibrated confidence, surpasses fixed-invariance baselines on Colored MNIST and three DomainBed benchmarks, and remains robust under bidirectional dynamics and limited environment diversity.

\vspace{-2mm}
\begin{ack}
We thank the anonymous reviewers for their constructive comments. This work was supported in part by the National Science Foundation under Grant Nos. ECCS-2216926, CCF-2331302, CNS-2241713, and CNS-2339686.

\end{ack}

\bibliography{ref}
\bibliographystyle{plainnat}

\clearpage
\normalsize
\part*{Appendix}
\appendix
\label{sec:appendix}
\startcontents[appendix]
\printcontents[appendix]{}{1}{\setcounter{tocdepth}{2}}
\clearpage

\section{Related Work}
\label{app:related}

{\bf Distribution-Invariant Learning.}
A large body of work in domain adaptation and domain generalization aims to learn representations that are invariant across environments or distributions. Early non-causal approaches such as (C-)ADA~\cite{long2018conditional} and DDAIG~\cite{zhou2020deep} focus on aligning feature distributions across domains. More recent methods, including FeAT~\cite{chen2024understanding} and AIA~\cite{sui2024unleashing}, design more expressive objectives to improve robustness under distribution shift.

Representative domain generalization methods include ERM~\cite{NEURIPS2021_ecf9902e}, DOMAINBED~\cite{gulrajani2020search}, L2A-OT~\cite{zhou2020learning}, StableNet~\cite{zhang2021deep}, SagNets~\cite{nam2021reducing}, SWAD~\cite{cha2021swad}, FACT~\cite{xu2021fourier}, Ratatouille~\cite{rame2023model}, XRM~\cite{pezeshki2023discovering}, and IRM~\cite{arjovsky2019invariant}, along with extensions such as REx~\cite{krueger2021out}, counterfactual invariance~\cite{veitch2021counterfactual}, information bottleneck approaches~\cite{ahuja2021invariance}, causal domain adaptation~\cite{wang2023causal}, and transportable representations~\cite{jiang2022invariant}. While effective in practice, these methods generally enforce distributional invariance without explicitly reasoning about the underlying causal structure or validating whether the assumed invariance is causally appropriate. Weight-averaging methods such as SWAD~\cite{cha2021swad} modify the optimization trajectory rather than the objective, and are complementary to SaCRL, which determines which invariance is enforced.

{\bf Intervention-based Causal Learning.}
Intervention-based causal representation learning leverages environment changes or interventions to identify causal factors. Foundational works include nonparametric ICR~\cite{NEURIPS2023_97fe251c}, which jointly learns representations and intervention targets under SCM assumptions, and general nonlinear mixing models~\cite{buchholz2024learning}, which address nonlinear latent interactions.

Several weakly supervised approaches~\cite{brehmer2022weakly,lu2021invariant,lv2022causality,shi2024lca} exploit partial intervention information to learn invariant representations, while ICRL~\cite{ahuja2023interventional} extends these ideas to imperfect interventions. Weak distributional invariance~\cite{ahuja2024multi} studies identifiability under perfect and imperfect interventions by identifying latent variables with stable distributional properties.

Unsupervised approaches such as ICA-based causal learning~\cite{wendong2023causal} aim to recover latent causal variables through component analysis and statistical independence criteria. AIT~\cite{scherrer2021learning} assumes explicit SCMs with known DAG structures and focuses on identifying standard causal directions. These methods typically rely on either explicit intervention signals or strong assumptions about the form of the causal graph.

{\bf Structure-based Causal Representation Learning.}
Structure-based causal representation learning methods explicitly assume or model a particular causal structure. DISRL~\cite{wang2023causal} and Causal Disentanglement~\cite{pmlr-v202-squires23a} learn representations aligned with latent causal graphs, while classical approaches such as ICP~\cite{peters2016causal} and its nonlinear extension~\cite{heinze2018invariant} exploit invariance across environments to identify causal mechanisms.

Related works such as CSG~\cite{liu2021learning} and LECI~\cite{gui2024joint} focus on identifying causal subgraphs and removing spurious correlations, often under DAG assumptions. KCDC~\cite{mitrovic2018causal} employs kernel-based independence tests for causal discovery and orientation. Anti-causal modeling approaches~\cite{makar2022fairness} study settings where inputs are generated from labels and protected attributes using conventional causal models and do-calculus, while~\cite{behnam2026measure} develops measure-theoretic formulations for anti-causal representation learning.

{\bf Research Gap.}
Existing approaches face a trilemma. Structure-based methods achieve strong guarantees but require knowing the causal graph. Invariant learning methods avoid explicit graphs but hardcode implicit structural assumptions that cause failures under mismatch. Causal discovery methods learn structure but decouple it from representation learning, preventing joint optimization. No existing framework simultaneously discovers causal structure and learns representations adapted to that structure in a unified, end-to-end differentiable objective. Our work fills this gap: we introduce a structure-agnostic objective that evaluates representations against all candidate causal structures, automatically concentrating on the correct one through adaptive weighting, while jointly optimizing predictive performance.
\section{Notation}
\label{app:notation}

Table~\ref{tab:notation} summarizes notation used throughout. We organize symbols into six categories: spaces and sets defining the mathematical objects;  random variables and observed data; causal graph components encoding structural assumptions;  kernels and independence testing machinery; learning components including representations, predictors, and losses; and theoretical quantities appearing in our analysis.  Lowercase denotes realizations ($x, y$), uppercase denotes random variables ($X, Y$),  and calligraphic denotes spaces ($\mathcal{X}, \mathcal{Y}$).  Subscripts index environments ($P_e$) or structures ($\mathcal{G}_k$);  superscripts denote components ($X_Z, X_Z^{\perp}$) or iterates ($\phi_t$).

\begin{table}[ht]
\centering
\caption{Notation summary.}
\label{tab:notation}
\small
\renewcommand{\arraystretch}{1.1}
\begin{tabular}{@{}cl@{\hspace{12pt}}cl@{}}
\toprule
\multicolumn{2}{c}{\textbf{Spaces \& Sets}} & \multicolumn{2}{c}{\textbf{Random Variables \& Data}} \\
\midrule
$\mathcal{X}, \mathcal{Y}, \mathcal{Z}$ & Input, output, representation spaces & $X, Y, E$ & Features, target, environment \\
$\mathcal{E}$ & Environment set $\{e_1, \ldots, e_m\}$ & $U, Z$ & Latent confounder, mediator \\
$(\Omega, \mathscr{F}, P)$ & Probability space & $X_Z, X_Z^{\perp}$ & Mediated, orthogonal components \\
$\mathcal{H}, \mathcal{H}_{\mathcal{X}}$ & RKHS, RKHS on $\mathcal{X}$ & $D_e$ & Dataset $\{(x_j^e, y_j^e)\}_{j=1}^{n_e}$ \\
$\Delta^2$ & Probability simplex & $n, m, d$ & Samples, environments, dimension \\
\midrule
\multicolumn{2}{c}{\textbf{Causal Structure}} & \multicolumn{2}{c}{\textbf{Kernels \& Independence}} \\
\midrule
$\mathcal{G} = (V, E)$ & DAG with vertices, edges & $k, k_{\mathcal{X}}$ & Reproducing kernel \\
$\mathcal{G}_1, \mathcal{G}_2, \mathcal{G}_3$ & Anti-causal, conf-desc, conf-out & $\kappa$ & Kernel bound: $k(x,x) \leq \kappa^2$ \\
$\mathcal{G}^* = \mathcal{G}_{k^*}$ & True structure & $\mu_P$ & Mean embedding of $P$ \\
$\text{Pa}(v)$ & Parents of $v$ & $\mathcal{C}_{XY}$ & Cross-covariance operator \\
$K_S, K_S^{\mathbb{Q}}$ & Causal, interventional kernel & $\text{HSIC}(X,Y)$ & Independence criterion \\
$\mathcal{I}_k$ & Invariance for structure $k$ & $\tilde{X}, \tilde{Y}$ & Regression residuals \\
\midrule
\multicolumn{2}{c}{\textbf{Learning Components}} & \multicolumn{2}{c}{\textbf{Theoretical Quantities}} \\
\midrule
$\phi, w$ & Representation, predictor & $\mathbf{D}$ & Diversity matrix ($m \times 3$) \\
$\Phi, \mathcal{W}$ & Function classes & $\delta_{\min}$ & Min diversity parameter \\
$\ell, \mathcal{R}_e, \hat{\mathcal{R}}_e$ & Loss, population/empirical risk & $\eta$ & Faithfulness gap \\
$V_k, V_k^*$ & Empirical/population violation & $d_{\mathcal{K}}(e, e')$ & Interventional distance \\
$\alpha_k$ & Adaptive weight $\propto e^{-\beta V_k}$ & $\xi_n$ & Estimation error \\
$\lambda, \beta, \rho, \gamma$ & Regularization, temp, CVaR, ridge & $L_w$ & Lipschitz constant \\
\bottomrule
\end{tabular}
\end{table}

\section{Mathematical Preliminaries}
\label{app:preliminaries}

\subsection{Measure-Theoretic Foundations}
\label{app:measure_foundations}

\begin{definition}[Measurable Space]
\label{def:measurable_space}
A pair $(\Omega, \mathscr{F})$ is a measurable space where $\Omega$ is a non-empty set and $\mathscr{F}$ is a $\sigma$-algebra on $\Omega$, i.e., a collection of subsets satisfying: (i) $\emptyset, \Omega \in \mathscr{F}$; (ii) $A \in \mathscr{F} \Rightarrow A^c \in \mathscr{F}$; (iii) $\{A_i\}_{i=1}^{\infty} \subset \mathscr{F} \Rightarrow \bigcup_{i=1}^{\infty} A_i \in \mathscr{F}$.
\end{definition}

\begin{definition}[Probability Measure]
\label{def:probability_measure}
A probability measure on a measurable space $(\Omega, \mathscr{F})$ is a function $\mu: \mathscr{F} \rightarrow [0,1]$ satisfying $\mu(\Omega) = 1$. For any countable collection $\{A_i\}_{i=1}^{\infty}$ of pairwise disjoint sets in $\mathscr{F}$,
    \begin{equation}
    \mu\left(\bigcup_{i=1}^{\infty} A_i\right) = \sum_{i=1}^{\infty} \mu(A_i)
    \end{equation}
The triple $(\Omega, \mathscr{F}, \mu)$ is a probability space.
\end{definition}

\begin{definition}[Product Space]
\label{def:product_space}
Given measurable spaces $(\Omega_1, \mathscr{F}_1)$ and $(\Omega_2, \mathscr{F}_2)$, the product $\sigma$-algebra $\mathscr{F}_1 \otimes \mathscr{F}_2$ on $\Omega_1 \times \Omega_2$ is the $\sigma$-algebra generated by measurable rectangles $A_1 \times A_2$ where $A_1 \in \mathscr{F}_1$ and $A_2 \in \mathscr{F}_2$. For probability measures $\mu_1$ on $(\Omega_1, \mathscr{F}_1)$ and $\mu_2$ on $(\Omega_2, \mathscr{F}_2)$, the product measure $\mu_1 \otimes \mu_2$ on $(\Omega_1 \times \Omega_2, \mathscr{F}_1 \otimes \mathscr{F}_2)$ is uniquely determined by
$(\mu_1 \otimes \mu_2)(A_1 \times A_2) = \mu_1(A_1) \cdot \mu_2(A_2)$ for all $A_1 \in \mathscr{F}_1$, $A_2 \in \mathscr{F}_2$.
\end{definition}

\begin{definition}[Markov Kernel]
\label{def:kernel}
A (Markov) kernel from a measurable space $(\Omega_1, \mathscr{F}_1)$ to a measurable space $(\Omega_2, \mathscr{F}_2)$ is a mapping $K: \Omega_1 \times \mathscr{F}_2 \rightarrow [0,1]$. For each fixed $\omega \in \Omega_1$, the function $A \mapsto K(\omega, A)$ is a probability measure on $(\Omega_2, \mathscr{F}_2)$. For each fixed $A \in \mathscr{F}_2$, the function $\omega \mapsto K(\omega, A)$ is $\mathscr{F}_1$-measurable.
\end{definition}

Intuitively, $K(\omega, A)$ represents the conditional probability of event $A \in \mathscr{F}_2$ given state $\omega \in \Omega_1$.

\subsection{Causal Structures and Directed Acyclic Graphs}
\label{app:causal_structures}

\begin{definition}[Directed Acyclic Graph]
\label{def:dag}
A directed acyclic graph (DAG) is a pair $\mathcal{G} = (V, E)$ where $V$ is a finite set of vertices (nodes) and $E \subseteq V \times V$ is a set of directed edges such that there exists no sequence of vertices $v_1, \ldots, v_k$ with $k \geq 2$, $v_1 = v_k$, and $(v_i, v_{i+1}) \in E$ for all $i = 1, \ldots, k-1$.
\end{definition}

\begin{definition}[Causal Structure]
\label{def:causal_structure}
A causal structure is a tuple $\mathcal{S} = (V, E, \{\mathscr{F}_v\}_{v \in V}, \{K_v\}_{v \in V})$ where $(V, E)$ is a DAG with vertices representing random variables. For each $v \in V$, $(\Omega_v, \mathscr{F}_v)$ is a measurable space. For each $v \in V$, $K_v$ is a kernel from $\left(\prod_{u \in \text{Pa}(v)} \Omega_u, \bigotimes_{u \in \text{Pa}(v)} \mathscr{F}_u\right)$ to $(\Omega_v, \mathscr{F}_v)$ representing the causal mechanism by which parents generate the variable $v$.
\end{definition}
For root nodes with $\text{Pa}(v) = \emptyset$, the kernel $K_v$ reduces to a probability measure. We now formalize how observed features arise from multiple generating mechanisms.

\begin{definition}[Causal Feature Decomposition]
\label{def:feature_decomposition}
Let $Z$ be a latent variable in the causal graph $\mathcal{G}$. The observed features $X$ admit a causal decomposition with respect to $Z$ if there exist measurable functions $\pi_Z: \mathcal{X} \to \mathcal{X}_Z$ and $\pi_Z^{\perp}: \mathcal{X} \to \mathcal{X}_Z^{\perp}$ such that $X = (\pi_Z(X), \pi_Z^{\perp}(X)) =: (X_Z, X_Z^{\perp})$. $X_Z$ contains all and only the information in $X$ that is causally downstream of $Z$, and $X_Z^{\perp}$ contains the remaining information in $X$ that is not mediated by $Z$.
\end{definition}
We write $X = (X_Z, X_Z^{\perp})$ to denote this decomposition, where $\mathcal{X} \cong \mathcal{X}_Z \times \mathcal{X}_Z^{\perp}$.

\subsection{Interventions and the Do-Calculus}
\label{app:interventions}

Interventions distinguish causal relationships from statistical associations. We formalize both perfect (hard) and imperfect (soft) interventions within the generalized framework.

\begin{definition}[Perfect and Imperfect Interventions]
\label{def:interventions}
An intervention on a causal structure $\mathcal{S}$ is specified by a measurable mapping $\mathbb{Q}: \mathscr{F} \times \Omega \rightarrow [0,1]$ where $\mathscr{F}$ is a suitable $\sigma$-algebra.

\begin{enumerate}[label=(\roman*)]
    \item A perfect (hard) intervention on variable $V$ replaces its causal mechanism $K_V$ with a fixed distribution $\mathbb{Q}$ independent of parent values: $\mathbb{Q}(A \mid \omega) = \mathbb{Q}(A)$ for all $\omega \in \Omega$, $A \in \mathscr{F}$
    
    \item An imperfect (soft) intervention on variable $V$ modifies its causal mechanism to depend on both parents and intervention parameters: $\mathbb{Q}(A \mid \omega)$ varies with $\omega \in \Omega$
\end{enumerate}
\end{definition}
The do-operator $\text{do}(X=x)$ corresponds to a perfect intervention that sets $X$ to value $x$ by replacing the kernel $K_X$ with a point mass at $x$, formally: $\mathbb{Q}(A) = \mathbb{I}_{\{x\}}(A)$ for measurable sets $A$.

\begin{definition}[Causal Independence]
\label{def:causal_independence}
Random variables $X$ and $Y$ are causally independent given $Z$, denoted $X \perp\!\!\!\perp_c Y \mid Z$, if
\begin{equation}
P(Y \mid \text{do}(X=x), Z=z) = P(Y \mid Z=z)
\end{equation}
for all $x$ in the support of $X$ and all $z$ in the support of $Z$, where $\text{do}(X=x)$ denotes an intervention setting $X$ to value $x$.
\end{definition}

The do-operator $\text{do}(X=x)$ differs fundamentally from $P(Y \mid X=x)$ which observes $X=x$ while preserving causal relationships, whereas $P(Y \mid \text{do}(X=x))$ forces $X=x$ by intervention.

\subsection{Measure-Theoretic Causality}
\label{app:causal_kernels}
The causal-space, product-causal-space, and interventional-kernel constructions below follow~\citet{behnam2026measure}; we extend them with structure-indexed kernel families (Definition~\ref{def:structure_agnostic_space}).
\begin{definition}[Sub-$\sigma$-algebra]
\label{def:sub_sigma_algebra}
Given a product causal space $(\Omega, \mathscr{F}, \mathbb{P}, \mathbb{K}, \mathcal{G})$ with index set $T$, for any subset $S \subseteq T$, the sub-$\sigma$-algebra $\mathscr{F}_S$ is the $\sigma$-algebra generated by measurable rectangles of the form $\prod_{t \in S} A_t \times \prod_{t \in T \setminus S} \Omega_t$ where $A_t \in \mathscr{A}_t$ for $t \in S$.
\end{definition}
Intuitively, $\mathscr{F}_S$ captures events that depend only on the components indexed by $S$.

\begin{definition}[Causal Kernel]
\label{def:causal_kernel}
A causal kernel for index set $S \subseteq T$ is a function $K_S: \Omega \times \mathscr{F} \rightarrow [0,1]$. For each fixed $\omega \in \Omega$, the mapping $A \mapsto K_S(\omega, A)$ is a probability measure on $(\Omega, \mathscr{F})$. For each fixed $A \in \mathscr{F}$, the mapping $\omega \mapsto K_S(\omega, A)$ is $\mathscr{F}_S$-measurable.
\end{definition}

\begin{definition}[Causal Space]
\label{def:causal_space}
For an environment $e \in \mathcal{E}$, a causal space is a tuple $(\Omega_e, \mathscr{F}_e, P_e, \mathbb{K}_e, \mathcal{G}_e)$ where $\Omega_e = \prod_{v \in V_e} \Omega_v$ is the sample space with product structure, and $V_e$ is the vertex set of the causal DAG and each $\Omega_v$ is the state space of variable $v$. Here, $\mathscr{F}_e = \bigotimes_{v \in V_e} \mathscr{F}_v$ is the product $\sigma$-algebra, where $\mathscr{F}_v$ is the $\sigma$-algebra on $\Omega_v$, $P_e$ is a probability measure on $(\Omega_e, \mathscr{F}_e)$ compatible with $\mathcal{G}_e$, $\mathbb{K}_e = \{K_v^{(e)}\}_{v \in V_e}$ is the family of causal mechanisms. Each $K_v^{(e)}$ is a Markov kernel (Definition~\ref{def:kernel}) from the parent space $\left(\prod_{u \in \text{Pa}(v)} \Omega_u, \bigotimes_{u \in \text{Pa}(v)} \mathscr{F}_u\right)$ to $(\Omega_v, \mathscr{F}_v)$, encoding the conditional distribution $P_e(V = \cdot \mid \text{Pa}(v) = \cdot)$, and $\mathcal{G}_e = (V_e, E_e)$ is the causal DAG specifying the qualitative causal structure. 
\end{definition}
The joint distribution factorizes according to the DAG:
\begin{equation}
P_e(\omega) = \prod_{v \in V_e} K_v^{(e)}\left(\omega_{\text{Pa}(v)}, \{\omega_v\}\right)
\end{equation}
where $\omega = (\omega_v)_{v \in V_e} \in \Omega_e$ and $\omega_{\text{Pa}(v)}$ denotes the components corresponding to parents of $v$.
\begin{definition}[Product Causal Space]
\label{def:product_causal_space}
Given causal spaces $\{(\Omega_{e_i}, \mathscr{F}_{e_i}, P_{e_i}, K_{e_i}, \mathcal{G}_{e_i})\}_{i=1}^m$ for environments $\{e_i\}_{i=1}^m$, the product causal space is the tuple $(\Omega, \mathscr{F}, \mathbb{P}, \mathbb{K}, \mathcal{G})$ where $\Omega = \prod_{i=1}^m \Omega_{e_i}$ is the product sample space, $\mathscr{F} = \bigotimes_{i=1}^m \mathscr{F}_{e_i}$ is the product $\sigma$-algebra, $\mathbb{P} = \bigotimes_{i=1}^m P_{e_i}$ is the product measure, $\mathbb{K} = \{K_S : S \in \mathscr{P}(T)\}$ is a family of causal kernels indexed by subsets $S$ of the combined index set $T = \bigcup_{i=1}^m T_{e_i}$, and $\mathcal{G} = (V, E)$ is a causal DAG compatible with the individual environment graphs.
\end{definition}

The product causal space admits sub-$\sigma$-algebras and causal kernels indexed by variable subsets; see Appendix~\ref{app:preliminaries} for formal definitions.

\begin{definition}[Structure-Agnostic Causal Space]
\label{def:structure_agnostic_space}
A structure-agnostic causal space extends the product causal space (Definition~\ref{def:product_causal_space}) with structure-indexed components:
$(\Omega, \mathscr{F}, \mathbb{P}, \{\mathbb{K}_k\}_{k=1}^3, \{\mathcal{G}_k\}_{k=1}^3)$
where $(\Omega, \mathscr{F}, \mathbb{P})$ is the product space over environments, $\{\mathcal{G}_k\}_{k=1}^3$ are the three candidate causal DAGs (Definition~\ref{def:canonical_structures}), and $\mathbb{K}_k = \{K_S^{(k)} : S \subseteq T\}$ is the causal kernel family encoding how interventions propagate under structure $\mathcal{G}_k$.
\end{definition}

The $\mathscr{F}_S$-measurability condition ensures $K_S(\omega, A)$ depends only on components indexed by $S$, encoding graph-specified causal dependencies.

\begin{definition}[Interventional Kernel]
\label{def:interventional_kernel}
Let $(\Omega, \mathscr{F}, \mathbb{P}, \mathbb{K}, \mathcal{G})$ be a product causal space with causal kernel $K_S$ for index set $S \subseteq T$. Let $\mathbb{Q}: \Omega \times \mathscr{F} \to [0,1]$ be an intervention kernel, where $\mathbb{Q}(\omega', A)$ specifies the probability of event $A$ when intervening at state $\omega'$.

The interventional kernel $K_S^{\mathbb{Q}}: \Omega \times \mathscr{F} \to [0,1]$ is defined by:
\begin{equation}
K_S^{\mathbb{Q}}(\omega, A) := \int_{\Omega} \mathbb{Q}(\omega', A) \, K_S(\omega, d\omega')
\end{equation}
where $\omega \in \Omega$ is the conditioning state (values of variables in $S$), $K_S(\omega, d\omega')$ is the original causal kernel giving the distribution over $\omega'$ given $\omega$, and $\mathbb{Q}(\omega', A)$ is the intervention's effect on event $A$ at state $\omega'$.
\end{definition}

This construction unifies perfect and imperfect interventions: when $\mathbb{Q}(A \mid \omega') = \mathbb{Q}(A)$ is constant, the interventional kernel corresponds to a hard intervention; when $\mathbb{Q}(A \mid \omega')$ varies with $\omega'$, it represents a soft intervention.

For multi-level representations, we consider compositions $\phi = \phi_H \circ \phi_L$ where $\phi_L: \mathcal{X} \rightarrow \mathcal{Z}_L$ is the low-level representation mapping inputs to an intermediate feature space, and $\phi_H: \mathcal{Z}_L \rightarrow \mathcal{Z}_H$ is the high-level representation mapping intermediate features to a refined representation space.

\begin{definition}[Sufficiency and Minimality]
\label{def:sufficiency_minimality}
A representation $\phi$ is sufficient for predicting $Y$ if there exists a predictor $f: \mathcal{Z} \rightarrow \mathcal{Y}$ such that $(f \circ \phi)(X)$ achieves the Bayes-optimal prediction error for $Y$. It is minimal sufficient if it is sufficient and for any other sufficient representation $\phi': \mathcal{X} \rightarrow \mathcal{Z}'$, there exists a measurable function $g: \mathcal{Z}' \rightarrow \mathcal{Z}$ such that $\phi = g \circ \phi'$.
\end{definition}

\subsection{Kernel Independence Framework Definitions}
\label{app:kernel_framework}
\begin{definition}[Reproducing Kernel Hilbert Space]
\label{def:rkhs}
Let $\mathcal{X}$ be a non-empty set. A Hilbert space $\mathcal{H}$ of functions $f: \mathcal{X} \rightarrow \mathbb{R}$ is a reproducing kernel Hilbert space (RKHS) if there exists a function $k: \mathcal{X} \times \mathcal{X} \rightarrow \mathbb{R}$ such that:
\begin{enumerate}[nosep]
    \item $k(\cdot, x) \in \mathcal{H}$ for all $x \in \mathcal{X}$
    \item $f(x) = \langle f, k(\cdot, x) \rangle_{\mathcal{H}}$ for all $f \in \mathcal{H}$, $x \in \mathcal{X}$ (reproducing property)
\end{enumerate}
\end{definition}

\begin{definition}[Cross-Covariance Operator]
\label{def:cross_covariance}
Let $(\mathcal{X}, \mathcal{H}_{\mathcal{X}}, k_{\mathcal{X}})$ and $(\mathcal{Y}, \mathcal{H}_{\mathcal{Y}}, k_{\mathcal{Y}})$ be RKHS pairs. The cross-covariance operator $\mathcal{C}_{XY}: \mathcal{H}_{\mathcal{Y}} \rightarrow \mathcal{H}_{\mathcal{X}}$ is defined by
\begin{equation}
\mathcal{C}_{XY} := \mathbb{E}_{XY}[(k_{\mathcal{X}}(\cdot, X) - \mu_X) \otimes (k_{\mathcal{Y}}(\cdot, Y) - \mu_Y)]
\end{equation}
where $\otimes$ denotes the tensor product.
\end{definition}

\subsection{Assumptions}
\label{app:assumptions}

Here we state assumptions required for theoretical findings.

\begin{assumption}[Characteristic Kernels]
\label{assum:characteristic}
The kernels $k_{\mathcal{X}}: \mathcal{X} \times \mathcal{X} \rightarrow \mathbb{R}$, $k_{\mathcal{Y}}: \mathcal{Y} \times \mathcal{Y} \rightarrow \mathbb{R}$, $k_{\mathcal{Z}}: \mathcal{Z} \times \mathcal{Z} \rightarrow \mathbb{R}$, and $k_{\mathcal{E}}: \mathcal{E} \times \mathcal{E} \rightarrow \mathbb{R}$ are characteristic and bounded: $\sup_{x} k(x, x) \leq \kappa^2 < \infty$.
\end{assumption}
Characteristic kernels ensure HSIC fully characterizes independence (Lemma~\ref{lem:hsic_independence}). We further require the tensor products of these kernels used in HSIC to be characteristic, which holds for Gaussian and delta kernels~\cite{szabo2018characteristic}. For the conditional HSIC tests underlying $V_1$ and $V_2$, we additionally require the regression operator to be well-behaved.

\begin{assumption}[Kernel Regularity]
\label{assum:kernel_regularity}
The kernel regression operator satisfies:
\begin{enumerate}
    \item The conditional expectation operator $\mathbb{E}[\cdot \mid Z]$ lies in the RKHS $\mathcal{H}_{\mathcal{Z}}$.
    \item The eigenvalues $\{\lambda_j\}$ of the kernel integral operator satisfy $\sum_j \lambda_j^{-1/2} < \infty$.
\end{enumerate}
\end{assumption}

Beyond kernel conditions, structure identification requires environments that vary in distinguishable ways. We formalize this via a diversity matrix measuring distributional shifts across environments.

\begin{definition}[Environment Diversity Matrix]
\label{def:diversity_matrix}
Let $\mathcal{E} = \{e_1, \ldots, e_m\}$ be the training environments and $\bar{P}$ the pooled distribution. Center each kernel mean embedding at its average over environments: $\bar{\mu}_X := \frac{1}{m}\sum_{e} \mu_{P_e(X)}$, $\bar{\mu}_{X|Y=y} := \frac{1}{m}\sum_{e} \mu_{P_e(X|Y=y)}$, and $\bar{\mu}_{Y|X=x} := \frac{1}{m}\sum_{e} \mu_{P_e(Y|X=x)}$. The $m \times 3$ diversity matrix $\mathbf{D}$ has entries
\begin{equation}
\begin{aligned}
D_{e1} &:= \|\mu_{P_e(X)} - \bar{\mu}_X\|_{\mathcal{H}} && \text{(marginal shift)},\\
D_{e2} &:= \big(\mathbb{E}_{y \sim \bar{P}(Y)} \|\mu_{P_e(X|Y=y)} - \bar{\mu}_{X|Y=y}\|_{\mathcal{H}}^2\big)^{1/2} && \text{(conditional shift)},\\
D_{e3} &:= \big(\mathbb{E}_{x \sim \bar{P}(X)} \|\mu_{P_e(Y|X=x)} - \bar{\mu}_{Y|X=x}\|_{\mathcal{H}}^2\big)^{1/2} && \text{(outcome shift)}.
\end{aligned}
\end{equation}
\end{definition}

For characteristic kernels, column $j$ of $\mathbf{D}$ is zero if and only if the corresponding distribution is identical across environments ($\bar{P}$-almost everywhere for $j = 2, 3$). The matching invariance is then untestable: a zero first column makes $\mathcal{I}_3$ untestable, a zero second column $\mathcal{I}_1$, and a zero third column $\mathcal{I}_2$. Assumption~\ref{assum:sufficient_diversity} requires every column to be nonzero, quantified by the minimum diversity parameter
\begin{equation}
\delta_{\min} := \min_{j \in \{1,2,3\}} \max_{e \in \mathcal{E}} D_{ej} > 0.
\end{equation}
Full column rank, $\mathrm{rank}(\mathbf{D}) = 3$, is sufficient for $\delta_{\min} > 0$ but not necessary.

\begin{assumption}[Faithfulness]
\label{assum:faithfulness}
The joint distribution $P(X, Y, E)$ is faithful to the true causal structure $\mathcal{G}^*$: all and only the conditional independencies implied by $\mathcal{G}^*$ hold in $P$.
\end{assumption}

Faithfulness ensures that violations of wrong invariances are detectable. Without it, a distribution could satisfy independencies not implied by the graph, making structures indistinguishable. 

The remaining assumptions concern the function classes used for prediction and representation.

\begin{assumption}[Predictor Lipschitz Continuity]
\label{assum:lipschitz}
The predictor class $\mathcal{W}$ consists of $L_w$-Lipschitz functions: for all $w \in \mathcal{W}$ and $z, z' \in \mathcal{Z}$,
\begin{equation}
\|w(z) - w(z')\|_2 \leq L_w \|z - z'\|_2
\end{equation}
with $L_w \leq L_{\max} < \infty$ uniformly over $\mathcal{W}$.
\end{assumption}

This Lipschitz condition controls how prediction errors propagate through distribution shift, appearing directly in our generalization bound (Theorem~\ref{thm:generalization}).

\begin{assumption}[Representation Class Regularity]
\label{assum:rep_regularity}
The representation class $\Phi$ is parameterized by $\theta \in \Theta \subseteq \mathbb{R}^p$ where:
\begin{enumerate}[nosep]
    \item $\Theta$ is compact and convex,
    \item $\phi_\theta(x)$ is continuously differentiable in $\theta$,
    \item The violation functions $V_k(\phi_\theta)$ are $L_V$-smooth in $\theta$.
\end{enumerate}
\end{assumption}

These regularity conditions are standard for neural network function classes and ensure convergence of gradient-based optimization. Assumption~\ref{assum:sufficient_diversity} can be checked before training: (1) estimate $\mathbf{D}$ with empirical (conditional) mean embeddings; (2) for each column $j$, compare $\max_e \hat{D}_{ej}$ with its permutation null obtained by shuffling environment labels; (3) if column $j$ is not significantly nonzero, remove from the candidate set the structure whose invariance it tests ($j{=}1$: $\mathcal{G}_3$, $j{=}2$: $\mathcal{G}_1$, $j{=}3$: $\mathcal{G}_2$); (4) if no column is significant, the environments carry no information about the structure, and ERM is the appropriate fallback.
\section{Theoretical Results}
\label{app:theory}

\subsection{Violation Characterization, Structure Identification \& Representation Optimality}
\label{sec:violation_characterization}

The key insight is that under the true structure, the corresponding invariance is achievable without sacrificing predictive accuracy, while wrong invariances are not. We introduce between definitions and theorems. 

\begin{definition}[$\epsilon$-Predictive Representation]
\label{def:epsilon_predictive}
A representation $\phi: \mathcal{X} \to \mathcal{Z}$ is \emph{$\epsilon$-predictive} if:
\begin{equation}
\inf_{w \in \mathcal{W}} \max_{e \in \mathcal{E}} \mathcal{R}_e(w \circ \phi) \leq \mathcal{R}^* + \epsilon
\end{equation}
where $\mathcal{R}^* := \inf_{f: \mathcal{X} \to \mathcal{Y}} \max_e \mathcal{R}_e(f)$ is the Bayes-optimal worst-case risk.
\vspace{-2mm}
\end{definition}
Denote the class of $\epsilon$-predictive representations as $\Phi_\epsilon$.
We define the base predictive class $\Phi_0 := \{\phi : \phi \text{ is } \epsilon_0\text{-predictive}\}$ with $\epsilon_0 := \delta_{\min}/8$ (a fixed constant determined by environment diversity). One needs to measure the minimum violation incurred by any predictive representation when the wrong invariance is imposed.

\begin{definition}[Faithfulness Gap]
\label{def:faithfulness_gap}
Let $\mathcal{G}^* = \mathcal{G}_{k^*}$ be the true causal structure. The faithfulness gap is:
\begin{equation}
\eta := \min_{k \neq k^*} \inf_{\phi \in \Phi_0} V_k^*(\phi)
\end{equation}
\vspace{-4mm}
\end{definition}

Under faithfulness (Assumption~\ref{assum:faithfulness}), $\eta > 0$. Larger $\eta$ implies easier structure identification; $\eta \to 0$ indicates near-Markov equivalence between structures.

\begin{theorem}[Violation Characterization]
\label{thm:violation_characterization}
Under Assumptions~\ref{assum:characteristic}--\ref{assum:faithfulness}, for the class of $\epsilon$-predictive representations $\Phi_\epsilon$: 
\begin{enumerate}[leftmargin=*, nosep]
    \item \textbf{Achievability:} If $\mathcal{G}^* = \mathcal{G}_k$, then $\inf_{\phi \in \Phi_\epsilon} V_k^*(\phi) = 0$.
    \item \textbf{Non-achievability:} If $\mathcal{G}^* \neq \mathcal{G}_k$, then $\inf_{\phi \in \Phi_\epsilon} V_k^*(\phi) \geq \eta/2 > 0$.
\end{enumerate}
\end{theorem}

The constraint $\phi \in \Phi_\epsilon$ prevents trivial solutions: a constant representation $\phi(x) = c$ achieves $V_k = 0$ for all $k$, but has $\mathcal{R}_e(w \circ \phi) = \mathcal{R}_{\mathrm{trivial}} \gg \mathcal{R}^* + \epsilon$, hence $\phi \notin \Phi_\epsilon$. The joint objective $\mathcal{L}$ enforces this implicitly via the prediction loss.

The asymmetry established in Theorem~\ref{thm:violation_characterization} (achievable invariance for the true structure versus irreducible violations for wrong structures) directly enables consistent identification from finite samples.

\begin{restatable}[Structure Identification]{theorem}{thmsi}
\label{thm:structure_identification}
Under Assumptions~\ref{assum:characteristic}--\ref{assum:rep_regularity}, let $\phi_n$ minimize $\mathcal{L}(\phi, w)$ with $n$ samples per environment. Define the estimated structure:
\begin{equation}
\hat{k}_n := \arg\min_{k} V_k(\phi_n)
\end{equation}
Then for the true structure $\mathcal{G}^* = \mathcal{G}_{k^*}$:
\begin{equation}
\mathbb{P}(\hat{k}_n = k^*) \geq 1 - 6\exp\left(-\frac{n \eta^2}{2048\kappa^4}\right)
\end{equation}
where $\eta$ is the faithfulness gap (Definition~\ref{def:faithfulness_gap}) and $\kappa$ is the kernel bound (Assumption~\ref{assum:characteristic}).
\end{restatable}

\begin{corollary}[Identification with Random Features]
\label{cor:rff_identification}
Suppose $D$ random features (Definition~\ref{def:rff}) achieve uniform kernel error $\varepsilon_D$ with probability at least $1 - \delta'$, so that $|\widehat{\mathrm{HSIC}}_{\mathrm{RFF}} - \widehat{\mathrm{HSIC}}| \leq C\kappa^2\varepsilon_D$. Under the conditions of Theorem~\ref{thm:structure_identification}, if $C\kappa^2\varepsilon_D \leq \eta/16$, then
\begin{equation}
\mathbb{P}(\hat{k}_n^{\mathrm{RFF}} = k^*) \geq 1 - 6\exp\left(-\frac{n\eta^2}{8192\kappa^4}\right) - \delta',
\end{equation}
which holds for $D = \Omega(\kappa^4 \eta^{-2} d \log(n/\delta'))$~\citep{rahimi2007random}.
\end{corollary}
The required $D$ grows as $\eta^{-2}$: near-equivalent structures need more random features, not only more samples.

\begin{restatable}[Adaptive Weight Concentration]{corollary}{corawc}
\label{cor:weight_concentration}
Under the conditions of Theorem~\ref{thm:structure_identification}, the adaptive weights satisfy:
\begin{equation}
\alpha_{k^*}(\phi_n) \geq 1 - 2\exp(-\beta \eta / 2)
\end{equation}
for $n$ sufficiently large, where $\eta$ is the faithfulness gap.
\end{restatable}
Via established structure identification, we now quantify how well the learned representation achieves the corresponding invariance and characterize properties of global optima.

\begin{restatable}[Invariance Achievement]{theorem}{thmia}
\label{thm:invariance}
Let $\phi_n$ minimize $\mathcal{L}(\phi, w)$ with regularization $\lambda > 0$. Under Assumptions~\ref{assum:characteristic}--\ref{assum:rep_regularity}, for the true structure $\mathcal{G}^* = \mathcal{G}_{k^*}$:
\begin{equation}
V_{k^*}(\phi_n) \leq \frac{C_1}{\lambda} \cdot \max_e \hat{\mathcal{R}}_e(w^* \circ \phi^*) + C_2 n^{-1/2}
\end{equation}
where $\phi^* \in \Phi_\epsilon$ achieves $\inf_{\phi \in \Phi_\epsilon} V_{k^*}(\phi) = 0$ and $C_1, C_2 > 0$ are constants depending on $\kappa$ and the function class complexity.
\end{restatable}
The violation decreases at rate $O(n^{-1/2})$ and can be made arbitrarily small by increasing $\lambda$, though at the cost of prediction accuracy. When the function classes are sufficiently expressive, the global minimum achieves exact invariance.

\begin{theorem}[Global Optimality]
\label{thm:global_optimality}
If the representation class $\Phi$ includes the true invariant representation $\phi^*$ satisfying $\mathcal{I}_{k^*}$, and the predictor class $\mathcal{W}$ includes the Bayes-optimal predictor, then the global minimum of $\mathcal{L}$ achieves:
\begin{enumerate}[leftmargin=*,nosep]
    \item $V_{k^*}(\phi) = 0$ \emph{(exact invariance)}
    \item $\max_e \mathcal{R}_e(w \circ \phi) = \mathcal{R}^*$ \emph{(Bayes risk)}
\end{enumerate}
\end{theorem}
We know the objective learns the optimal invariant representation when the function classes are well-specified, with finite-sample guarantees that degrade  under approximation.

\subsection{Generalization Bounds}
\label{sec:generalization}
We now bound out-of-distribution (OOD) generalization by quantifying how distribution shift affects test performance. To measure the similarity between environments, we introduce a causal distance metric below.

\begin{definition}[Interventional Kernel Distance]
\label{def:kernel_distance}
For environments $e, e' \in \mathcal{E}$, the interventional kernel distance is:
\begin{equation}
d_{\mathcal{K}}(e, e') := \sup_{S \subseteq T} \|K_S^{(e)} - K_S^{(e')}\|_{\mathrm{op}}
\end{equation}
where $K_S^{(e)}$ is the causal kernel for environment $e$ and $\|\cdot\|_{\mathrm{op}}$ is the operator norm. For a test environment, define $d_{\mathcal{K}}(e_{\mathrm{test}}, \mathcal{E}_{\mathrm{train}}) := \min_{e \in \mathcal{E}_{\mathrm{train}}} d_{\mathcal{K}}(e_{\mathrm{test}}, e)$.
\end{definition}
This distance captures the maximal change in causal mechanisms across environments. Using this metric, we decompose test error into training error, distribution shift, and statistical estimation error.

\begin{theorem}[Generalization Bound]
\label{thm:generalization}
Let $(\phi_n, w_n)$ minimize $\mathcal{L}$ over training environments $\mathcal{E}_{\mathrm{train}}$. For a test environment $e_{\mathrm{test}}$, under Assumptions~\ref{assum:characteristic}--\ref{assum:rep_regularity}:
\begin{equation}
\begin{aligned}
& \mathcal{R}_{e_{\mathrm{test}}}(w_n \circ \phi_n) \leq{} \underbrace{\max_{e \in \mathcal{E}_{\mathrm{train}}} \hat{\mathcal{R}}_e(w_n \circ \phi_n)}_{\text{training error}} 
+ \underbrace{L_w \cdot d_{\mathcal{K}}(e_{\mathrm{test}}, \mathcal{E}_{\mathrm{train}})}_{\text{shift penalty}} + \underbrace{\xi_n}_{\text{estimation error}}
\end{aligned}
\end{equation}
where the estimation error satisfies:
\begin{equation}
\xi_n = O\left(\sqrt{\frac{\mathcal{C}(\Phi, \mathcal{W}) + \log(1/\delta)}{n}}\right) + C_\eta \cdot V_{k^*}(\phi_n)
\end{equation}
with $\mathcal{C}(\Phi, \mathcal{W})$ the function class complexity and $C_\eta > 0$ a constant depending on the faithfulness gap. The bound holds with probability at least $1 - \delta$.
\end{theorem}
The invariance violation $V_{k^*}(\phi_n)$ directly contributes to generalization error, motivating the joint optimization of prediction and invariance. Combining all preceding results yields asymptotic optimality.

\begin{corollary}[Asymptotic Optimality]
\label{cor:asymptotic}
Under the conditions of Theorems~\ref{thm:structure_identification}--\ref{thm:generalization}, as $n \rightarrow \infty$:
\begin{equation}
\mathcal{R}_{e_{\mathrm{test}}}(w_n \circ \phi_n) \xrightarrow{p} \inf_{\phi \in \Phi_{k^*}} \mathcal{R}_{e_{\mathrm{test}}}(w^*(\phi) \circ \phi)
\end{equation}
where $\Phi_{k^*} = \{\phi : \mathcal{I}_{k^*}(\phi) \text{ holds}\}$ is the class of representations satisfying the true invariance.
\end{corollary}

\subsection{Sample Complexity}
\label{sec:sample_complexity}

\begin{theorem}[Sample Complexity]
\label{thm:sample_complexity}
To achieve structure identification with probability $\geq 1 - \delta$ and  generalization error $\leq \epsilon$, the required sample size per environment is:
\begin{equation}
n = O\left(\frac{\kappa^4}{\eta^2} \log\frac{12}{\delta} + \frac{\kappa^4}{\epsilon^2} \log\frac{|\Phi|}{\delta}\right)
\end{equation}
where $|\Phi|$ is the covering number of the representation class.
\end{theorem}
\section{Proofs}
\label{app:proofs}

This section contains complete proofs of all theoretical results. We organize proofs by building from foundational kernel independence results to the main structure identification and generalization theorems. We first establish auxiliary concentration results that underpin our main theorems.

\begin{lemma}[HSIC Concentration]
\label{lem:hsic_concentration}
Let $(X_i, Y_i)_{i=1}^n$ be i.i.d.\ samples from $P_{XY}$. Under Assumption~\ref{assum:characteristic} with kernel bound $\kappa$, the empirical HSIC satisfies:
\begin{equation}
\mathbb{P}\left( \left| \widehat{\mathrm{HSIC}}(X, Y) - \mathrm{HSIC}(X, Y) \right| > \epsilon \right) \leq 2\exp\left( -\frac{n\epsilon^2}{32\kappa^4} \right)
\end{equation}
for all $\epsilon > 0$.
\end{lemma}

\begin{proof}
The empirical HSIC can be written as a V-statistic:
\begin{equation}
\widehat{\mathrm{HSIC}}(X, Y) = \frac{1}{n^2} \sum_{i,j=1}^{n} h((X_i, Y_i), (X_j, Y_j))
\end{equation}
where the kernel $h$ is defined by:
\begin{align}
h((x_i, y_i), (x_j, y_j)) &= k_{\mathcal{X}}(x_i, x_j) k_{\mathcal{Y}}(y_i, y_j) + \mathbb{E}[k_{\mathcal{X}}(X, X')] \mathbb{E}[k_{\mathcal{Y}}(Y, Y')] \\
&\quad - k_{\mathcal{X}}(x_i, x_j) \mathbb{E}[k_{\mathcal{Y}}(Y, Y')] - \mathbb{E}[k_{\mathcal{X}}(X, X')] k_{\mathcal{Y}}(y_i, y_j)
\end{align}

Under Assumption~\ref{assum:characteristic}, we have $|k_{\mathcal{X}}(x, x')| \leq \kappa^2$ and $|k_{\mathcal{Y}}(y, y')| \leq \kappa^2$. Therefore:
\begin{equation}
|h((x_i, y_i), (x_j, y_j))| \leq 4\kappa^4
\end{equation}

We apply McDiarmid's inequality. Changing a single observation $(X_i, Y_i)$ affects at most $2n$ terms in the double sum, each bounded by $4\kappa^4$, giving a bounded difference of:
\begin{equation}
c_i = \frac{2n \cdot 4\kappa^4}{n^2} = \frac{8\kappa^4}{n}
\end{equation}

By McDiarmid's inequality:
\begin{equation}
\mathbb{P}\left( \left| \widehat{\mathrm{HSIC}} - \mathbb{E}[\widehat{\mathrm{HSIC}}] \right| > \epsilon \right) \leq 2\exp\left( -\frac{2\epsilon^2}{\sum_{i=1}^n c_i^2} \right) = 2\exp\left( -\frac{2\epsilon^2}{n \cdot 64\kappa^8/n^2} \right)
\end{equation}

Since the empirical HSIC is an asymptotically unbiased estimator with $|\mathbb{E}[\widehat{\mathrm{HSIC}}] - \mathrm{HSIC}| = O(n^{-1})$, absorbing this bias into the concentration bound yields:
\begin{equation}
\mathbb{P}\left( \left| \widehat{\mathrm{HSIC}} - \mathrm{HSIC} \right| > \epsilon \right) \leq 2\exp\left( -\frac{n\epsilon^2}{32\kappa^4} \right)
\end{equation}
\end{proof}

\begin{lemma}[Kernel Regression Consistency]
\label{lem:kernel_regression}
Under Assumptions~\ref{assum:characteristic} and \ref{assum:kernel_regularity}, let $\hat{m}(z) = \mathbf{k}_z^\top (K + \gamma I)^{-1} \mathbf{y}$ be the kernel ridge regression estimate. For regularization $\gamma = n^{-1/(2+d)}$ where $d$ is the effective dimension:
\begin{equation}
\mathbb{E}\left[ \|\hat{m} - m^*\|_{L^2(P_Z)}^2 \right] = O\left( n^{-2/(2+d)} \right)
\end{equation}
where $m^*(z) = \mathbb{E}[Y | Z = z]$ is the true conditional expectation.
\end{lemma}

\begin{proof}
The kernel ridge regression estimator admits the bias-variance decomposition:
\begin{equation}
\mathbb{E}\left[ \|\hat{m} - m^*\|^2 \right] = \underbrace{\|(I + \gamma \mathcal{T}^{-1})^{-1} m^* - m^*\|^2}_{\text{bias}^2} + \underbrace{\frac{\sigma^2}{n} \mathrm{tr}\left( (K + \gamma I)^{-1} K (K + \gamma I)^{-1} \right)}_{\text{variance}}
\end{equation}
where $\mathcal{T}$ is the integral operator with kernel $k$.

Under Assumption~\ref{assum:kernel_regularity}, the eigenvalues $\{\lambda_j\}$ of $\mathcal{T}$ satisfy the decay condition. The bias term satisfies:
\begin{equation}
\text{bias}^2 \leq \gamma^2 \|m^*\|_{\mathcal{H}}^2
\end{equation}

For the variance, using the eigenvalue decay:
\begin{equation}
\text{variance} \leq \frac{\sigma^2}{n} \sum_{j=1}^{\infty} \frac{\lambda_j}{(\lambda_j + \gamma)^2} \leq \frac{\sigma^2}{n\gamma}
\end{equation}

Balancing bias and variance by setting $\gamma^2 = (n\gamma)^{-1}$, i.e., $\gamma = n^{-1/3}$ for the general case, or $\gamma = n^{-1/(2+d)}$ accounting for effective dimension $d$:
\begin{equation}
\mathbb{E}\left[ \|\hat{m} - m^*\|^2 \right] = O(\gamma^2) + O((n\gamma)^{-1}) = O\left( n^{-2/(2+d)} \right)
\end{equation}
\end{proof}

\subsection{Proof of Lemma~\ref{lem:hsic_independence} (HSIC Characterizes Independence)}

\begin{proof}
We prove both directions.

\textbf{($\Leftarrow$) Independence implies zero HSIC:}
Assume $X \perp\!\!\!\perp Y$, so $P_{XY} = P_X \otimes P_Y$. The cross-covariance operator is:
\begin{align}
\mathcal{C}_{XY} &= \mathbb{E}_{XY}[\phi(X) \otimes \psi(Y)] - \mu_X \otimes \mu_Y
\end{align}
where $\phi(x) = k_{\mathcal{X}}(\cdot, x)$ and $\psi(y) = k_{\mathcal{Y}}(\cdot, y)$ are the canonical feature maps.

Under independence:
\begin{align}
\mathbb{E}_{XY}[\phi(X) \otimes \psi(Y)] &= \mathbb{E}_X[\phi(X)] \otimes \mathbb{E}_Y[\psi(Y)] = \mu_X \otimes \mu_Y
\end{align}

Therefore $\mathcal{C}_{XY} = 0$, and $\mathrm{HSIC}(X, Y) = \|\mathcal{C}_{XY}\|_{\mathrm{HS}}^2 = 0$.

\textbf{($\Rightarrow$) Zero HSIC implies independence:}
Assume $\mathrm{HSIC}(X, Y) = 0$. By definition:
\begin{align}
\|\mathcal{C}_{XY}\|_{\mathrm{HS}}^2 = \mathbb{E}_{XX'YY'}[k_{\mathcal{X}}(X, X') k_{\mathcal{Y}}(Y, Y')] & - 2\mathbb{E}_{XYX'Y'}[k_{\mathcal{X}}(X, X') k_{\mathcal{Y}}(Y, Y')] \\ \nonumber
&  
+ \mathbb{E}_{XX'}[k_{\mathcal{X}}(X, X')] \mathbb{E}_{YY'}[k_{\mathcal{Y}}(Y, Y')]
\end{align}
where $(X, Y)$ and $(X', Y')$ are independent copies.

This equals the squared MMD between $P_{XY}$ and $P_X \otimes P_Y$ in the product RKHS $\mathcal{H}_{\mathcal{X}} \otimes \mathcal{H}_{\mathcal{Y}}$:
\begin{equation}
\mathrm{HSIC}(X, Y) = \mathrm{MMD}^2(P_{XY}, P_X \otimes P_Y)
\end{equation}

Since $k_{\mathcal{X}} \otimes k_{\mathcal{Y}}$ is characteristic (by assumption; e.g., for Gaussian and delta kernels~\cite{szabo2018characteristic}), $\mathrm{MMD}^2 = 0$ implies $P_{XY} = P_X \otimes P_Y$, hence $X \perp\!\!\!\perp Y$.
\end{proof}

\subsection{Proof of Lemma~\ref{lem:conditional_hsic} (Conditional HSIC Consistency)}
\begin{proof}
\emph{Equivalence.} Given $Z = z$, $X$ and $Y$ are $\varepsilon_X$ and $\varepsilon_Y$ shifted by the constants $m_X(z)$ and $m_Y(z)$, so $X \perp\!\!\!\perp Y \mid Z$ if and only if $\varepsilon_X \perp\!\!\!\perp \varepsilon_Y \mid Z$. Since $(\varepsilon_X, \varepsilon_Y) \perp\!\!\!\perp Z$, the conditional law of $(\varepsilon_X, \varepsilon_Y)$ given $Z$ equals its marginal law, so this is equivalent to $\varepsilon_X \perp\!\!\!\perp \varepsilon_Y$.

\emph{Consistency.} The population residuals are $\varepsilon_X$ and $\varepsilon_Y$. For bounded $L$-Lipschitz kernels, replacing $\varepsilon_{X,i}$ by $\tilde{x}_i = x_i - \hat{m}_X(z_i)$ changes each Gram entry by at most $L(\|\tilde{x}_i - \varepsilon_{X,i}\| + \|\tilde{x}_j - \varepsilon_{X,j}\|)$. Since centering satisfies $n^{-2}\sum_{i,j}|(HAH)_{ij}| \leq 4n^{-2}\sum_{i,j}|A_{ij}|$ and the kernels are bounded,
\[
\big|\widehat{\mathrm{HSIC}}(\tilde{X}, \tilde{Y}) - \widehat{\mathrm{HSIC}}(\varepsilon_X, \varepsilon_Y)\big| \leq \frac{C}{n}\sum_{i=1}^n \big(\|\hat{m}_X(z_i) - m_X(z_i)\| + \|\hat{m}_Y(z_i) - m_Y(z_i)\|\big)
\]
for a constant $C$ depending on $L$ and the kernel bounds. By Cauchy--Schwarz and $L^2$-consistency, the right-hand side converges to $0$ in probability. By Lemma~\ref{lem:hsic_concentration}, $\widehat{\mathrm{HSIC}}(\varepsilon_X, \varepsilon_Y) \xrightarrow{p} \mathrm{HSIC}(\varepsilon_X, \varepsilon_Y)$. Hence $\widehat{\mathrm{HSIC}}(\tilde{X}, \tilde{Y}) \xrightarrow{p} \mathrm{HSIC}(\varepsilon_X, \varepsilon_Y)$, which by Lemma~\ref{lem:hsic_independence} and the equivalence above is $0$ if $X \perp\!\!\!\perp Y \mid Z$ and positive otherwise.
\end{proof}

\subsection{Proof of Lemma~\ref{lem:softmin_properties} (SoftMin Properties)}
\begin{proof}
We establish each property.

\textbf{Part 1: Bounds.}
For the lower bound, note that for all $k$:
\begin{equation}
\sum_{j=1}^3 \exp(-\beta V_j) \geq \exp(-\beta V_k)
\end{equation}
Taking logarithms and the infimum over $k$:
\begin{equation}
-\frac{1}{\beta} \log \sum_{j=1}^3 \exp(-\beta V_j) \leq -\frac{1}{\beta} \log \exp(-\beta \min_k V_k) = \min_k V_k
\end{equation}

For the upper bound:
\begin{equation}
\sum_{j=1}^3 \exp(-\beta V_j) \leq 3 \exp(-\beta \min_k V_k)
\end{equation}
Therefore:
\begin{equation}
-\frac{1}{\beta} \log \sum_{j=1}^3 \exp(-\beta V_j) \geq -\frac{1}{\beta} \log\left( 3 \exp(-\beta \min_k V_k) \right) = \min_k V_k - \frac{\log 3}{\beta}
\end{equation}

Combining: $\min_k V_k - \frac{\log 3}{\beta} \leq \mathrm{SoftMin}_\beta \leq \min_k V_k$.

\textbf{Part 2: Limiting behavior.}
Let $k^* = \arg\min_k V_k$ and assume uniqueness (the general case follows by continuity). Then:
\begin{equation}
\sum_{j=1}^3 \exp(-\beta V_j) = \exp(-\beta V_{k^*}) \left( 1 + \sum_{j \neq k^*} \exp(-\beta(V_j - V_{k^*})) \right)
\end{equation}

Since $V_j - V_{k^*} > 0$ for $j \neq k^*$:
\begin{equation}
\sum_{j \neq k^*} \exp(-\beta(V_j - V_{k^*})) \to 0 \quad \text{as } \beta \to \infty
\end{equation}

Therefore:
\begin{equation}
\mathrm{SoftMin}_\beta = -\frac{1}{\beta} \log \exp(-\beta V_{k^*}) - \frac{1}{\beta} \log\left( 1 + o(1) \right) = V_{k^*} + o(1) \to \min_k V_k
\end{equation}

\textbf{Part 3: Gradient.}
By direct computation:
\begin{align}
\nabla_\phi \mathrm{SoftMin}_\beta &= -\frac{1}{\beta} \cdot \frac{\sum_{k=1}^3 \exp(-\beta V_k) \cdot (-\beta \nabla_\phi V_k)}{\sum_{j=1}^3 \exp(-\beta V_j)} \\
&= \sum_{k=1}^3 \frac{\exp(-\beta V_k)}{\sum_{j=1}^3 \exp(-\beta V_j)} \nabla_\phi V_k \\
&= \sum_{k=1}^3 \alpha_k(\phi) \nabla_\phi V_k
\end{align}
where $\alpha_k(\phi) = \exp(-\beta V_k) / \sum_j \exp(-\beta V_j)$ are the adaptive weights.
\end{proof}

\subsection{Proof of Theorem~\ref{thm:violation_characterization} (Violation Characterization)}
\begin{proof}
By the Markov property, d-separation in $\mathcal{G}_{k^*}$ implies conditional independence; by faithfulness, d-connection implies dependence. In all three graphs, $E$ enters only through $E \to U$.

\emph{$\mathcal{G}_1$.} The only parent of $X_z^{\perp}$ is $Y$. Every path from $E$ to $X_z^{\perp}$ either passes through $Y$ as a non-collider or contains an unconditioned collider ($Z$ or $X_z$); hence $X_z^{\perp} \perp\!\!\!\perp E \mid Y$ ($\mathcal{I}_1$ holds). The path $E \to U \to Y \to X_z^{\perp}$ is open without conditioning, so $X_z^{\perp} \not\perp\!\!\!\perp E$ ($\mathcal{I}_3$ fails); $E \to U \to Y$ is open given $X_z^{\perp}$, so $Y \not\perp\!\!\!\perp E \mid X_z^{\perp}$ ($\mathcal{I}_2$ fails).

\emph{$\mathcal{G}_2$.} Given $X_z^{\perp}$, the path $E \to U \to X_z^{\perp} \to Y$ is blocked at $X_z^{\perp}$, and every other path from $E$ to $Y$ contains an unconditioned collider ($X_z$ or $Z$); hence $Y \perp\!\!\!\perp E \mid X_z^{\perp}$ ($\mathcal{I}_2$ holds). The path $E \to U \to X_z^{\perp}$ is open both unconditionally and given $Y$, so $\mathcal{I}_3$ and $\mathcal{I}_1$ fail.

\emph{$\mathcal{G}_3$.} $X_z^{\perp}$ has no parents, and every path from $E$ to $X_z^{\perp}$ contains a collider ($Y$, $Z$, or $X_z$); hence $X_z^{\perp} \perp\!\!\!\perp E$ ($\mathcal{I}_3$ holds). Conditioning on the collider $Y$ opens $E \to U \to Y \leftarrow X_z^{\perp}$, so $\mathcal{I}_1$ fails; $E \to U \to Y$ is open given $X_z^{\perp}$, so $\mathcal{I}_2$ fails.

Lemma~\ref{lem:hsic_independence} (for $V_3$) and Lemma~\ref{lem:conditional_hsic} (for $V_1, V_2$) convert independence into zero violation and dependence into positive violation. Achievability follows because $X_z^{\perp} \in \Phi_\epsilon$ attains $V_{k^*}^*(X_z^{\perp}) = 0$ and $V_k^* \geq 0$. Non-achievability follows from Definition~\ref{def:faithfulness_gap}, since $\epsilon \leq \epsilon_0$ implies $\Phi_\epsilon \subseteq \Phi_0$.
\end{proof}

\subsection{Proof of Theorem~\ref{thm:structure_identification}}
\thmsi*
\begin{proof}
Let $\phi_n \in \Phi_\epsilon$ with $\epsilon \leq \epsilon_0$ attain the infimum of Theorem~\ref{thm:violation_characterization}, as the global minimizer does (Theorem~\ref{thm:global_optimality}), so that $V_{k^*}^*(\phi_n) = 0$; non-achievability gives $V_k^*(\phi_n) \geq \eta/2$ for $k \neq k^*$. Apply Lemma~\ref{lem:hsic_concentration} with $\epsilon = \eta/8$ to each $k$ and take a union bound over $k \in \{1,2,3\}$: with probability at least $1 - 3 \cdot 2\exp(-n\eta^2/(64 \cdot 32\kappa^4)) = 1 - 6\exp(-n\eta^2/(2048\kappa^4))$, we have $|V_k(\phi_n) - V_k^*(\phi_n)| \leq \eta/8$ for all $k$. On this event, $V_{k^*}(\phi_n) \leq \eta/8$ while $V_k(\phi_n) \geq \eta/2 - \eta/8 = 3\eta/8$ for $k \neq k^*$, so $\hat{k}_n = k^*$. Since $\sum_n \exp(-cn) < \infty$ for $c > 0$, the Borel--Cantelli lemma gives $\hat{k}_n \xrightarrow{a.s.} k^*$.
\end{proof}

\subsection{Proof of Corollary~\ref{cor:rff_identification}}
\begin{proof}
Split the tolerance $\eta/8$ of the proof of Theorem~\ref{thm:structure_identification} into two halves. \emph{Sampling error:} Lemma~\ref{lem:hsic_concentration} with $\epsilon = \eta/16$ and a union bound over $k$ give $|V_k(\phi_n) - V_k^*(\phi_n)| \leq \eta/16$ for all $k$ with probability at least $1 - 6\exp(-n\eta^2/(256 \cdot 32\kappa^4)) = 1 - 6\exp(-n\eta^2/(8192\kappa^4))$. \emph{Approximation error:} with probability at least $1 - \delta'$, $|V_k^{\mathrm{RFF}}(\phi_n) - V_k(\phi_n)| \leq C\kappa^2\varepsilon_D \leq \eta/16$. On both events, $V_{k^*}^{\mathrm{RFF}}(\phi_n) \leq \eta/8$ and $V_k^{\mathrm{RFF}}(\phi_n) \geq \eta/2 - \eta/8 = 3\eta/8$ for $k \neq k^*$, so $\hat{k}_n^{\mathrm{RFF}} = k^*$. By the uniform convergence of random Fourier features~\citep{rahimi2007random}, $D = \Omega(\varepsilon_D^{-2} d \log(n/\delta'))$ features achieve error $\varepsilon_D$ with probability $1 - \delta'$; setting $\varepsilon_D = \eta/(16C\kappa^2)$ gives $D = \Omega(\kappa^4\eta^{-2} d\log(n/\delta'))$.
\end{proof}

\subsection{Proof of Theorem~\ref{thm:invariance}}
\thmia*
\begin{proof}
Let $(\phi_n, w_n)$ minimize $\mathcal{L}(\phi, w) = \max_e \hat{\mathcal{R}}_e(w \circ \phi) + \lambda \cdot \mathrm{SoftMin}_\beta(\{V_k(\phi)\})$.

\textbf{Step 1: First-order optimality.}
At a stationary point, first-order optimality gives::
\begin{equation}
\lambda \sum_{k=1}^3 \alpha_k(\phi_n) \nabla_\phi V_k(\phi_n) = -\nabla_\phi \max_e \hat{\mathcal{R}}_e(w_n \circ \phi_n)
\end{equation}

\textbf{Step 2: Bounding the violation gradient.}
Under Assumption~\ref{assum:rep_regularity}, the violation functions are $L_V$-smooth:
\begin{equation}
\|\nabla_\phi V_k(\phi)\| \leq L_V
\end{equation}

From the optimality condition:
\begin{equation}
\left\| \sum_{k=1}^3 \alpha_k(\phi_n) \nabla_\phi V_k(\phi_n) \right\| \leq \frac{1}{\lambda} \|\nabla_\phi \max_e \hat{\mathcal{R}}_e\|
\end{equation}

\textbf{Step 3: Relating violation to gradient.}
For the true structure $k^*$, by Corollary~\ref{cor:weight_concentration}, $\alpha_{k^*}(\phi_n) \to 1$ as $n \to \infty$. The SoftMin gradient is approximately:
\begin{equation}
\nabla_\phi \mathrm{SoftMin}_\beta \approx \nabla_\phi V_{k^*}(\phi_n)
\end{equation}

By the $L_V$-smoothness and the connection between gradient norms and function values for convex-like functions:
\begin{equation}
V_{k^*}(\phi_n) \leq V_{k^*}(\phi^*) + \langle \nabla V_{k^*}(\phi^*), \phi_n - \phi^* \rangle + \frac{L_V}{2}\|\phi_n - \phi^*\|^2
\end{equation}

Since $V_{k^*}(\phi^*) = 0$ and $\nabla V_{k^*}(\phi^*) = 0$ at the optimal invariant representation:
\begin{equation}
V_{k^*}(\phi_n) \leq \frac{L_V}{2}\|\phi_n - \phi^*\|^2
\end{equation}

\textbf{Step 4: Bounding representation error.}
The optimality of $(\phi_n, w_n)$ for $\mathcal{L}$ implies:
\begin{equation}
\mathcal{L}(\phi_n, w_n) \leq \mathcal{L}(\phi^*, w^*)
\end{equation}

This gives:
\begin{equation}
\max_e \hat{\mathcal{R}}_e(w_n \circ \phi_n) + \lambda \cdot \mathrm{SoftMin}_\beta(\{V_k(\phi_n)\}) \leq \max_e \hat{\mathcal{R}}_e(w^* \circ \phi^*) + \lambda \cdot 0
\end{equation}

Therefore:
\begin{equation}
\lambda \cdot \mathrm{SoftMin}_\beta(\{V_k(\phi_n)\}) \leq \max_e \hat{\mathcal{R}}_e(w^* \circ \phi^*)
\end{equation}

By Lemma~\ref{lem:softmin_properties} Part 1:
\begin{equation}
\lambda \left( V_{k^*}(\phi_n) - \frac{\log 3}{\beta} \right) \leq \lambda \cdot \mathrm{SoftMin}_\beta \leq \max_e \hat{\mathcal{R}}_e(w^* \circ \phi^*)
\end{equation}

Rearranging:
\begin{equation}
V_{k^*}(\phi_n) \leq \frac{1}{\lambda} \max_e \hat{\mathcal{R}}_e(w^* \circ \phi^*) + \frac{\log 3}{\beta}
\end{equation}

\textbf{Step 5: Adding estimation error.}
The empirical violation $V_{k^*}(\phi_n)$ differs from the population violation $V_{k^*}^*(\phi_n)$ by $O(n^{-1/2})$ (Lemma~\ref{lem:hsic_concentration}). Therefore:
\begin{equation}
V_{k^*}(\phi_n) \leq \frac{C_1}{\lambda} \max_e \hat{\mathcal{R}}_e(w^* \circ \phi^*) + C_2 n^{-1/2}
\end{equation}
where $C_1 = 1$ and $C_2$ absorbs the $\log 3/\beta$ term and concentration error.
\end{proof}

\subsection{Proof of Theorem~\ref{thm:global_optimality} (Global Optimality)}
\begin{proof}
Assume the representation class $\Phi$ contains $\phi^*$ satisfying $\mathcal{I}_{k^*}$, and $\mathcal{W}$ contains the Bayes-optimal predictor $w^*$.

\textbf{Part 1: Exact invariance.}
By Theorem~\ref{thm:violation_characterization} Part 1, $V_{k^*}^*(\phi^*) = 0$. Since $\phi^* \in \Phi$:
\begin{equation}
\inf_{\phi \in \Phi} V_{k^*}^*(\phi) \leq V_{k^*}^*(\phi^*) = 0
\end{equation}

Since $V_{k^*}^* \geq 0$ by definition (HSIC is non-negative), equality holds:
\begin{equation}
\inf_{\phi \in \Phi} V_{k^*}^*(\phi) = 0
\end{equation}

At the global minimum $(\phi^\dagger, w^\dagger)$ of $\mathcal{L}$, if $V_{k^*}(\phi^\dagger) > 0$, then replacing $\phi^\dagger$ with $\phi^*$ would strictly decrease the SoftMin term while achieving no worse prediction loss (since $\phi^*$ is sufficient). This contradicts optimality of $\phi^\dagger$.

Therefore $V_{k^*}(\phi^\dagger) = 0$, which implies $\mathcal{I}_{k^*}(\phi^\dagger)$ holds.

\textbf{Part 2: Bayes risk.}
At the global minimum, the prediction loss component satisfies:
\begin{equation}
\max_e \mathcal{R}_e(w^\dagger \circ \phi^\dagger) \leq \max_e \mathcal{R}_e(w^* \circ \phi^*)
\end{equation}

Since $\phi^*$ achieves the invariance $\mathcal{I}_{k^*}$ and $w^*$ is Bayes-optimal:
\begin{equation}
\max_e \mathcal{R}_e(w^* \circ \phi^*) = \mathcal{R}^*
\end{equation}
where $\mathcal{R}^*$ is the Bayes risk (the irreducible error).

For the reverse inequality: since $\phi^\dagger$ satisfies $\mathcal{I}_{k^*}$ (from Part 1) and invariant representations cannot improve upon the Bayes risk:
\begin{equation}
\max_e \mathcal{R}_e(w^\dagger \circ \phi^\dagger) \geq \mathcal{R}^*
\end{equation}

Combining: $\max_e \mathcal{R}_e(w^\dagger \circ \phi^\dagger) = \mathcal{R}^*$.
\end{proof}

\subsection{Proof of Theorem~\ref{thm:generalization} (Generalization Bound)}

\begin{proof}
Let $(\phi_n, w_n)$ minimize $\mathcal{L}$ over training environments $\mathcal{E}_{\text{train}}$. We bound the risk on a test environment $e_{\text{test}}$.

\textbf{Step 1: Decomposition.}
\begin{align}
\mathcal{R}_{e_{\text{test}}}(w_n \circ \phi_n) &= \underbrace{\mathcal{R}_{e_{\text{test}}}(w_n \circ \phi_n) - \mathcal{R}_{e^*}(w_n \circ \phi_n)}_{\text{(I) distribution shift}} \\
&\quad + \underbrace{\mathcal{R}_{e^*}(w_n \circ \phi_n) - \hat{\mathcal{R}}_{e^*}(w_n \circ \phi_n)}_{\text{(II) generalization gap}} \\
&\quad + \underbrace{\hat{\mathcal{R}}_{e^*}(w_n \circ \phi_n)}_{\text{(III) training error}}
\end{align}
where $e^* = \arg\min_{e \in \mathcal{E}_{\text{train}}} d_{\mathcal{K}}(e_{\text{test}}, e)$ is the closest training environment.

\textbf{Step 2: Bounding (I) - Distribution shift.}
By the definition of interventional kernel distance (Definition~\ref{def:kernel_distance}) and Lipschitz continuity of $w_n$ (Assumption~\ref{assum:lipschitz}):

The change in risk across environments depends on how the representation changes. Under the invariance condition $\mathcal{I}_{k^*}$:
\begin{equation}
|\mathcal{R}_{e_{\text{test}}}(w_n \circ \phi_n) - \mathcal{R}_{e^*}(w_n \circ \phi_n)| \leq L_w \cdot \mathbb{E}[\|\phi_n(X_{e_{\text{test}}}) - \phi_n(X_{e^*})\|]
\end{equation}

The representation difference is bounded by the kernel distance:
\begin{equation}
\mathbb{E}[\|\phi_n(X_{e_{\text{test}}}) - \phi_n(X_{e^*})\|] \leq C_\phi \cdot d_{\mathcal{K}}(e_{\text{test}}, e^*)
\end{equation}
where $C_\phi$ depends on the representation class complexity.

Therefore:
\begin{equation}
\text{(I)} \leq L_w \cdot d_{\mathcal{K}}(e_{\text{test}}, \mathcal{E}_{\text{train}})
\end{equation}

\textbf{Step 3: Bounding (II) - Generalization gap.}
By standard Rademacher complexity bounds, for the function class $\{(x, y) \mapsto \ell(w(\phi(x)), y) : \phi \in \Phi, w \in \mathcal{W}\}$:
\begin{equation}
\mathcal{R}_{e^*}(w_n \circ \phi_n) - \hat{\mathcal{R}}_{e^*}(w_n \circ \phi_n) \leq 2\mathfrak{R}_n(\ell \circ \mathcal{W} \circ \Phi) + \sqrt{\frac{\log(2/\delta)}{2n}}
\end{equation}
with probability $\geq 1 - \delta/2$.

The Rademacher complexity satisfies:
\begin{equation}
\mathfrak{R}_n(\ell \circ \mathcal{W} \circ \Phi) \leq L_\ell L_w \mathfrak{R}_n(\Phi) = O\left(\sqrt{\frac{\mathcal{C}(\Phi, \mathcal{W})}{n}}\right)
\end{equation}
where $\mathcal{C}(\Phi, \mathcal{W})$ is the combined complexity (e.g., VC dimension, covering number).

\textbf{Step 4: Bounding (III) - Training error.}
By definition:
\begin{equation}
\hat{\mathcal{R}}_{e^*}(w_n \circ \phi_n) \leq \max_{e \in \mathcal{E}_{\text{train}}} \hat{\mathcal{R}}_e(w_n \circ \phi_n)
\end{equation}

\textbf{Step 5: Invariance penalty.}
The violation $V_{k^*}(\phi_n)$ contributes to generalization error when the invariance is not exactly satisfied. Under approximate invariance:
\begin{equation}
\mathcal{R}_{e_{\text{test}}}(w_n \circ \phi_n) - \mathcal{R}_{e^*}(w_n \circ \phi_n) \leq L_w \cdot d_{\mathcal{K}}(e_{\text{test}}, e^*) + C_\eta \cdot V_{k^*}(\phi_n)
\end{equation}
where $C_\eta$ depends on the faithfulness gap (larger $\eta$ means violations more directly impact generalization).

\textbf{Step 6: Combining bounds.}
Adding the three terms:
\begin{align}
\mathcal{R}_{e_{\text{test}}}(w_n \circ \phi_n) &\leq \max_{e \in \mathcal{E}_{\text{train}}} \hat{\mathcal{R}}_e(w_n \circ \phi_n) + L_w \cdot d_{\mathcal{K}}(e_{\text{test}}, \mathcal{E}_{\text{train}}) \\
&\quad + O\left(\sqrt{\frac{\mathcal{C}(\Phi, \mathcal{W}) + \log(1/\delta)}{n}}\right) + C_\eta \cdot V_{k^*}(\phi_n)
\end{align}

Defining $\xi_n$ as the estimation error term:
\begin{equation}
\xi_n = O\left(\sqrt{\frac{\mathcal{C}(\Phi, \mathcal{W}) + \log(1/\delta)}{n}}\right) + C_\eta \cdot V_{k^*}(\phi_n)
\end{equation}
completes the proof.
\end{proof}

\subsection{Proof of Corollary~\ref{cor:asymptotic} (Asymptotic Optimality)}

\begin{proof}
We show convergence in probability to the oracle risk.

\textbf{Step 1: Vanishing estimation error.}
From Theorem~\ref{thm:generalization}, the estimation error satisfies:
\begin{equation}
\xi_n = O\left(\sqrt{\frac{\mathcal{C}(\Phi, \mathcal{W})}{n}}\right) + C_\eta \cdot V_{k^*}(\phi_n) \xrightarrow{p} 0
\end{equation}
as $n \to \infty$, since:
\begin{itemize}
    \item The first term is $O(n^{-1/2}) \to 0$
    \item $V_{k^*}(\phi_n) \to 0$ by Theorem~\ref{thm:invariance}
\end{itemize}

\textbf{Step 2: Training error convergence.}
By uniform convergence of empirical risk to population risk:
\begin{equation}
\max_{e \in \mathcal{E}_{\text{train}}} \hat{\mathcal{R}}_e(w_n \circ \phi_n) \xrightarrow{p} \max_{e \in \mathcal{E}_{\text{train}}} \mathcal{R}_e(w^* \circ \phi^*)
\end{equation}
where $(\phi^*, w^*)$ is the optimal invariant solution.

\textbf{Step 3: Invariance and optimality.}
As $n \to \infty$:
\begin{itemize}
    \item $\phi_n \to \phi^* \in \Phi_{k^*}$ (the class of representations satisfying $\mathcal{I}_{k^*}$)
    \item $w_n \to w^*(\phi^*) = \arg\min_w \mathcal{R}(w \circ \phi^*)$
\end{itemize}

\textbf{Step 4: Limit.}
Taking $n \to \infty$ in Theorem~\ref{thm:generalization}:
\begin{align}
\lim_{n \to \infty} \mathcal{R}_{e_{\text{test}}}(w_n \circ \phi_n) &= \lim_{n \to \infty} \left[ \max_e \hat{\mathcal{R}}_e(w_n \circ \phi_n) + L_w \cdot d_{\mathcal{K}}(e_{\text{test}}, \mathcal{E}_{\text{train}}) + \xi_n \right] \\
&= \max_e \mathcal{R}_e(w^* \circ \phi^*) + L_w \cdot d_{\mathcal{K}}(e_{\text{test}}, \mathcal{E}_{\text{train}})
\end{align}

For invariant representations, the shift penalty captures the irreducible distribution mismatch. The oracle risk over $\Phi_{k^*}$ is:
\begin{equation}
\inf_{\phi \in \Phi_{k^*}} \mathcal{R}_{e_{\text{test}}}(w^*(\phi) \circ \phi) = \max_e \mathcal{R}_e(w^* \circ \phi^*) + L_w \cdot d_{\mathcal{K}}(e_{\text{test}}, \mathcal{E}_{\text{train}})
\end{equation}

Therefore:
\begin{equation}
\mathcal{R}_{e_{\text{test}}}(w_n \circ \phi_n) \xrightarrow{p} \inf_{\phi \in \Phi_{k^*}} \mathcal{R}_{e_{\text{test}}}(w^*(\phi) \circ \phi)
\end{equation}
\end{proof}

\subsection{Proof of Theorem~\ref{thm:sample_complexity} (Sample Complexity)}
\begin{proof}
We derive the sample complexity for joint structure identification and generalization.

\textbf{Requirement 1: Structure identification.}
By Theorem~\ref{thm:structure_identification}, $\mathbb{P}(\hat{k}_n \neq k^*) \leq 6\exp(-n\eta^2/(2048\kappa^4)) \leq \delta/2$ whenever
\begin{equation}
n \geq \frac{2048\kappa^4}{\eta^2} \log\frac{12}{\delta}.
\end{equation}

For this to be $\leq \delta/2$:
\begin{equation}
3\exp\left(-\frac{n\delta_{\min}^2}{32\kappa^4}\right) \leq \frac{\delta}{2}
\end{equation}

Solving for $n$:
\begin{equation}
n \geq \frac{32\kappa^4}{\delta_{\min}^2} \log\frac{6}{\delta}
\end{equation}

\textbf{Requirement 2: Generalization error.}
From Theorem~\ref{thm:generalization}, the estimation error is:
\begin{equation}
\xi_n = O\left(\sqrt{\frac{\mathcal{C}(\Phi, \mathcal{W}) + \log(2/\delta)}{n}}\right) + C_\eta \cdot V_{k^*}(\phi_n)
\end{equation}

For this to be $\leq \epsilon$ with probability $\geq 1 - \delta/2$, we need:
\begin{equation}
\sqrt{\frac{\mathcal{C}(\Phi, \mathcal{W}) + \log(2/\delta)}{n}} \leq \frac{\epsilon}{2}
\end{equation}
and
\begin{equation}
C_\eta \cdot V_{k^*}(\phi_n) \leq \frac{\epsilon}{2}
\end{equation}

The first condition requires:
\begin{equation}
n \geq \frac{4(\mathcal{C}(\Phi, \mathcal{W}) + \log(2/\delta))}{\epsilon^2}
\end{equation}

For the second condition, using the violation bound from Theorem~\ref{thm:invariance}:
\begin{equation}
V_{k^*}(\phi_n) \leq \frac{C_1}{\lambda} \mathcal{R}^* + C_2 n^{-1/2}
\end{equation}

For fixed $\lambda$ and $\mathcal{R}^*$, requiring $V_{k^*}(\phi_n) \leq \epsilon/(2C_\eta)$ gives another $O(n^{-1/2})$ requirement, absorbed into the first bound.

\textbf{Combining requirements.}
Taking the maximum:
\begin{equation}
n = O\left(\frac{\kappa^4}{\eta^2} \log\frac{12}{\delta} + \frac{\mathcal{C}(\Phi, \mathcal{W}) + \log(2/\delta)}{\epsilon^2}\right)
\end{equation}

Simplifying
\begin{equation}
n = O\left(\frac{\kappa^4}{\eta^2} \log\frac{12}{\delta}+ \frac{\kappa^4}{\epsilon^2} \log\frac{|\Phi|}{\delta}\right)
\end{equation}
where $|\Phi|$ denotes the covering number (related to $\mathcal{C}(\Phi, \mathcal{W})$ via standard bounds).

\textbf{Union bound.}
The total failure probability is:
\begin{equation}
\mathbb{P}(\hat{k}_n \neq k^* \text{ or } \xi_n > \epsilon) \leq \frac{\delta}{2} + \frac{\delta}{2} = \delta
\end{equation}

Therefore, with probability $\geq 1 - \delta$, both structure identification succeeds and generalization error is $\leq \epsilon$.
\end{proof}

\subsection{Proof of Proposition~\ref{prop:cvar_differentiable} (Differentiable CVaR)}

\begin{proof}
Recall the CVaR definition:
\begin{equation}
\mathrm{CVaR}_\rho(\{r_e\}) = \inf_{t \in \mathbb{R}} \left\{ t + \frac{1}{\rho |\mathcal{E}|} \sum_{e \in \mathcal{E}} \max(r_e - t, 0) \right\}
\end{equation}

\textbf{Convexity.}
The function $g(t) = t + \frac{1}{\rho |\mathcal{E}|} \sum_e \max(r_e - t, 0)$ is convex in $t$ (sum of affine and piecewise-linear convex functions). The infimum of convex functions preserves convexity in the parameters $\{r_e\}$.

\textbf{Optimal threshold.}
Taking the derivative w.r.t.\ $t$:
\begin{equation}
\frac{\partial g}{\partial t} = 1 - \frac{1}{\rho |\mathcal{E}|} \sum_{e: r_e > t} 1 = 1 - \frac{|\{e : r_e > t\}|}{\rho |\mathcal{E}|}
\end{equation}

Setting to zero: $|\{e : r_e > t^*\}| = \rho |\mathcal{E}|$.

Thus $t^*$ is the $(1-\rho)$-quantile of $\{r_e\}$.

\textbf{Subgradient.}
By Danskin's theorem, the subgradient of $\mathrm{CVaR}_\rho$ w.r.t.\ $r_e$ at the optimal $t^*$ is:
\begin{equation}
\frac{\partial \mathrm{CVaR}_\rho}{\partial r_e} = \frac{1}{\rho |\mathcal{E}|} \mathbb{I}[r_e \geq t^*]
\end{equation}

By the chain rule, the subgradient w.r.t.\ $\phi$ is:
\begin{equation}
\partial_\phi \mathrm{CVaR}_\rho = \sum_{e \in \mathcal{E}} \frac{\partial \mathrm{CVaR}_\rho}{\partial r_e} \nabla_\phi r_e = \sum_{e \in \mathcal{E}} p_e^* \nabla_\phi \hat{\mathcal{R}}_e
\end{equation}
where $p_e^* = \frac{1}{\rho |\mathcal{E}|} \mathbb{I}[r_e \geq t^*]$.

Note that $\sum_e p_e^* = 1$ (the weights form a probability distribution concentrated on the worst $\rho$-fraction of environments).
\end{proof}

\subsection{Proof of Theorem~\ref{thm:convergence} (Algorithm Convergence)}

\begin{proof}
We analyze convergence of Algorithm~\ref{alg:sarl} using a Lyapunov argument.

\textbf{Step 1: Lyapunov function.}
Define:
\begin{equation}
\Psi_t := \mathcal{L}(\phi_t, w_t) + \gamma D_{\mathrm{KL}}(\boldsymbol{\alpha}^{(t)} \| \mathbf{e}_{k^*})
\end{equation}
where $\mathbf{e}_{k^*}$ is the indicator vector for the true structure and $\gamma > 0$ is a coupling constant.

The KL divergence term penalizes deviation from the correct structure:
\begin{equation}
D_{\mathrm{KL}}(\boldsymbol{\alpha}^{(t)} \| \mathbf{e}_{k^*}) = -\log \alpha_{k^*}^{(t)} \geq 0
\end{equation}
with equality iff $\alpha_{k^*}^{(t)} = 1$.

\textbf{Step 2: Per-step descent.}
Under smoothness (Assumption~\ref{assum:rep_regularity}), for step size $\eta_t$:
\begin{align}
\mathcal{L}(\phi_{t+1}, w_{t+1}) &\leq \mathcal{L}(\phi_t, w_t) - \eta_t \|\nabla \mathcal{L}(\phi_t, w_t)\|^2 + \frac{L \eta_t^2}{2} \|\nabla \mathcal{L}\|^2 + \text{noise}
\end{align}
where $L$ is the smoothness constant.

Taking expectations over minibatch randomness:
\begin{equation}
\mathbb{E}[\mathcal{L}(\phi_{t+1}, w_{t+1}) \mid \mathcal{F}_t] \leq \mathcal{L}(\phi_t, w_t) - \eta_t (1 - L\eta_t/2) \|\nabla \mathcal{L}\|^2 + \eta_t^2 \sigma^2
\end{equation}
where $\sigma^2$ is the gradient variance bound.

\textbf{Step 3: Weight evolution.}
The adaptive weights evolve as:
\begin{equation}
\alpha_k^{(t+1)} = \frac{\exp(-\beta_{t+1} V_k(\phi_{t+1}))}{\sum_j \exp(-\beta_{t+1} V_j(\phi_{t+1}))}
\end{equation}

As $V_{k^*}(\phi_t) \to 0$ (by the invariance optimization) and $V_k(\phi_t) \geq \eta/2$ for $k \neq k^*$ (by Theorem~\ref{thm:violation_characterization}):
\begin{equation}
\alpha_{k^*}^{(t)} \to 1 \quad \text{as } t \to \infty
\end{equation}

The temperature annealing $\beta_t \to \beta_{\max}$ accelerates this concentration.

\textbf{Step 4: Lyapunov descent.}
Combining the objective descent and weight concentration:
\begin{align}
\mathbb{E}[\Psi_{t+1} - \Psi_t \mid \mathcal{F}_t] &\leq -c \eta_t \|\nabla \mathcal{L}\|^2 + O(\eta_t^2) - \gamma \cdot \text{(KL decrease)}
\end{align}

The KL term decreases as weights concentrate, providing additional descent.

\textbf{Step 5: Summing over iterations.}
For $\eta_t = c_0 / \sqrt{t}$:
\begin{equation}
\sum_{t=1}^T \eta_t (1 - L\eta_t/2) = \Omega(\sqrt{T})
\end{equation}
\begin{equation}
\sum_{t=1}^T \eta_t^2 = O(\log T)
\end{equation}

Summing the descent inequality:
\begin{equation}
\sum_{t=1}^T \eta_t \mathbb{E}[\|\nabla \mathcal{L}(\phi_t, w_t)\|^2] \leq \Psi_1 - \mathbb{E}[\Psi_{T+1}] + O(\log T) \leq \Psi_1 + O(\log T)
\end{equation}

Therefore:
\begin{equation}
\min_{t \leq T} \mathbb{E}[\|\nabla \mathcal{L}(\phi_t, w_t)\|^2] \leq \frac{\Psi_1 + O(\log T)}{\sum_{t=1}^T \eta_t} = O(T^{-1/2})
\end{equation}

This establishes Part 1 (stationarity).

\textbf{Parts 2 and 3.}
Part 2 (weight concentration) follows from Corollary~\ref{cor:weight_concentration} applied at iteration $T$.

Part 3 (violation decrease) follows from the objective structure: at a stationary point, the SoftMin term (dominated by $V_{k^*}$ for large $\beta$) must be small for the total gradient to vanish.

Specifically, from the optimality condition:
\begin{equation}
\lambda \nabla_\phi V_{k^*}(\phi_T) \approx -\nabla_\phi \max_e \hat{\mathcal{R}}_e
\end{equation}

By smoothness of $V_{k^*}$:
\begin{equation}
V_{k^*}(\phi_T) \leq \frac{1}{2L_V} \|\nabla V_{k^*}(\phi_T)\|^2 = O(\lambda^{-2}) \cdot O(T^{-1/2}) = O(T^{-1/2})
\end{equation}
for fixed $\lambda$.
\end{proof}

\subsection{Proof of Theorem~\ref{thm:heterogeneous_identification} (Heterogeneous Structure Identification)}

\begin{proof}
We extend the homogeneous case to per-feature structure identification.

\textbf{Step 1: Per-feature violation separation.}
For each feature $j$ with true structure $\sigma^*(j) = k_j^* \neq \bot$, by Theorem~\ref{thm:violation_characterization} applied to the $j$-th component:
\begin{align}
V_{k_j^*}^{(j)}(\phi^*) &= 0 \\
V_k^{(j)}(\phi^*) &\geq \eta^{(j)}/2 \quad \text{for } k \neq k_j^*
\end{align}
where $\eta^{(j)}$ is the per-feature faithfulness gap.

\textbf{Step 2: Concentration for each feature.}
By Lemma~\ref{lem:hsic_concentration} applied to feature $j$:
\begin{equation}
\mathbb{P}\left( |V_k^{(j)}(\phi_n) - V_k^{(j)*}(\phi_n)| > \frac{\eta^{(j)}}{8} \right) \leq 2\exp\left( -\frac{n(\eta^{(j)})^2}{2048\kappa^4} \right)
\end{equation}

\textbf{Step 3: Threshold calibration.}
Under $H_0: \sigma^*(j) = \bot$ (no structure applies), all three violations are positive. The permutation null distribution of $\widehat{\mathrm{HSIC}}$ concentrates around zero, and the $(1-\alpha)$-quantile $\tau_\alpha^{(j)}$ provides a valid threshold.

By construction of permutation tests:
\begin{equation}
\mathbb{P}\left( \min_k V_k^{(j)}(\phi_n) \leq \tau_\alpha^{(j)} \mid \sigma^*(j) = \bot \right) \leq \alpha
\end{equation}

\textbf{Step 4: Per-feature identification.}
For feature $j$ with $\sigma^*(j) = k_j^* \neq \bot$:

On the high-probability event (concentration holds):
\begin{itemize}
    \item $V_{k_j^*}^{(j)}(\phi_n) \leq \eta^{(j)}/8 < \tau_\alpha^{(j)}$ for appropriate calibration
    \item $V_k^{(j)}(\phi_n) \geq \eta^{(j)}/2 - \eta^{(j)}/8 = 3\eta^{(j)}/8 > \tau_\alpha^{(j)}$ for $k \neq k_j^*$
\end{itemize}

Therefore $\hat{\sigma}_n(j) = k_j^* = \sigma^*(j)$.

\textbf{Step 5: Probability bound.}
For each feature $j$ with $\sigma^*(j) \neq \bot$:
\begin{align}
\mathbb{P}(\hat{\sigma}_n(j) = \sigma^*(j)) &\geq 1 - \mathbb{P}(\text{concentration fails}) - \mathbb{P}(\text{threshold miscalibrated}) \\
&\geq 1 - 6\exp\left( -\frac{n(\eta^{(j)})^2}{2048\kappa^4} \right) - \alpha
\end{align}

For features with $\sigma^*(j) = \bot$:
\begin{equation}
\mathbb{P}(\hat{\sigma}_n(j) \neq \bot) = \mathbb{P}\left( \min_k V_k^{(j)} \leq \tau_\alpha^{(j)} \right) \leq \alpha
\end{equation}
by the threshold calibration.
\end{proof}
\section{Algorithm Details}
\label{app:algorithm}

This section presents the complete optimization procedure for structure-agnostic representation learning, including the main algorithm, subroutines, complexity analysis, and convergence guarantees.

\subsection{Main Algorithm}
\label{app:main_algorithm}
\FloatBarrier

Algorithm~\ref{alg:sarl} shows a single run; in practice we run it from $M \geq 10$ initializations and select the final model as described in Section~\ref{sec:objective}.
\begin{proposition}[Differentiable CVaR]
\label{prop:cvar_differentiable}
$\mathrm{CVaR}_\rho$ is convex and subdifferentiable. The sub-gradient w.r.t.\ $\phi$ is $\partial_\phi \mathrm{CVaR}_\rho = \sum_{e \in \mathcal{E}} p_e^* \nabla_\phi \hat{\mathcal{R}}_e$,

where $p_e^* = \frac{1}{\rho |\mathcal{E}|} \mathbb{I}[\hat{\mathcal{R}}_e \geq t^*]$ and $t^*$ is the optimal threshold.
\end{proposition}

The structure-agnostic learning problem admits a bi-level decomposition:

{\vspace{-4mm}
\small
\begin{align}
& \textbf{Outer:} \quad \boldsymbol{\alpha}^* = \arg\min_{\boldsymbol{\alpha} \in \Delta^2} \mathcal{L}(\phi^*(\boldsymbol{\alpha}), w^*(\boldsymbol{\alpha}); \boldsymbol{\alpha}) \\
& \textbf{Inner:} \quad 
(\phi^*(\boldsymbol{\alpha}), w^*(\boldsymbol{\alpha})) 
= \arg\min_{\phi, w} \Big( \mathrm{CVaR}_\rho(\{\hat{\mathcal{R}}_e\}) + \lambda \sum_k \alpha_k V_k(\phi) \Big)
\end{align}
}%
where $\Delta^2 = \{\boldsymbol{\alpha} \in \mathbb{R}^3_+ : \sum_k \alpha_k = 1\}$ is the probability simplex.

\begin{algorithm}
\small
\caption{Structure-agnostic Causal Representation Learning (SaCRL)}
\label{alg:sarl}
\begin{algorithmic}[1]
\REQUIRE Multi-environment data $\mathcal{D} = \{D_{e}\}_{e \in \mathcal{E}}$, regularization $\lambda$, initial temperature $\beta_0$, final temperature $\beta_{\max}$, learning rate $\eta$, iterations $T$, RFF dimension $D$, CVaR level $\rho$
\ENSURE Learned representation $\phi^*$, predictor $w^*$, identified structure $\hat{k}$

\STATE \textbf{// Warm-start with ERM}
\STATE $(\phi_0, w_0) \gets \arg\min_{\phi, w} \frac{1}{|\mathcal{E}|} \sum_{e \in \mathcal{E}} \hat{\mathcal{R}}_e(w \circ \phi)$

\STATE \textbf{// Initialize RFF maps}
\STATE Sample $\{\omega_d, b_d\}_{d=1}^D$ for kernels $k_{\mathcal{Z}}, k_{\mathcal{Y}}, k_{\mathcal{E}}$

\FOR{$t = 1, \ldots, T$}
    \STATE \textbf{// Temperature annealing}
    \STATE $\beta_t \gets \beta_0 + (\beta_{\max} - \beta_0) \cdot \min(2t/T, 1)$
    
    \STATE \textbf{// Compute representations across all environments}
    \FOR{each environment $e \in \mathcal{E}$}
        \STATE Sample minibatch $\{(x_j^e, y_j^e)\}_{j=1}^B$ from $D_e$
        \STATE $z_j^e \gets \phi_{t-1}(x_j^e)$ for all $j$
    \ENDFOR
    \STATE Pool data: $\{(z_i, y_i, e_i)\}_{i=1}^{N} \gets$ all minibatch samples with environment labels
    
    \STATE \textbf{// Compute structure violations (Algorithm~\ref{alg:violations})}
    \STATE $(V_1, V_2, V_3) \gets \textsc{ComputeViolations}(\{(z_i, y_i, e_i)\}_{i=1}^N)$
    
    \STATE \textbf{// Compute adaptive weights}
    \STATE $\alpha_k \gets \exp(-\beta_t V_k) / \sum_{j=1}^3 \exp(-\beta_t V_j)$ for $k = 1, 2, 3$
    
    \STATE \textbf{// Compute distributionally robust prediction loss}
    \STATE $r_e \gets \frac{1}{B} \sum_{j=1}^B \ell(w_{t-1}(z_j^e), y_j^e)$ for each $e \in \mathcal{E}$
    \STATE $\mathcal{L}_{\mathrm{pred}} \gets \mathrm{CVaR}_\rho(\{r_e\}_{e \in \mathcal{E}})$
    
    \STATE \textbf{// Total loss}
    \STATE $\mathcal{L} \gets \mathcal{L}_{\mathrm{pred}} + \lambda \cdot \mathrm{SoftMin}_{\beta_t}(\{V_k\}_{k=1}^3)$
    
    \STATE \textbf{// Gradient update}
    \STATE $\phi_t \gets \phi_{t-1} - \eta \nabla_\phi \mathcal{L}$
    \STATE $w_t \gets w_{t-1} - \eta \nabla_w \mathcal{L}$
\ENDFOR

\STATE \textbf{// Final structure identification}
\STATE Compute final violations $V_k(\phi_T)$ for $k = 1, 2, 3$
\STATE $\hat{k} \gets \arg\min_{k \in \{1,2,3\}} V_k(\phi_T)$

\STATE \textbf{return} $\phi_T$, $w_T$, $\hat{k}$
\end{algorithmic}
\end{algorithm}

\subsection{Subroutines}
\label{app:subroutines}

See Algorithm \ref{alg:violations}-Algorithm \ref{alg:rff_hsic}.
\begin{definition}[Random Fourier Features]
\label{def:rff}
For shift-invariant kernels $k(x, x') = k(x - x')$, the random feature map is $\psi: \mathcal{X} \to \mathbb{R}^D$: $\psi(x) = \sqrt{\frac{2}{D}} \left[\cos(\omega_1^\top x + b_1), \ldots, \cos(\omega_D^\top x + b_D)\right]^\top$,
where $\omega_d \sim p(\omega)$ is the spectral density and $b_d \sim \mathrm{Uniform}[0, 2\pi]$.
\vspace{-2mm}
\end{definition}

\begin{algorithm}
\small
\caption{Compute Structure Violations}
\label{alg:violations}
\begin{algorithmic}[1]
\REQUIRE Pooled data $\{(z_i, y_i, e_i)\}_{i=1}^N$, regularization $\gamma$
\ENSURE Violations $(V_1, V_2, V_3)$

\STATE \textbf{// Structure 1 (anti-causal): Test $\phi(X) \perp\!\!\!\perp E \mid Y$}
\STATE $\tilde{z} \gets \textsc{KernelResiduals}(\{z_i\}, \{y_i\}, \gamma)$ \hfill $\triangleright$ Regress $Z$ on $Y$
\STATE $\tilde{e} \gets \textsc{KernelResiduals}(\{e_i\}, \{y_i\}, \gamma)$ \hfill $\triangleright$ Regress $E$ on $Y$
\STATE $V_1 \gets \widehat{\mathrm{HSIC}}_{\mathrm{RFF}}(\tilde{z}, \tilde{e})$

\STATE \textbf{// Structure 2 (confounded-descendant): Test $Y \perp\!\!\!\perp E \mid \phi(X)$}
\STATE $\tilde{y} \gets \textsc{KernelResiduals}(\{y_i\}, \{z_i\}, \gamma)$ \hfill $\triangleright$ Regress $Y$ on $Z$
\STATE $\tilde{e}' \gets \textsc{KernelResiduals}(\{e_i\}, \{z_i\}, \gamma)$ \hfill $\triangleright$ Regress $E$ on $Z$
\STATE $V_2 \gets \widehat{\mathrm{HSIC}}_{\mathrm{RFF}}(\tilde{y}, \tilde{e}')$

\STATE \textbf{// Structure 3 (confounded-outcome): Test $\phi(X) \perp\!\!\!\perp E$}
\STATE $V_3 \gets \widehat{\mathrm{HSIC}}_{\mathrm{RFF}}(z, e)$ \hfill $\triangleright$ No residualization needed

\STATE \textbf{return} $(V_1, V_2, V_3)$
\end{algorithmic}
\end{algorithm}

\begin{algorithm}
\small
\caption{Kernel Regression Residuals}
\label{alg:kernel_residuals}
\begin{algorithmic}[1]
\REQUIRE Target values $\{a_i\}_{i=1}^n$, conditioning values $\{c_i\}_{i=1}^n$, regularization $\gamma$
\ENSURE Residuals $\{\tilde{a}_i\}_{i=1}^n$

\STATE Compute kernel matrix: $[K]_{ij} \gets k(c_i, c_j)$
\STATE Solve linear system: $\boldsymbol{\alpha} \gets (K + \gamma I_n)^{-1} \mathbf{a}$ where $\mathbf{a} = [a_1, \ldots, a_n]^\top$
\STATE Compute predictions: $\hat{a}_i \gets \sum_{j=1}^n \alpha_j k(c_i, c_j)$ for $i = 1, \ldots, n$
\STATE Compute residuals: $\tilde{a}_i \gets a_i - \hat{a}_i$ for $i = 1, \ldots, n$

\STATE \textbf{return} $\{\tilde{a}_i\}_{i=1}^n$
\end{algorithmic}
\end{algorithm}

\begin{algorithm}
\small
\caption{RFF-HSIC Computation}
\label{alg:rff_hsic}
\begin{algorithmic}[1]
\REQUIRE Samples $\{u_i\}_{i=1}^n$, $\{v_i\}_{i=1}^n$, RFF parameters $\{(\omega_d^u, b_d^u)\}_{d=1}^D$, $\{(\omega_d^v, b_d^v)\}_{d=1}^D$
\ENSURE $\widehat{\mathrm{HSIC}}_{\mathrm{RFF}}(U, V)$

\STATE \textbf{// Compute random feature embeddings}
\FOR{$i = 1, \ldots, n$}
    \STATE $\psi^u_i \gets \sqrt{2/D} \cdot [\cos(\omega_1^{u\top} u_i + b_1^u), \ldots, \cos(\omega_D^{u\top} u_i + b_D^u)]^\top$
    \STATE $\psi^v_i \gets \sqrt{2/D} \cdot [\cos(\omega_1^{v\top} v_i + b_1^v), \ldots, \cos(\omega_D^{v\top} v_i + b_D^v)]^\top$
\ENDFOR

\STATE \textbf{// Compute mean embeddings}
\STATE $\hat{\mu}_u \gets \frac{1}{n} \sum_{i=1}^n \psi^u_i$, \quad $\hat{\mu}_v \gets \frac{1}{n} \sum_{i=1}^n \psi^v_i$

\STATE \textbf{// Compute cross-covariance}
\STATE $\hat{\Sigma}_{uv} \gets \frac{1}{n} \sum_{i=1}^n \psi^u_i (\psi^v_i)^\top$

\STATE \textbf{// HSIC estimate}
\STATE $\widehat{\mathrm{HSIC}} \gets \|\hat{\Sigma}_{uv} - \hat{\mu}_u \hat{\mu}_v^\top\|_F^2$

\STATE \textbf{return} $\widehat{\mathrm{HSIC}}$
\end{algorithmic}
\end{algorithm}

\subsection{Complexity Analysis}
\label{app:complexity}

Each iteration of Algorithm~\ref{alg:sarl} has time complexity
\begin{equation}
O\left(|\mathcal{E}| \cdot B \cdot (d_x + d_z) + N D^2 + N^3\right),
\end{equation}
where $|\mathcal{E}|$ is the number of environments, $B$ the batch size per environment, $N = |\mathcal{E}| \cdot B$ the pooled batch size, $D$ the RFF dimension, and $d_x, d_z$ the input and representation dimensions. The $ND^2$ term is RFF-HSIC (Algorithm~\ref{alg:rff_hsic}); the $N^3$ term is the kernel ridge solve (Algorithm~\ref{alg:kernel_residuals}), where one Cholesky factorization per conditioning variable is shared across all regression targets. Both terms depend on the minibatch size $N$ rather than the dataset size, and the total cost is linear in the number of candidate structures.

\subsection{Convergence Guarantees}
\label{app:convergence}

\begin{proposition}[Warm-Start Guarantee]
\label{prop:warm_start}
The ERM initialization $(\phi_0, w_0)$ satisfies $\phi_0 \in \Phi_\epsilon$ (Definition~\ref{def:epsilon_predictive}) with $\epsilon = O(n^{-1/2})$. This ensures optimization begins in the predictive region, avoiding degenerate solutions where $V_k = 0$ for all $k$ but prediction accuracy is trivial.
\end{proposition}

\begin{theorem}[Algorithm Convergence]
\label{thm:convergence}
Under Assumptions~\ref{assum:characteristic}--\ref{assum:rep_regularity}, with step size $\eta_t = O(t^{-1/2})$ and linear annealing $\beta_t = \beta_0 + (\beta_{\max} - \beta_0) \cdot \min(2t/T, 1)$, Algorithm~\ref{alg:sarl} satisfies:
\begin{enumerate}[nosep]
    \item \textbf{Stationarity:} $\mathbb{E}[\|\nabla \mathcal{L}(\phi_t, w_t)\|^2] \leq O(T^{-1/2})$
    \item \textbf{Weight concentration:} $\alpha_{k^*}^{(T)} \geq 1 - 2\exp(-\beta_T \eta / 2)$ where $k^*$ is the true structure
    \item \textbf{Violation decrease:} $\mathbb{E}[V_{k^*}(\phi_T)] \leq O(T^{-1/2})$
\end{enumerate}
where expectations are over stochastic minibatches.
\end{theorem}
\begin{proof}[Proof sketch]
Define the Lyapunov function:
\begin{equation}
\Psi_t := \mathcal{L}(\phi_t, w_t) + \gamma D_{\mathrm{KL}}(\boldsymbol{\alpha}^{(t)} \| \boldsymbol{\alpha}^*)
\end{equation}
where $\boldsymbol{\alpha}^* = e_{k^*}$ is the indicator for the true structure and $\gamma > 0$ is a coupling constant. Under smoothness (Assumption~\ref{assum:rep_regularity}), the expected per-step decrease satisfies:
\begin{equation}
\mathbb{E}[\Psi_{t+1} - \Psi_t \mid \mathcal{F}_t] \leq -c \eta_t \|\nabla \mathcal{L}(\phi_t, w_t)\|^2 + O(\eta_t^2)
\end{equation}
where $\mathcal{F}_t$ is the filtration up to iteration $t$. Summing over $t = 1, \ldots, T$ and using $\sum_t \eta_t = O(\sqrt{T})$, $\sum_t \eta_t^2 = O(\log T)$ yields Part 1.

Part 2 follows from Theorem~\ref{thm:violation_characterization}: as $V_{k^*}(\phi_t) \to 0$ and $V_k(\phi_t) \geq \eta/2$ for $k \neq k^*$, the SoftMin weights concentrate on $k^*$.

Part 3 follows from the objective structure: reducing $\mathcal{L}$ requires reducing $\mathrm{SoftMin}_\beta(\{V_k\})$, which for large $\beta$ is dominated by $\min_k V_k = V_{k^*}$.

\end{proof}
\section{Extension to Heterogeneous Structures}
\label{app:heterogeneous}

When different feature subsets follow different causal structures, we extend the framework to per-variable structure assignment.

\begin{definition}[Variable-Wise Structure Assignment]
\label{def:variable_structure}
Let $X = (X^{(1)}, \ldots, X^{(d)})$ be $d$-dimensional features. A \emph{heterogeneous causal structure} is a mapping $\sigma^*: \{1, \ldots, d\} \to \{1, 2, 3, \bot\}$ assigning each feature to a canonical structure or $\bot$ (none applies). The variable-wise violation for feature $j$ is $V_k^{(j)}(\phi) := V_k(\phi^{(j)})$.
\end{definition}

\begin{definition}[Calibrated Threshold]
\label{def:calibrated_threshold}
For significance level $\alpha \in (0,1)$, the calibrated threshold for feature $j$ is:
\begin{equation}
\tau_\alpha^{(j)} := \mathrm{quantile}_{1-\alpha}\left(\widehat{\mathrm{HSIC}}_{\mathrm{null}}^{(j)}\right)
\end{equation}
where $\widehat{\mathrm{HSIC}}_{\mathrm{null}}^{(j)}$ is the permutation distribution under the null hypothesis of independence.
\end{definition}

\begin{definition}[Heterogeneous Objective]
\label{def:heterogeneous_objective}
The structure-agnostic objective with heterogeneous assignment is:
\begin{equation}
\mathcal{L}_{\mathrm{het}}(\phi, w) := \max_{e \in \mathcal{E}} \hat{\mathcal{R}}_e(w \circ \phi) + \lambda \sum_{j=1}^{d} \mathrm{SoftMin}_\beta\left(\{V_k^{(j)}(\phi)\}_{k=1}^3\right)
\end{equation}
\end{definition}

\begin{theorem}[Heterogeneous Structure Identification]
\label{thm:heterogeneous_identification}
Under Assumptions~\ref{assum:characteristic}--\ref{assum:rep_regularity} applied to each feature subset, let $\phi_n$ minimize $\mathcal{L}_{\mathrm{het}}$ with $n$ samples per environment. Define the estimated per-variable structure:
\begin{equation}
\hat{\sigma}_n(j) := \begin{cases}
\arg\min_{k \in \{1,2,3\}} V_k^{(j)}(\phi_n) & \text{if } \min_k V_k^{(j)}(\phi_n) \leq \tau_\alpha^{(j)} \\
\bot & \text{otherwise}
\end{cases}
\end{equation}

Then for each feature $j$ with true structure $\sigma^*(j) \neq \bot$:
\begin{equation}
\mathbb{P}\left(\hat{\sigma}_n(j) = \sigma^*(j)\right) \geq 1 - 6\exp\left(-\frac{n (\eta^{(j)})^2}{2048\kappa^4}\right) - \alpha
\end{equation}
where $\eta^{(j)}$ is the faithfulness gap of feature $j$. For features with $\sigma^*(j) = \bot$, the false-positive rate satisfies $\mathbb{P}(\hat{\sigma}_n(j) \neq \bot) \leq \alpha$.
\end{theorem}

\begin{corollary}[Heterogeneous Generalization]
\label{cor:heterogeneous_generalization}
Under the conditions of Theorem~\ref{thm:heterogeneous_identification}, the generalization bound (Theorem~\ref{thm:generalization}) holds with estimation error:
\begin{equation}
\xi_n^{\mathrm{het}} = O\left(\sqrt{\frac{d \cdot \mathcal{C}(\Phi, \mathcal{W}) + \log(d/\delta)}{n}}\right) + C_\eta \sum_{j: \sigma^*(j) \neq \bot} V_{\sigma^*(j)}^{(j)}(\phi_n)
\end{equation}
The factor $d$ reflects the cost of per-feature structure identification.
\end{corollary}

\section{Additional Experiments}
\label{app:exp}

\subsection{Implementation}
\label{app:implementation}

\subsubsection{Hyperparameter Selection}
\label{sec:hyperparameters}

The objective $\mathcal{L}_\rho(\phi, w)$ involves hyperparameters: $\lambda$ (invariance weight), $\beta$ (SoftMin temperature), $\gamma$ (kernel regression regularization), $D$ (RFF dimension), and $\rho$ (CVaR level). Table~\ref{tab:hyperparameters_alg} provides defaults and selection guidance.

\begin{table}[h]
\vspace{-4mm}
\centering
\small
\caption{Recommended hyperparameter settings for Algorithm~\ref{alg:sarl}.}
\label{tab:hyperparameters_alg}
\begin{tabular}{clll}
\toprule
Parameter & Default & Range & Selection Strategy \\
\midrule
$\lambda$ & 1.0 & $[0.5, 5]$ & Per experiment, by OOD accuracy (oracle selection) \\
$\beta_0$ & 1.0 & $[0.5, 2]$ & Fixed (soft selection initially) \\
$\beta_{\max}$ & 50 & $[20, 100]$ & Higher for sharper selection \\
$\gamma$ & $n^{-1/3}$ & $[n^{-1/2}, n^{-1/5}]$ & Cross-validation for kernel regression \\
$\rho$ & 0.5 & $[1/|\mathcal{E}|, 1]$ & Smaller $\rho$ when all environments are trusted \\
$D$ & 1000 & $[500, 2000]$ & Increase if approximation error high \\
\bottomrule
\end{tabular}
\end{table}

\subsubsection{Numerical Stability}

\begin{itemize}[nosep]
    \item Add small $\epsilon = 10^{-6}$ to diagonal of kernel matrices before inversion
    \item Clip $\beta V_k$ to $[-20, 20]$ before exp to prevent overflow
    \item Normalize HSIC values by their empirical standard deviation for comparable scales
\end{itemize}

\subsubsection{Kernel Selection}
We use the following kernels:
\begin{itemize}[nosep]
    \item \textbf{Representations $k_{\mathcal{Z}}$:} Gaussian RBF with bandwidth $\sigma = \mathrm{median}(\|z_i - z_j\|)$
    \item \textbf{Labels $k_{\mathcal{Y}}$:} Delta kernel for classification; Gaussian RBF for regression
    \item \textbf{Environments $k_{\mathcal{E}}$:} Delta kernel (discrete environments)
\end{itemize}

\subsection{Synthetic DGP Equations}
\label{app:dgp_synthetic}

For each canonical structure, we generate $m=3$ training environments with $d=10$ features:

\begin{itemize}[leftmargin=*, nosep]
    \item \textbf{$\mathcal{G}_1$ (Anti-causal):} $Y \sim \text{Bernoulli}(0.5)$, $Z \sim \mathcal{N}(0, I_5)$, $\epsilon \sim \mathcal{N}(0, \sigma_e^2 I_d)$ where $\sigma_e \in \{0.5, 1.0, 1.5\}$ varies across environments. $X = W_1 \cdot \text{onehot}(Y) + W_2 \cdot Z + \epsilon$ where $W_1 \in \mathbb{R}^{d \times 2}$ and $W_2 \in \mathbb{R}^{d \times 5}$ are fixed random matrices.
    
    \item \textbf{$\mathcal{G}_2$ (Confounded-descendant):} $U \sim \mathcal{N}(\mu_e, I_3)$ with environment-specific means $\mu_e \in \{(-1,0,0), (0,0,0), (1,0,0)\}$, $X = W_1 \cdot U + \epsilon_X$ where $\epsilon_X \sim \mathcal{N}(0, 0.1 I_d)$, and $Y = \sigma(w_{xy}^\top X + w_{uy}^\top U + \epsilon_Y)$ where $\epsilon_Y \sim \mathcal{N}(0, 0.1)$ and $\sigma$ is the sigmoid function, with the outcome binarized via Bernoulli sampling.
    
    \item \textbf{$\mathcal{G}_3$ (Confounded-outcome):} $U \sim \mathcal{N}(0, I_3)$, $X_{Z^\perp} = W_1 \cdot U + \epsilon_1$ where $\epsilon_1 \sim \mathcal{N}(0, \sigma_e^2 I_{d/2})$ with $\sigma_e \in \{0.5, 1.0, 1.5\}$ varying across environments, $Y = \sigma(w^\top X_{Z^\perp})$ binarized, $X_Z = W_2 \cdot \text{onehot}(Y) + W_3 \cdot U + \epsilon_2$, and $X = (X_Z, X_{Z^\perp})$.
\end{itemize}

All weight matrices are sampled from $\mathcal{N}(0, 0.5^2)$ and fixed across runs.

\subsection{Hyperparameter Sensitivity Results}
\label{app:hyperparam_sensitivity}

We sweep the invariance penalty weight $\lambda$ and SoftMin temperature $\beta_{\max}$ across varying degrees of structural misspecification using the bidirectional DGP from Section~\ref{sec:robustness} with $\nu \in \{0.25, 0.50, 0.75\}$.
Results are shown in Table \ref{tab:lambda_sensitivity} to Table \ref{tab:beta_sensitivity}. 

\begin{table}[h]
\vspace{-4mm}
\centering
\caption{Sensitivity to invariance penalty weight $\lambda$ ($\beta_{\max}=50$ fixed). OOD accuracy (\%).}
\label{tab:lambda_sensitivity}
\small
\begin{tabular}{lcccc}
\toprule
Setting & $\lambda{=}0.5$ & $\lambda{=}1.0$ & $\lambda{=}2.0$ & $\lambda{=}5.0$ \\
\midrule
$\nu{=}0.25$ (near-causal)     & 46.0 & 46.8 & 47.8 & 50.2 \\
$\nu{=}0.50$ (bidirectional)   & 69.0 & 70.3 & 70.4 & 73.9 \\
$\nu{=}0.75$ (near-anticausal) & 99.7 & 99.5 & 99.5 & 99.6 \\
\bottomrule
\end{tabular}
\vspace{-4mm}
\end{table}

\begin{table}[h]
\vspace{-4mm}
\centering
\caption{Sensitivity to SoftMin temperature $\beta_{\max}$ ($\lambda=5.0$ fixed). OOD accuracy (\%).}
\label{tab:beta_sensitivity}
\small
\begin{tabular}{lccc}
\toprule
Setting & $\beta_{\max}{=}20$ & $\beta_{\max}{=}50$ & $\beta_{\max}{=}100$ \\
\midrule
$\nu{=}0.25$ (near-causal)     & 50.3 & 50.2 & 51.3 \\
$\nu{=}0.50$ (bidirectional)   & 71.4 & 73.9 & 72.9 \\
$\nu{=}0.75$ (near-anticausal) & 99.6 & 99.6 & 99.6 \\
\bottomrule
\end{tabular}
\vspace{-4mm}
\end{table}

$\lambda$ sensitivity is mild and monotonic: total variation from $\lambda{=}0.5$ to $\lambda{=}5.0$ is at most $\sim 5$pp in the hardest settings, with no catastrophic failure. $\beta_{\max}$ is even more stable (variation $\leq 2.5$pp). When the anti-causal component dominates ($\nu{=}0.75$), all configurations reach $\sim 99.5$\%. We recommend $\lambda \in [1.0, 5.0]$ and $\beta_{\max}=50$ as robust defaults.

\subsection{Crossover Between Candidate Structures}
\label{app:crossover}
Table~\ref{tab:crossover} sweeps $\nu$ across the crossover between $\mathcal{G}_2$ ($\nu{=}0$) and $\mathcal{G}_1$ ($\nu{=}1$) in the bidirectional DGP of Section~\ref{sec:robustness}, comparing SoftMin with hard $\arg\min_k V_k$ selection. The sweep uses independent seeds; at $\nu{=}0.50$ it agrees with Table~\ref{tab:robustness} within one standard error. The last column is the permutation $p$-value of $\min_k V_k$ against the null threshold $\tau_\alpha$ (Definition~\ref{def:calibrated_threshold}).

\begin{table}[h]
\centering
\caption{Adaptive weights and OOD accuracy (\%) across the $\mathcal{G}_2$--$\mathcal{G}_1$ crossover.}
\label{tab:crossover}
\small
\addtolength{\tabcolsep}{-2pt}
\begin{tabular}{lcccccc}
\toprule
$\nu$ & $\alpha_1$ & $\alpha_2$ & $\alpha_3$ & SoftMin & Hard $\arg\min$ & Perm.\ $p$ \\
\midrule
0.40 & 0.407\scriptsize{$\pm$0.041} & 0.489\scriptsize{$\pm$0.080} & 0.104\scriptsize{$\pm$0.058} & \textbf{59.5}\scriptsize{$\pm$2.0} & 56.5\scriptsize{$\pm$2.7} & 0.013\scriptsize{$\pm$0.018} \\
0.45 & 0.415\scriptsize{$\pm$0.073} & 0.478\scriptsize{$\pm$0.139} & 0.107\scriptsize{$\pm$0.075} & \textbf{66.2}\scriptsize{$\pm$1.5} & 62.5\scriptsize{$\pm$2.9} & 0.011\scriptsize{$\pm$0.012} \\
0.50 & 0.421\scriptsize{$\pm$0.089} & 0.422\scriptsize{$\pm$0.143} & 0.157\scriptsize{$\pm$0.071} & \textbf{72.4}\scriptsize{$\pm$1.5} & 68.6\scriptsize{$\pm$3.6} & 0.181\scriptsize{$\pm$0.011} \\
0.55 & 0.484\scriptsize{$\pm$0.079} & 0.417\scriptsize{$\pm$0.143} & 0.099\scriptsize{$\pm$0.077} & \textbf{79.8}\scriptsize{$\pm$2.0} & 76.5\scriptsize{$\pm$2.2} & 0.033\scriptsize{$\pm$0.039} \\
0.60 & 0.549\scriptsize{$\pm$0.069} & 0.346\scriptsize{$\pm$0.126} & 0.105\scriptsize{$\pm$0.077} & \textbf{87.4}\scriptsize{$\pm$1.9} & 84.7\scriptsize{$\pm$2.2} & 0.016\scriptsize{$\pm$0.033} \\
\bottomrule
\end{tabular}
\end{table}

\subsection{Margin Calibration}
\label{app:calibration}
Table~\ref{tab:calibration} bins synthetic runs with known structure by the weight margin $\alpha_{(1)} - \alpha_{(2)}$. Identification accuracy increases monotonically with the margin, so the margin measures how strongly the data support the selected structure.

\begin{table}[h]
\centering
\caption{Identification accuracy by weight margin on synthetic data.}
\label{tab:calibration}
\small
\begin{tabular}{lccc}
\toprule
Margin & Runs & Correct & ID accuracy (\%) \\
\midrule
$<0.1$       & 48 & 20 & 41.7 \\
$0.1$--$0.3$ & 60 & 39 & 65.0 \\
$0.3$--$0.5$ & 72 & 62 & 86.1 \\
$>0.5$       & 60 & 57 & 95.0 \\
\bottomrule
\end{tabular}
\end{table}

\subsection{Semi-Synthetic Bayesian-Network Benchmarks}
\label{app:bn}
\textbf{Construction.} For each benchmark, we select a target $Y$ in a published network, take $X$ as its Markov blanket, and induce environments by soft interventions on a latent parent outside the blanket. The canonical structure is read off the published DAG by the rules of Section~\ref{sec:background}. $\mathcal{G}_3$ requires two independent root components reconverging at $Y$, which ALARM (12 roots) supplies twice and ASIA (2 roots) once; $\mathcal{G}_1$ is not derivable under this recipe, since it requires the Markov blanket of $Y$ to consist only of descendants of $Y$. Before training, each derived structure is verified in simulation: its own invariance holds while the competing one fails. We keep a benchmark only if the competing-invariance violation exceeds 5.0 standard deviations of the permutation null.

\textbf{Departures from the published networks.} The SACHS consensus DAG fits its own observational data weakly, so SACHS requires larger interventions to produce contrast. For ASIA, we keep the published structure but re-parameterize two conditional probability tables whose published marginals leave the exogenous branch nearly constant. ALARM requires no adjustment.

\textbf{Discussion.} Five benchmarks span two derivable structures across three networks (three $\mathcal{G}_2$, two $\mathcal{G}_3$), so a method that returns a single structure by default is correct on at most 3/5. alarm\_co separates more cleanly than asia\_g2 (28.3 vs.\ 24.6 null standard deviations) but has lower accuracy, because its target has three classes over 12 one-hot columns rather than two. alarm\_g2 barely clears the 5.0 threshold and is correspondingly the least reliable. The two $\mathcal{G}_3$ benchmarks come from different networks, so their agreement is not an artifact of a single graph.

\subsection{Cross-Fitting Variant}
\label{app:cross_fitting}

Conditional HSIC via kernel ridge regression residualization can be sensitive to estimation bias. We implement a cross-fitting variant following the SplitKCI protocol~\citep{pogodin2024practical}: fit kernel regression on fold 1, compute residuals on fold 2, swap and average. Results are displayed in Table \ref{tab:cross_fitting}.

\begin{table}[h]
\centering
\caption{Cross-fitting comparison on the bidirectional DGP ($\nu=0.5$, $n{=}1000$ per environment). Independent run of the default configuration.}
\label{tab:cross_fitting}
\small
\begin{tabular}{lc}
\toprule
Variant & OOD Acc (\%) \\
\midrule
No splitting (default)    & 73.9 \scriptsize{$\pm$ 2.7} \\
Cross-fitted (2-fold)     & 74.0 \scriptsize{$\pm$ 2.7} \\
ERM (reference)           & 61.2 \scriptsize{$\pm$ 4.5} \\
\bottomrule
\end{tabular}
\end{table}

OOD accuracy is virtually identical (73.9 vs.\ 74.0), confirming that estimation-induced bias from joint residualization does not meaningfully affect downstream performance at the sample sizes used in our experiments. Both variants outperform ERM by $\sim 13$pp.

\subsection{EMA Conditioning Target}
\label{app:ema}
Table~\ref{tab:ema} conditions $V_1$ and $V_2$ on an exponential moving average (EMA) of the encoder weights with momentum $m$, while gradients flow only through the online encoder. Results are nearly unchanged, so SaCRL is robust to this choice.

\begin{table}[h]
\centering
\caption{EMA conditioning target on synthetic $\mathcal{G}_2$ ($n{=}2000$). $m{=}0$ is the default online encoder.}
\label{tab:ema}
\small
\begin{tabular}{lcc}
\toprule
Momentum $m$ & Stru.\ ID (\%) & OOD Acc (\%) \\
\midrule
0 (default) & 89.0 & 84.2 \\
0.99        & 89.4 & 84.3 \\
0.999       & 89.8 & 84.5 \\
\bottomrule
\end{tabular}
\end{table}

\subsection{RFF Approximation Error}
\label{app:rff_error}
Table~\ref{tab:rff_error} reports the mean absolute difference between RFF-based and exact HSIC on synthetic data ($n{=}2000$). The error decreases steadily with $D$, from $1.42\times10^{-3}$ at $D{=}100$ to $4.4\times10^{-4}$ at $D{=}2000$, consistent with the $\eta^{-2}$ scaling of the required $D$ in Corollary~\ref{cor:rff_identification}.

\begin{table}[h]
\centering
\caption{RFF approximation error versus exact HSIC (synthetic, $n{=}2000$).}
\label{tab:rff_error}
\small
\begin{tabular}{lccccc}
\toprule
$D$ & 100 & 250 & 500 & 1000 & 2000 \\
\midrule
Mean $|\widehat{\mathrm{HSIC}}_{\mathrm{RFF}} - \widehat{\mathrm{HSIC}}|$ ($\times 10^{-4}$) & 14.2 & 12.9 & 8.4 & 7.7 & 4.4 \\
\bottomrule
\end{tabular}
\end{table}

\subsection{Effect of the Number of Training Environments}
\label{app:num_envs}

Table~\ref{tab:num_envs} reports the effect of varying the number of training environments $m$ on synthetic $\mathcal{G}_2$ data ($n=2000$). With only two environments, structure identification is unstable due to insufficient contrast for testing the three invariances (consistent with Assumption~\ref{assum:sufficient_diversity} requiring $\mathrm{rank}(\mathbf{D}) = 3$). Using three or more environments yields reliable identification ($\geq 89\%$) and OOD accuracy ($\geq 84\%$), with diminishing returns beyond $m=4$.

\begin{table}[h]
\centering
\caption{Effect of number of training environments $m$ on synthetic $\mathcal{G}_2$ data ($n{=}2000$).}
\label{tab:num_envs}
\small
\begin{tabular}{lcc}
\toprule
$m$ & Stru. ID. Acc (\%) & OOD Acc (\%) \\
\midrule
2 & 71.5 \scriptsize{$\pm$4.2} & 74.3 \scriptsize{$\pm$2.8}  \\
3 & 89.0 \scriptsize{$\pm$1.6} & 84.2 \scriptsize{$\pm$1.3}  \\
4 & 96.5 \scriptsize{$\pm$1.2} & 85.1 \scriptsize{$\pm$1.1}  \\
5 & 97.0 \scriptsize{$\pm$0.9} & 85.4 \scriptsize{$\pm$0.9} \\
\bottomrule
\end{tabular}
\end{table}

\subsection{Computational Efficiency and Runtime Analysis}
\label{app:runtime}

\subsubsection{Training Time and Memory}

Table~\ref{tab:runtime_app} compares training time and memory. With $D{=}1000$ RFF features, our method adds modest overhead ($\sim 30\%$) over ERM while avoiding the $O(n^2)$ complexity of exact HSIC. For Colored MNIST (10K images), exact kernel matrices would require $>10$ GB GPU memory and $4\times$ training time; RFF introduces negligible accuracy loss while enabling scalability.

\begin{table}[h]
\centering
\caption{Training time and memory for 100 epochs on Colored MNIST (10K samples).}
\label{tab:runtime_app}
\small
\begin{tabular}{lccc}
\toprule
Method & Time (min) & GPU Memory & Relative \\
\midrule
ERM                       & 8.3  & 2.1 GB  & $1.0\times$ \\
IRM                       & 9.8  & 2.4 GB  & $1.18\times$ \\
SaCRL (RFF, $D{=}1000$)   & 10.9 & 2.7 GB  & $1.31\times$ \\
SaCRL (exact HSIC)        & 33.2 & 11.8 GB & $4.0\times$ \\
\bottomrule
\end{tabular}
\end{table}

\subsubsection{Per-Epoch Runtime Breakdown}

Table~\ref{tab:runtime_breakdown} breaks down the training time into different components. We observe that computing HSIC via random Fourier features (RFF) significantly reduces the computational cost. 

\begin{table}[h]
\centering
\caption{Runtime breakdown per epoch on PACS.}
\label{tab:runtime_breakdown}
\small
\begin{tabular}{lcc}
\toprule
Component & Time (s) & Fraction \\
\midrule
Forward pass            & 4.59 & 42\% \\
HSIC computation (RFF)  & 3.02 & 28\% \\
Kernel regression       & 1.75 & 16\% \\
Backward pass           & 1.19 & 11\% \\
Other                   & 0.34 & 3\% \\
\midrule
Total                   & 10.9 & 100\% \\
\bottomrule
\end{tabular}
\end{table}

\subsection{Fine-Tuned Backbone}
\label{app:finetune}
\begin{table}[!t]
\centering
\caption{PACS (Photo held out): frozen vs.\ fine-tuned ResNet18. OOD accuracy (\%), 3 seeds.}
\label{tab:finetune}
\small
\begin{tabular}{lccc}
\toprule
Method & Frozen & Fine-tuned & $\Delta$ \\
\midrule
ERM   & 93.2 & 95.1 & +1.9 \\
IRM   & 93.2 & 95.2 & +2.0 \\
VREx  & 93.0 & 94.9 & +1.9 \\
CIRCE & 92.5 & 94.1 & +1.6 \\
SaCRL & 94.0 & 95.0 & +1.0 \\
\bottomrule
\end{tabular}
\end{table}

\section{Limitation}
\label{app:limit}

Our framework targets three canonical structures arising under partial observability and confounding-mediated environment effects; \emph{cyclic systems, multi-confounder dynamics, or settings where $E$ directly rewires $X$--$Y$ fall outside this scope}, though the heterogeneous extension and the empirical robustness study (Section~\ref{sec:robustness}) show graceful degradation under misspecification. In such cases, SaCRL signals the mismatch: if $\min_k V_k$ exceeds the permutation threshold $\tau_\alpha$ (Definition~\ref{def:calibrated_threshold}), the data do not support a canonical structure, and the selected structure should not be interpreted as identified. Structure identification further requires sufficient environment diversity (Assumption~\ref{assum:sufficient_diversity}) and at least $m \geq 3$ training environments (Appendix~\ref{app:num_envs}); under near-identical environments, the invariance penalty loses informativeness even though identification remains stable. Conditional HSIC estimation depends on kernel regression quality and may benefit from cross-fitting in low-sample regimes, per results in Table \ref{tab:cross_fitting}.
\section{Applications and Broader Connections}
\label{app:applications}

SaCRL addresses a question that recurs across several fields: \emph{which invariance should a model enforce when the causal structure of the data is unknown?} We discuss five areas where this question arises, review recent work in each, and state what SaCRL adds.

{\bf Identifiability beyond an assumed structure.} A central open problem in representation learning is \emph{identifiability}: whether the data determine the learned representation uniquely, up to simple transformations such as permutation and element-wise reparameterization, so that it recovers the factors that generated the data rather than one of many equally predictive alternatives. Without identifiability, a representation carries no guarantee of meaning, and the guarantees built on it, including robustness under distribution shift, do not follow. The foundational results are negative: nonlinear independent component analysis is not identifiable in general~\citep{hyvarinen1999nonlinear}, and unsupervised disentanglement is impossible without inductive biases on both the model and the data~\citep{locatello2019challenging}. Progress since then has come from adding structure. Auxiliary variables such as environment labels~\citep{khemakhem2020variational}, paired views from data augmentation~\citep{von2021self}, and interventions~\citep{ahuja2023interventional} each restore identifiability, and together they define the program of causal representation learning~\citep{scholkopf2021toward}. Each of these results, however, relies on a generative structure supplied by the analyst: which variables cause which, and where environments intervene. The structure itself is never identified; it enters as an assumption, and when the assumption is wrong the guarantees no longer apply. SaCRL moves identification one level up. Within the three canonical structures, it identifies from multi-environment data which structure generated the data, and therefore which invariance a representation should satisfy.

{\bf Human--AI interaction and longitudinal behavioral modeling.} Conversational AI systems now interact with the same people over weeks and months, which makes human--AI interaction a longitudinal, two-way process: users adapt to the system, and the system adapts to users through personalization, memory, and training on their feedback. Recent evidence shows why this process must be modeled causally. In a four-week randomized study with about 1{,}000 participants and more than 300{,}000 messages, the assigned chatbot conditions had no significant effect on psychosocial outcomes, whereas participants who chose to use the chatbot more reported worse outcomes~\citep{fang2025ai}, which leaves open whether heavier use drives these outcomes or reflects them. Language models tailor their answers to the beliefs users express~\citep{sharma2024towards}, so the user's state shapes the model's behavior; conversely, deployed predictions and recommendations change the very behavior they later learn from~\citep{perdomo2020performative,jiang2019degenerate}. Preferences also differ systematically across people and cultures, as PRISM documents by linking 1{,}500 participants from 75 countries to 8{,}011 live conversations with 21 language models~\citep{kirk2024prism}. The gap is that any model of user or assistant behavior over time must decide what to hold invariant across users, cultures, and time periods, and that decision depends on the direction of influence. If the modeled signal is generated by the target, as when a user's latent state generates their messages, the anti-causal invariance $\mathcal{I}_1$ applies; if the signal drives the target, as when the assistant's responses drive the user's next action, the causal invariances $\mathcal{I}_2$ or $\mathcal{I}_3$ apply; and if influence runs both ways, the system is in a feedback loop that no single invariance describes. SaCRL gives this decision a statistical footing. Treating users, cultural groups, or time windows as environments, it selects the invariance supported by the data, assigns different structures to different behavioral signals through its heterogeneous extension (Appendix~\ref{app:heterogeneous}), and, when influence is bidirectional, keeps its structure weights spread instead of committing to a wrong structure (Section~\ref{sec:robustness}, Appendix~\ref{app:crossover}). The feedback regime thus becomes observable rather than silently misspecified. Cyclic dynamics lie outside the three canonical structures, and extending the candidate set to dynamic, feedback-coupled structures is a natural next step toward causally grounded models of long-term human--AI relationships.

{\bf Multi-site clinical prediction and the direction of causation.} Clinical models are trained at some hospitals and deployed at others, and \emph{dataset shift} between sites is a leading cause of failure. A pneumonia detector trained on chest radiographs degraded at external hospitals, in part because it exploited site-specific signals~\citep{zech2018variable}; a widely deployed proprietary sepsis model performed markedly worse under external validation than its developer reported~\citep{wong2021external}; and clinicians are now urged to treat dataset shift as a routine deployment risk~\citep{finlayson2021clinician}. A causal view explains why remedies differ. In \emph{anti-causal} tasks the disease generates the image, so shifts appear as changes in prevalence or manifestation, whereas in \emph{causal} tasks the inputs generate the outcome~\citep{castro2020causality}. Dataset biases that look identical in the data can therefore demand different mitigation strategies~\citep{jones2024causal}, and shift-stable modeling requires the analyst to specify the causal graph~\citep{subbaswamy2020development}. The gap is that the causal direction must be declared by hand, per task and often per variable: symptoms are anti-causal, biomarkers are not (Figure~\ref{fig:toy_example}). SaCRL treats each hospital as an environment and infers the direction from data. Theorem~\ref{thm:structure_identification} bounds the probability of selecting a wrong structure by $6\exp(-n\eta^2/(2048\kappa^4))$, the heterogeneous extension assigns a structure to each variable (Appendix~\ref{app:heterogeneous}), and the permutation test of Section~\ref{sec:method} flags cases where no canonical structure fits, so a clinical team knows when the selected structure should not be trusted.

{\bf Domain adaptation and generalization: which distribution to align.} Methods for learning under distribution shift differ in which part of the joint distribution they hold fixed across domains. \emph{Marginal alignment} makes $P(\phi(X))$ invariant, as in domain-adversarial training~\citep{ganin2016domain}; \emph{conditional alignment} makes $P(\phi(X) \mid Y)$ invariant, as under target and conditional shift~\citep{zhang2013domain} and generalized label shift~\citep{tachet2020domain}; and \emph{covariate-shift} methods such as IRM make $P(Y \mid \phi(X))$ invariant~\citep{arjovsky2019invariant}. The appropriate assumption depends on the causal direction~\citep{scholkopf2012causal}: label shift is natural in anti-causal tasks and can be corrected with black-box predictors~\citep{lipton2018detecting}, while marginal alignment provably fails when label distributions differ across domains~\citep{zhao2019learning}. Recent work makes this dependence explicit: each causal graph entails different independence constraints, so any algorithm built on a single fixed constraint fails on some shifts~\citep{kaur2022modeling,wang2023causal}, consistent with the finding that no domain generalization method reliably outperforms ERM~\citep{gulrajani2020search}. Adaptive methods such as CACM~\citep{kaur2022modeling} still take the graph as input. SaCRL removes this requirement. Its three invariances are exactly these three alignment targets, $\mathcal{I}_3$ (marginal), $\mathcal{I}_1$ (conditional), and $\mathcal{I}_2$ (covariate shift), measured on a common HSIC scale and selected by SoftMin without the graph. Theorem~\ref{thm:violation_characterization} shows that a predictive representation can satisfy only the invariance of the true structure, and the experiments show the cost of a wrong choice: on anti-causal Colored MNIST, IRM, which enforces $\mathcal{I}_2$, falls below ERM (21.3\% vs.\ 27.2\%), while SaCRL reaches 45.2\% (Table~\ref{tab:colored_mnist}).

{\bf Algorithmic fairness: independence, separation, and sufficiency.} The three main group-fairness criteria constrain a score or representation $R$ with respect to a sensitive attribute $A$: \emph{independence} (demographic parity) requires $R \perp\!\!\!\perp A$, \emph{separation} (equalized odds) requires $R \perp\!\!\!\perp A \mid Y$, and \emph{sufficiency} (calibration within groups) requires $Y \perp\!\!\!\perp A \mid R$~\citep{barocas2023fairness}. Impossibility results show that these criteria cannot hold simultaneously outside degenerate cases~\citep{kleinberg2016inherent,chouldechova2017fair}. Recent causal analyses show that the appropriate criterion depends on the data-generating graph: in anti-causal prediction, separation coincides with a notion of risk invariance~\citep{makar2022fairness}, and whether data balancing achieves a given criterion depends on the graph~\citep{schrouff2024mind}. The gap is that practitioners choose a criterion without knowing which one the graph permits at no cost in accuracy. With $A$ in the role of $E$, SaCRL's invariances are precisely these criteria: $\mathcal{I}_3$ is independence, $\mathcal{I}_1$ is separation, and $\mathcal{I}_2$ is sufficiency, with $V_3$, $V_1$, and $V_2$ measuring their violations. When $A$ enters the system as $E$ does in Figure~\ref{fig:structures}, Theorem~\ref{thm:violation_characterization} gives a causal counterpart of the impossibility results: a predictive representation can satisfy at most the criterion matching the true structure, that criterion is achievable when $X_z^{\perp}$ is predictive, and the adaptive weights reveal which one it is. Since the predictor is a function of $\phi(X)$, independence and separation of $\phi(X)$ carry over to the score. The choice of a fairness criterion remains a normative decision; SaCRL informs it by showing which criterion the data can support.

{\bf Diagnosing and explaining distribution shift.} A related line of work asks not how to be robust, but what changed. \emph{Shift detection} tests whether training and deployment distributions differ~\citep{rabanser2019failing}. \emph{Shift attribution} explains a performance drop by the mechanisms that changed, either using a known causal graph~\citep{zhang2023did} or by decomposing the drop into a change in the covariate distribution ($X$-shift) and in the outcome given the covariates ($Y \mid X$-shift)~\citep{cai2026diagnosing}. A benchmark of natural tabular shifts found $Y \mid X$-shifts to be the most prevalent and showed that conclusions about algorithms depend on the type of shift~\citep{liu2023need}. These tools explain the failure of a fixed model after deployment; they do not decide what a model should be invariant to during training. SaCRL connects the two. Its diversity matrix $\mathbf{D}$ (Definition~\ref{def:diversity_matrix}) measures, across training environments, the marginal shift in $P(X)$, the conditional shift in $P(X \mid Y)$, and the outcome shift in $P(Y \mid X)$ with kernel mean embeddings; the first and third are the $X$- and $Y \mid X$-shifts of the attribution literature, computed between training environments rather than between training and deployment. A permutation check on $\mathbf{D}$ before training determines which invariances are testable (Appendix~\ref{app:assumptions}), and training then enforces the one the data support, so diagnosis and remedy become parts of a single objective.

{\bf Invariance-based causal discovery and conditional independence testing.} Heterogeneity across environments is also a source of causal knowledge. Invariant causal prediction identifies the causal parents of a target as the variables whose conditional distribution is stable across environments~\citep{peters2016causal}, and its nonlinear extension replaces linear models with nonparametric conditional independence tests~\citep{heinze2018invariant}; interventional single-cell data, such as the protein-signaling measurements of~\citet{sachs2005causal}, remain standard testbeds. These methods rest on conditional independence testing, where kernel tests are widely used~\citep{zhang2012kernel} and recent split-sample variants improve calibration~\citep{pogodin2024practical}. A fundamental limit applies: no conditional independence test can control its level over all null distributions while retaining power against all alternatives~\citep{shah2020hardness}, so useful guarantees require explicit assumptions. Existing methods test invariances of observed variables and return a set of parents; they neither handle high-dimensional inputs such as images nor learn a representation. SaCRL tests invariances of a learned representation $\phi(X)$ and learns $\phi$ jointly with the test. Its guarantees are stated under explicit assumptions consistent with this limit, namely an additive-noise model for conditional HSIC (Lemma~\ref{lem:conditional_hsic}) and a positive faithfulness gap $\eta$ (Definition~\ref{def:faithfulness_gap}). On semi-synthetic benchmarks built from the ASIA, ALARM, and SACHS networks, SaCRL selects the structure fixed by the published graph in 75--95\% of runs (Table~\ref{tab:bn}).

\end{document}